\documentclass{article}

\usepackage{iclr2027_conference,times}
\usepackage[utf8]{inputenc}
\usepackage[T1]{fontenc}
\usepackage{hyperref}
\usepackage{url}
\usepackage{booktabs}
\usepackage{amsfonts}
\usepackage{microtype}
\usepackage{xcolor}
\usepackage{graphicx}
\usepackage{subcaption}
\usepackage{algorithm}
\usepackage{algorithmic}
\usepackage{amsmath}
\usepackage{amssymb}
\usepackage{mathtools}
\usepackage{placeins}


\usepackage{float}
\usepackage[normalem]{ulem}
\usepackage{amsthm}
\definecolor{DarkForestGreen}{RGB}{0,85,0}

\definecolor{DarkForestGreen}{RGB}{0,85,0}

\theoremstyle{plain}
\newtheorem{theorem}{Theorem}[section]
\newtheorem{proposition}[theorem]{Proposition}
\newtheorem{lemma}[theorem]{Lemma}
\newtheorem{corollary}[theorem]{Corollary}
\theoremstyle{definition}

\newtheorem{assumption}[theorem]{Assumption}
\theoremstyle{remark}
\newtheorem{remark}[theorem]{Remark}

\title{Policy as Data: Replay-Based Policy \\ Dual Averaging via Advantage Regression}

\author{Nianli Peng$^{1}$, Geoffrey J. Gordon$^{2}$, Kiant\'e Brantley$^{1}$\\
$^{1}$Harvard University, $^{2}$Carnegie Mellon University}

\iclrfinalcopy

\begin{document}

\maketitle
\lhead{Preprint}

\begin{abstract}
Actor-critic methods reuse past experience to improve sample efficiency. However,
historical data are typically regarded as off-policy samples for the current
policy-improvement update. This work introduces Regularized Dual Averaging Actor
Critic (RDA2C), which assigns a distinct role to replay. In regularized dual
averaging, the subsequent policy is determined by accumulated policy-improvement
feedback, so historical advantage estimates contribute directly to the actor objective rather
than solely to the most recent update. RDA2C stores state-action samples
with critic-estimated advantage labels, fits a dual score model $Z_\theta$ to
the aggregated dataset, and derives the current policy from the accumulated score model using the entropy mirror map. In this way, replay defines an empirical dual objective from which the policy is computed. To analyze RDA2C, we establish a finite-time value-gap decomposition, separating the regularized dual-averaging term from errors due to stale-replay supervised fitting, critic bias, finite-buffer variance, and replay coverage, and stating the assumptions under which each error is bounded. RDA2C accepts advantage labels from any critic. With GAE labels, RDA2C outperforms PPO on six of eight MuJoCo tasks and eight of twelve Atari games. RDA2C also outperforms AAPDA, the closest dual-averaging baseline, on six of eight MuJoCo tasks. With twin-$Q$ labels, RDA2C matches SAC at matched batch size and update frequency.
\end{abstract}

\section{Introduction}

Experience replay is an important mechanism for improving sample efficiency in deep reinforcement learning by allowing previously collected transitions to support multiple learning updates \citep{mnih2015human,schaul2015prioritized,fedus2020revisiting}.
For value-based learning, the role of replay is relatively straightforward.
Past transitions serve to calculate supplementary targets for temporal-difference updates or Bellman regression.
In contrast, using historical data for policy learning introduces additional complexities.
The policy improvement signal is defined with respect to the current policy, whereas replayed actions originate from previous behavior policies.
Standard actor updates therefore treat replayed samples as off-policy data for the current policy-improvement step \citep{wang2017acer,espeholt2018impala}.
Near-on-policy methods limit reuse and constrain the likelihood ratio between the behavior and current policies, while more explicitly off-policy methods use importance weighting or related corrections \citep{schulman2017proximal,espeholt2018impala,wang2017acer}.
Alternatively, methods such as DDPG, TD3, and SAC learn an off-policy critic from replay and update the actor through that critic, so the actor update itself stays on-policy while the critic absorbs the off-policy data, using target networks and slowly changing value estimates to stabilize learning
\citep{lillicrap2016ddpg,fujimoto2018td3,rw_haarnoja2018sac}.

We propose an alternative perspective and objective for the use of replay data.
The analysis begins with regularized dual averaging (RDA), which determines the subsequent decision rule using accumulated first-order feedback instead of relying solely on a local update around the latest iterate.
For RDA2C, this feedback is provided by critic-estimated advantages.
From this perspective, an older advantage estimate is one of the observations that defines the cumulative dual objective from which the next policy is computed, and it counts as much as the newest one.
Replay thus serves both as a source of additional optimization steps on prior experience and as a statistical representation of the actor’s optimization problem.

The distinction is most apparent when comparing dual averaging with local mirror descent. Methods such as natural policy gradient, TRPO, PPO, and policy mirror descent can be interpreted as local policy-space updates
\citep{kakade2001natural,schulman2015trust,schulman2017proximal,rw_lan2023pmd}.
Their optimization history is summarized by the most recent policy or actor parameters. In contrast, regularized dual averaging and Follow the Regularized Leader optimize accumulated linearized feedback in conjunction with a regularizer
\citep{nesterov2009primal,xiao2010dual,rw_mcmahan2011ftrl}.
This cumulative formulation does not claim a universal optimization rate advantage over mirror descent. It highlights a concrete statistical object, a historical average of advantage feedback that can be estimated from replay.

We instantiate this idea in Regularized Dual Averaging Actor Critic.
At each iteration, RDA2C stores state-action samples along with
critic-estimated advantage labels in a replay buffer. A dual score model
$Z_\theta(s,a)$ is trained via regression on this aggregated dataset, thereby
estimating a replay-weighted average of historical policy-improvement feedback. The
current policy is subsequently derived directly from the learned score using the
entropy mirror map. Thus the learned dual model provides the policy's natural-coordinate representation, with no separately optimized projection actor.

This construction motivates the phrase \emph{policy as data}. In conventional recursive actor updates, historical improvement information is implicitly compressed into the current network parameters. In RDA2C, the replay buffer explicitly
stores the advantage evidence that defines a moving empirical dual objective, and $Z_\theta$ summarizes this objective. Modifying the contents of the buffer alters the actor objective, and fitting $Z_\theta$ and applying the mirror map computes the policy induced by the updated data.
This perspective turns replay from a mechanism for reusing stale
experience into an explicit representation of the policy-improvement problem (Figure~\ref{fig:image}), leading to the algorithmic, theoretical, and empirical contributions below.\footnote{Code: \url{https://github.com/InsultedByMaths/rda2c}.}

\begin{itemize}
    \item \textbf{A replay-based dual-averaging actor update.}
    We derive an entropy-regularized actor--critic method in which replayed advantage-labeled samples define an empirical cumulative dual objective, a learned score model $Z_\theta$ estimates that objective, and the policy is obtained by an exact mirror-map readout for both finite and continuous action spaces.

    \item \textbf{A finite-time analysis of replayed dual feedback.}
    We give a value-gap decomposition that separates the standard
    RDA optimization term from the errors introduced by a practical replay implementation. The analyzed update closely tracks the implemented one, with the bounds applying to the replayed labels, finite FIFO window, and fitted score model. The decomposition covers local stale-replay drift, supervised fitting error, critic-label bias, finite-buffer variance, and replay coverage.

    \item \textbf{Evaluation across critic constructions and action spaces.}
    We instantiate RDA2C using both GAE and twin-$Q$ advantage labels. Across
    eight MuJoCo-v5 tasks and twelve Atari games, the results show that the
    same replay-based dual-score actor can be used with on-policy-style value
    critics, off-policy twin-$Q$ critics, continuous Gaussian policies, and
    discrete softmax policies.
\end{itemize}

\section{Preliminaries}

\begin{figure}[t]
    \centering
    \includegraphics[width=\linewidth]{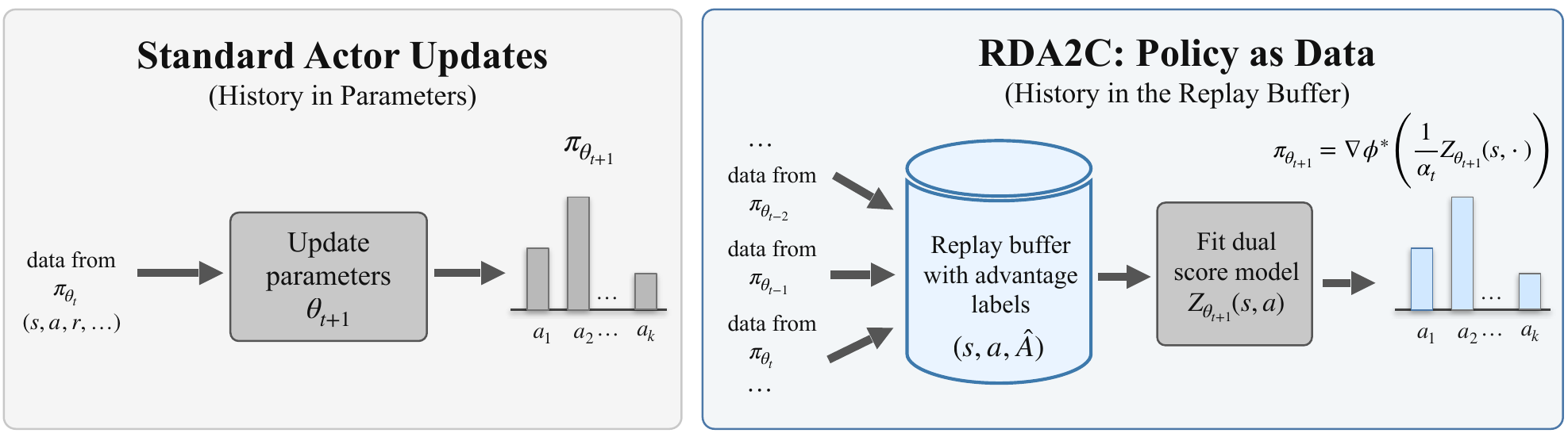}
    \caption{
\textbf{Left} Standard actor updates compress policy-improvement history into the current parameters.
\textbf{Right} In RDA2C, the policy is determined by the replayed data. The buffer stores accumulated advantage evidence, the dual score summarizes that evidence, and the mirror map converts it into the next policy.
}
    \label{fig:image}
\end{figure}

\label{sec:preliminaries}
\label{sec:unified_view}
We consider a discounted Markov decision process $\mathcal M=(\mathcal S,\mathcal A,P,r,\rho_0,\gamma)$, where $\mathcal S$ is the state space, $\mathcal A$ is the action space, $P(\cdot\mid s,a)$ is the transition kernel, $r(s,a)$ is the reward function, $\rho_0$ is the initial state distribution, and $\gamma\in[0,1)$ is the discount factor. A stationary policy $\pi:\mathcal S\to\Delta(\mathcal A)$ maps each state $s$ to an action distribution. We write $\tau=(s_0,a_0,s_1,a_1,\ldots)\sim\pi$ for a trajectory generated by $s_0\sim\rho_0$, $a_t\sim\pi(\cdot\mid s_t)$, and $s_{t+1}\sim P(\cdot\mid s_t,a_t)$. The discounted return objective is $J(\pi)=\mathbb E_{\tau\sim\pi}\left[\sum_{t=0}^{\infty}\gamma^t r(s_t,a_t)\right]$. For a policy $\pi$, the value function, action-value function, and advantage function are $V^\pi(s)=\mathbb E_{\tau\sim\pi}\left[\sum_{t=0}^{\infty}\gamma^t r(s_t,a_t)\,\middle|\,s_0=s\right]$, $Q^\pi(s,a)=r(s,a)+\gamma\mathbb E_{s'\sim P(\cdot\mid s,a)}\left[V^\pi(s')\right]$, and $A^\pi(s,a)=Q^\pi(s,a)-V^\pi(s)$. Equivalently, $V^\pi(s)=\mathbb E_{a\sim\pi(\cdot\mid s)}[Q^\pi(s,a)]$, and hence $\mathbb E_{a\sim\pi(\cdot\mid s)}[A^\pi(s,a)]=0$.
A policy $\pi$ induces the normalized discounted state visitation distribution $d^\pi(s)=(1-\gamma)\sum_{t=0}^{\infty}\gamma^t\mathbb P_\pi(s_t=s)$, where the probability is over trajectories generated by $\pi$. We write $\bar d_t=\frac{1}{t}\sum_{k=1}^t d^{\pi_k}$ for the average historical mixture of visitation distributions, and use $d$ without a policy superscript to denote an arbitrary state density over $\mathcal S$.
We also write $\mu^\pi$ for the induced discounted \emph{state-action} occupancy, with $\mu^\pi(s,a)=d^\pi(s)\pi(a\mid s)$ and $d^\pi(s)=\sum_a\mu^\pi(s,a)$, replacing the sum by an integral for continuous actions.

\section{Background}

Mirror descent is a first-order method for optimization over non-Euclidean geometries \citep{nemirovskij1983problem,beck2003mirror}. It measures movement with a Bregman divergence induced by a strictly convex mirror potential in place of the Euclidean norm. Let \(\varphi:\Delta(\mathcal A)\to\mathbb R\) be a differentiable, strictly convex statewise mirror potential on the action simplex. In this paper, we use the statewise negative entropy \(\varphi(p)=\sum_{a\in\mathcal A}p(a)\log p(a)\) as the mirror potential. The potential on the full policy is its state-weighted average, $\varphi_d(\pi)=\mathbb E_{s\sim d}[\varphi(\pi(\cdot\mid s))]$ for a state weighting $d$, and every statewise statement below extends to the policy space through this average.
For a fixed state, the Bregman divergence induced by this potential is the KL divergence \(D_{\varphi}(p,q)=D_{\mathrm{KL}}(p\|q)\).  The entropy mirror potential also identifies the action probabilities
$\pi(\cdot\mid s)$ as primal variables and policy logits as dual coordinates. We represent the latter by $\log\pi(\cdot\mid s)$, omitting action-independent offsets in $\nabla\varphi$ throughout because they do not affect the map back to the simplex.
Conversely, a dual score vector \(z\in\mathbb R^{|\mathcal A|}\) maps back to the simplex by the inverse entropy mirror map, equivalently \(\nabla\varphi^*(z)\) with \(\varphi^*(z)=\log\sum_{a\in\mathcal A}\exp(z(a))\), yielding \(\pi_z(a\mid s)=\nabla\varphi^*(z)(a)=\frac{\exp(z(a))}{\sum_{b\in\mathcal A}\exp(z(b))}\). We call this map from dual scores to a policy the readout. It is the gradient of the conjugate potential, so it belongs to the mirror geometry, and any method that maintains dual scores obtains its policy through it.

\subsection{Direct functional representation}
\label{subsec:direct_functional}

In the direct functional representation, the policy itself is the state indexed collection of action distributions $\{\pi(\cdot\mid s):s\in\mathcal S\}$. At iteration $t$, actor--critic methods obtain a local functional gradient direction from the critic and use it to update the policy. In policy space, this can be viewed as freezing the current visitation density $d^{\pi_t}$ and forming a linearized improvement objective around $\pi_t$.
With respect to the \(d^{\pi_t}\)-weighted inner product $\langle f,\pi\rangle_{d}=\mathbb E_{s\sim d}[\sum_a f(s,a)\pi(a\mid s)]$ between score functions and policies, the current-policy advantage \(A^{\pi_t}\) represents the functional gradient of \(J\) at \(\pi_t\), satisfying \(\langle \nabla_\pi J(\pi_t),\pi\rangle_{d^{\pi_t}}= \langle A^{\pi_t},\pi\rangle_{d^{\pi_t}}\), up to terms independent of \(\pi\).
We therefore use $A^{\pi_t}$ as the functional gradient representative and write the local linearized improvement term as $\langle A^{\pi_t},\pi\rangle_{d^{\pi_t}}$. Because \(d^{\pi_t}\) changes across rounds, accumulated feedback needs one common coordinate system.

\subsection{Regularized Dual Averaging in Policy Space}
\label{subsec:rda_policy_space}
Regularized dual averaging (RDA) uses the same policy-space mirror geometry as Policy Mirror Descent (PMD) \citep{rw_lan2023pmd}, but applies it to a cumulative linearized objective rather than to a one-step local surrogate \citep{nesterov2009primal,xiao2010dual}.
In its general form, RDA selects $x_{t+1}=\arg\min_x\{\tfrac1t\sum_{k=1}^t\langle g_k,x\rangle+\Psi(x)+\tfrac{\beta_t}{t}\varphi(x)\}$ for loss gradients $g_k$, and the three parts play different roles. The averaged linear terms carry the feedback, and every round contributes one term at equal weight.
The composite term $\Psi$ is persistent, maintaining a constant weight relative to the feedback average and thereby shaping the solution to which the iterates converge.
The stabilizer $\varphi$ damps early iterates, and its relative weight $\beta_t/t$ must vanish for the feedback to take over, which holds whenever $\beta_t$ grows sublinearly. Our instantiation takes both $\Psi$ and $\varphi$ to be entropy, with persistent weight $\lambda$ and, in our experiments, a constant $\beta_t=\beta$. The equal-weight sum over past feedback is also the term our replay window later truncates with a finite window that keeps only recent rounds, trading the variance reduction of a long memory against the staleness of old advantage estimates (Section~\ref{sec:theory}).

In actor--critic methods, the feedback from round $t$ is generated under the visitation distribution $d^{\pi_t}$, so each linear term retains the state weighting of the policy that produced it. To regularize different policies on the same historical state support, we weight by the historical mixture $\bar d_t$ from Section~\ref{sec:preliminaries}. The resulting state-weighted policy-space RDA update is
\begin{equation}
\label{eq:rda_update}
    \pi_{t+1}
    =
    \arg\max_{\pi\in\Pi}
    \left\{
        \frac{1}{t}\sum_{k=1}^t \langle A^{\pi_k},\pi\rangle_{d^{\pi_k}}
        -\Psi_{\bar d_t}(\pi)
        -\frac{\beta_t}{t} \varphi_{\bar d_t}(\pi)
    \right\},
\end{equation}
where $\Psi_{\bar d_t}$ is the persistent composite regularizer, $\varphi_{\bar d_t}$ is a strongly convex stabilizer, and $\beta_t\ge 0$ is the stabilizer weight.
With $g_k=-A^{\pi_k}$, minimizing the averaged loss is maximizing the averaged advantage, which gives the sign convention of \eqref{eq:rda_update}.

\section{Regularized Dual Averaging Actor Critic}
\label{sec:algorithm}

We now present RDA2C, a replay-based actor--critic algorithm that instantiates the policy space RDA update from the previous section. The central object in RDA2C is a dual score function $Z_\theta:\mathcal S\times\mathcal A\to\mathbb R$, which approximates the
state-density-weighted average advantage $Z_t$ defined in
Section~\ref{subsec:dual_average_estimation}. Throughout, a score is a function on state-action pairs read in dual (logit) coordinates,  not the
likelihood score $\nabla_\theta\log\pi_\theta$ used in policy gradient methods.
RDA2C estimates the cumulative policy improvement signal required by RDA and then converts this score directly into a policy using entropy geometry.

Concretely, at each policy round the critic attaches an advantage label to each sampled state--action pair. These labeled examples are retained for several recent rounds, and $Z_\theta$ is trained by supervised regression to predict their replay-weighted average.
The policy is then obtained directly from the learned score through the entropy mirror map. RDA2C thus separates the actor update into advantage-label generation, dual-score regression, and mirror-map readout.

\subsection{Deriving the Dual Function}
\label{subsec:dual_average_estimation}

RDA is written in terms of an average dual object, so we first rewrite the accumulated, state weighted exact advantage feedback as a single inner product under the historical state mixture.
For states with $\bar d_t(s)>0$, define the density weighted average advantage score by
\begin{equation}
    \label{eq:density_weighted_average_advantage_score}
    \frac{1}{t}
    \sum_{k=1}^t
    \langle A^{\pi_k},\pi\rangle_{d^{\pi_k}}
    =
    \langle Z_t,\pi\rangle_{\bar d_t},
    \quad
    \textrm{where}
    \quad
    Z_t(s,a)
    \coloneqq
    \frac{
        \sum_{k=1}^t d^{\pi_k}(s)A^{\pi_k}(s,a)
    }{
        \sum_{k=1}^t d^{\pi_k}(s)
    } .
\end{equation}
 The cumulative advantage feedback is then represented by the single average dual score \(Z_t\).

For \textsc{RDA2C}, we set the composite regularizer in \eqref{eq:rda_update} as \(\Psi_{\bar d_t}=\lambda\varphi_{\bar d_t}\). Both terms use the negative-entropy geometry, with the roles described in Section~\ref{subsec:rda_policy_space}, where \(\lambda\varphi_{\bar d_t}\) is the persistent entropy floor, and \((\beta_t/t)\varphi_{\bar d_t}\) is the transient stabilizer. Together they give the effective entropy temperature \(\alpha_t=\lambda+\beta_t/t\). Thus, after rewriting the normalized cumulative advantage feedback as \(\langle Z_t,\pi\rangle_{\bar d_t}\), the RDA update becomes \(\pi_{t+1}=\arg\max_{\pi\in\Pi}\left\{\langle Z_t,\pi\rangle_{\bar d_t}+\alpha_t H_{\bar d_t}(\pi)\right\}\), where \(H_{\bar d_t}=-\varphi_{\bar d_t}\).

Although the density-weighted objective in \eqref{eq:rda_update} is written over the product policy space, it decomposes statewise. For any state with \(\bar d_t(s)>0\), the same positive scalar \(\bar d_t(s)\) multiplies both the linear score term and the entropy term, so it does not affect the optimizer at that state, provided the score vector \(Z_t(s,\cdot)\) can be set independently at each state, as in the tabular case. The density weights therefore determine how the average score \(Z_t\) is formed and which states are covered, and once \(Z_t(s,\cdot)\) is fixed, the policy at state \(s\) is obtained by the entropy mirror-map readout. With function approximation the weights matter again, because the approximator trades larger error at low-weight states for smaller error at high-weight ones. For finite action spaces, the entropy conjugate readout $\nabla\varphi^*$ is the softmax map. For continuous action spaces, it is the normalized Boltzmann readout. We use the quadratic score class $Z_\theta(s,a)=-\tfrac12\sum_i (a_i-\mu_{\theta,i}(s))^2/\sigma_{\theta,i}^2(s)+b_\theta(s)$, for which this readout is the diagonal Gaussian $\mathcal N\big(\mu_\theta(s),\alpha_t\operatorname{diag}(\sigma_\theta^2(s))\big)$, and the bias $b_\theta(s)$ cancels in the normalization (Appendix~\ref{subsec:continuous_actions}). Solving the statewise objective in closed-form gives the readout:
\begin{equation}
\label{eq:closed_form_final}
    \pi_{t+1}^{\mathrm{RDA}}(\cdot \mid s)
    =
    \operatorname{softmax}\Bigg(\frac{Z_t(s,\cdot)}{\alpha_t}\Bigg)
\end{equation}

\begin{algorithm}[tb]
   \small
\caption{\textsc{RDA2C}: Regularized Dual Averaging Actor Critic}
   \label{alg:rda2c}
\begin{algorithmic}
   \STATE {\bfseries Input:} replay capacity or window schedule, regularization $\lambda$, stabilizer schedule $\{\beta_t\}_{t\ge1}$
   \STATE Initialize dual score model $Z_\theta$, critic learner $\mathcal C$, replay buffer $\mathcal B_0=\emptyset$, and initial policy $\pi_1$
   \FOR{iteration $t=1,\ldots,T$}
       \STATE Collect trajectories using current policy $\pi_t$
       \STATE Compute advantage labels $\hat A_t$ with critic learner $\mathcal C$ \hfill{\emph{e.g., GAE or twin-$Q$ labels}}
       \STATE Form labeled data $\mathcal D_t=\{(s_i,a_i,\hat A_i)\}$ and push $\mathcal D_t$ into $\mathcal B_t$ \hfill{\emph{evicting old data as needed}}
       \STATE Update critic learner $\mathcal C$ using the backbone's return or TD targets
       \STATE Set temperature $\alpha_t = \lambda + \beta_t/t$
       \STATE Fit $Z_\theta$ on minibatches from $\mathcal B_t$ using \eqref{eq:dual_regression_loss}
       \STATE Read out the next policy $\pi_{t+1}=\nabla\varphi^*(Z_{\theta_t}/\alpha_t)$
   \ENDFOR
\end{algorithmic}
\end{algorithm}

\subsection{Policy as data: fitting the aggregated advantage}
\label{subsec:policy_as_data}

The update in \eqref{eq:closed_form_final} is defined through the average score $Z_t$ from \eqref{eq:density_weighted_average_advantage_score}, and no algorithm can maintain that object exactly, because it averages full advantage functions over rounds. RDA2C realizes the average statistically instead, through dataset aggregation \citep{ross2011reduction}. Each round contributes advantage-labeled samples, a sliding window retains the recent rounds, and one squared-loss regression fits a single score model to everything retained. The regression does the averaging. The mirror map does the readout. The full loop is Algorithm~\ref{alg:rda2c}.

Three variants of the algorithm follow from what kind of advantage feedback is available. The first stores the full vector $A(s,\cdot)$ at each visited state. It shows most directly how regression performs dual averaging: when a state is visited in several rounds, the regression drives $Z_\theta(s,\cdot)$ toward the average of the stored vectors. It requires enumerating the action set and a critic that can be queried at every action, including actions the recent policies never took. The second stores only the action taken at each visited state, which is all a value critic can label. Regression then targets the behavior-weighted average at each pair instead of the unweighted one, an extra error that the analysis in Section~\ref{sec:theory} accounts for, in exchange for feasibility with large action spaces and on-policy labels. The third sits between the two. With a $Q$ critic, we score two actions sampled from the current policy at each replayed state, one as the label and one as the baseline, which gives a one-sample advantage label without enumeration. The GAE settings on MuJoCo and Atari use the second variant, and the twin-$Q$ MuJoCo setting uses the third. Appendix~\ref{app:implementation_details} gives the label construction.

Formally, for the second variant, at iteration $t$ the current policy $\pi_t$ generates labeled samples $\mathcal D_t=\{(s_{t,i},a_{t,i},\widehat A_{t,i})\}_{i=1}^{n_t}$, and a FIFO or sliding window retains the recent rounds, $\mathcal B_t=\mathrm{Window}(\mathcal D_1,\ldots,\mathcal D_t)$. At actor update time, minibatches are drawn from the realized buffer, so the sampled tuples follow the buffer's empirical distribution. Let $P_t^{\mathrm{rep}}$ denote the population replay law over labeled triples $(s,a,\widehat A)$ induced by the window at iteration $t$. The population dual regression problem is to minimize
\begin{equation}
\label{eq:dual_regression_loss}
    \mathcal L_{\mathrm{dual}}(\theta,t)
    =
    \mathbb E_{(s,a,\widehat A)\sim P_t^{\mathrm{rep}}}
    \left[(Z_{\theta}(s,a)-\widehat A)^2\right].
\end{equation}
In implementation, this loss is approximated by minibatches from $\mathcal B_t$. The reported GAE runs on MuJoCo and Atari replace the squared loss by the clipped variant of Appendix~\ref{subsec:appendix_clipped_dual_regression}, which the ablation in Appendix~\ref{app:ablation_adv_loss} shows matters in practice. The twin-$Q$ runs use the squared loss on standardized labels.

For the squared-loss objective, the population regression target is the
replay conditional mean label $\widehat Z_t(s,a) = \mathbb E_{P_t^{\mathrm{rep}}}[\widehat A\mid S=s,A=a]$.
The regression therefore fits the replay Bayes target $\widehat Z_t$, and fitting on a finite buffer introduces additional sampling error.
This is the sense in which the policy is represented by data, with the buffer defining an empirical dual objective, $Z_\theta$ serving as its fitted representation, and the entropy readout of $Z_\theta$ giving the current policy.

The action weighting is the one gap the analysis closes by assumption: the bounds charge it to a coverage coefficient, which is finite when the replayed policies put positive probability on every action at every state the comparator visits (Appendix~\ref{app:layer2_coverage}).

\subsection{Connections to Key Baselines}
\label{subsec:baseline_connections}
We highlight the algorithmic differences of the baselines most relevant to our method; a broader discussion of related work appears in Appendix~\ref{app:additional_related}.

\paragraph{Connection to SAC.}
SAC and the twin-$Q$ variant of RDA2C use the same critic-learning
procedure but differ in their actor targets and update mechanisms.
SAC's actor tracks the softmax of the latest $Q$ estimate, whereas
RDA2C derives its policy from a fitted replay-weighted average of
recent critic advantage labels. SAC optimizes a KL projection objective by stochastic gradient through reparameterized action samples $a=f_\phi(\epsilon;s)$, which requires continuous actions \citep{rw_haarnoja2018sac}. RDA2C
instead fits a quadratic dual score to stored labels by supervised
regression and obtains its policy through the entropy mirror map.

Despite these differences, we hypothesize comparable performance
under matched critic learning. \citet{rw_tomar2022mdpo} show that SAC's actor step is an entropy-regularized follow-the-regularized-leader update, the same form as the regularized dual averaging step behind RDA2C, so with the critic shared the two methods differ only in how the cumulative score is formed. They also report that off-policy MDPO-KL performs on par with SAC across all six MuJoCo tasks in their evaluation, providing an empirical precedent for comparable performance despite differences in the actor update.

This parity is useful because RDA2C is not tied to a particular critic construction or action space. The same actor update can use either
GAE labels from a value critic or advantage labels from twin-$Q$
critics, and applies to both discrete and continuous action spaces.
SAC, by contrast, relies on a soft $Q$-critic construction for its actor update and is therefore not directly applicable in settings where only trajectory-based GAE labels are available, and its reparameterized actor step does not extend to discrete actions. SAC is not applicable to discrete action settings without a separate derivation \citep{rw_christodoulou2019sacdiscrete}, and the resulting discrete variant needs its own stabilization \citep{rw_zhou2022discretesac}. Our results
show that RDA2C improves over PPO in the GAE setting while matching
SAC when twin-$Q$ labels are available. Parity at SAC's own strength, together with gains over PPO where SAC does not apply, makes RDA2C the better global default across the two suites.

\paragraph{Connection to AAPDA.}
AAPDA \citep{rw_ju2026aapda} and RDA2C both build on policy dual
averaging, but differ in how they obtain the policy from the
accumulated feedback. AAPDA learns a separate actor to approximate
the policy dual-averaging target. RDA2C instead fits the dual score
from replayed advantage labels and obtains the policy directly
through the entropy mirror map, without a separate actor-fitting
stage.

\section{Theoretical Analysis}
\label{sec:theory}

The analysis is for finite action spaces and separates the entropy-RDA optimization term from the replay score errors introduced by the fitted actor. The fitted actor is the entropy readout of the learned replay score. As in \eqref{eq:closed_form_final}, $\pi_{\theta_t}$ is the policy that Algorithm~\ref{alg:rda2c} reads out at the end of round $t-1$ from $Z_{\theta_{t-1}}$ with temperature $\alpha_{t-1}$,
\begin{equation}
\label{eq:fitted_readout_theory}
    \pi_{\theta_t}(\cdot\mid s)
    =
    \operatorname{softmax}\!\left(Z_{\theta_{t-1}}(s,\cdot)/\alpha_{t-1}\right),
    \qquad
    \alpha_{t-1}=\lambda+\beta/(t-1),\quad \lambda>0,\ \beta\ge0.
\end{equation}
Here \(\lambda\) is the persistent entropy floor and \(\beta/(t-1)\) is the decaying RDA stabilizer. The analysis specializes the stabilizer schedule $\{\beta_t\}$ of Algorithm~\ref{alg:rda2c} to a constant, $\beta_t\equiv\beta$, giving $\alpha_t=\lambda+\beta/t$. The proof converts policy-space regret into value loss through the performance difference lemma in Appendix~\ref{app:performance_difference_lemma}.

\paragraph{Replay window and regression bridge.}
The theory analyzes a replay window $\mathcal W_t=\{t-W_t,\dots,t-1\}$ of length $W_t=\max\{1,\lfloor t^\zeta\rfloor\}$. Using the state-action occupancy notation from Section~\ref{sec:preliminaries}, set $\mu_k=\mu^{\pi_{\theta_k}}$, i.e., $\mu_k(s,a)=d^{\pi_{\theta_k}}(s)\pi_{\theta_k}(a\mid s)$, and let $\mu_t(s,a)=W_t^{-1}\sum_{k\in\mathcal W_t}\mu_k(s,a)$ be the population replay marginal. The squared-loss replay regression fits $Z_{\theta_{t-1}}$ to the replay Bayes target of \eqref{eq:dual_regression_loss} and, with true advantage labels $A_k=A^{\pi_{\theta_k}}$, this target is
$
    Z_t^A(s,a)=\big(\textstyle\sum_{k\in\mathcal W_t}\mu_k(s,a)A_k(s,a)\big)\big/\big(\sum_{k\in\mathcal W_t}\mu_k(s,a)\big),
$
on the support of $\mu_t$. Thus $Z_t^A$ is the replay-weighted average advantage that bridges the current advantage $A_t$ and the learned score $Z_{\theta_{t-1}}$, where $A_k=A^{\pi_{\theta_k}}$ are the advantages of the learned policies. The comparator coverage by replay is measured by $C_{\mathrm{cov}}(\pi^\star)=\max_{1\le t\le T}\operatorname*{ess\,sup}_{s,a} d^{\pi^\star}(s)/\mu_t(s,a)$, which is finite if $\mu_t(s,a)>0$ for every action $a$ at every state with $d^{\pi^\star}(s)>0$, a universal action coverage condition at comparator states.

We first state the bound for the regularized value gap, the average regret against the comparator $\pi^\star$ with the persistent entropy correction kept on the left-hand side:
$$
    \mathcal G_T^{\mathrm{reg}}(\pi^\star)
    \coloneqq
    \frac1{T-1}\sum_{t=2}^T
    \left[
    \big(J(\pi^\star)-J(\pi_{\theta_t})\big)
    +
    \frac{\lambda}{1-\gamma}
    \mathbb E_{s\sim d^{\pi^\star}}
    \big(
        \varphi(\pi_{\theta_t}(\cdot\mid s))-\varphi(\pi^\star(\cdot\mid s))
    \big)
    \right].
$$

\begin{theorem}[Fitted replay value bound]
\label{thm:fitted_replay_value_bound}
Assume $\mathcal A$ finite, $|r(s,a)|\le R_{\max}$, the policy path stability of Assumption~\ref{ass:learned_policy_path_stability_clean} (constant $C_\pi$) and the replay error bounds of Assumption~\ref{ass:replay_error_decomp} (Appendix~\ref{app:theory_setup}), and $C_{\mathrm{cov}}(\pi^\star)<\infty$. Let $\zeta\in(0,1)$ and $C_{\mathrm{loc}}=\frac{(1+\gamma)R_{\max}C_\pi}{(1-\gamma)^2(1-2^{\zeta-1})}$. Then for every stationary comparator $\pi^\star$ satisfying the above assumptions,
$$
\begin{aligned}
\mathcal G_T^{\mathrm{reg}}(\pi^\star)
\le \frac{1}{1-\gamma}\Bigg[\,
&\frac{\beta\ln|\mathcal A|(1+\ln T)}{T-1}
+\frac{2C_{\mathrm{loc}}}{\zeta}\frac{T^\zeta}{T-1}\\[-0.4ex]
&+2\sqrt{C_{\mathrm{cov}}(\pi^\star)}\Big(\bar\epsilon_{\mathrm{fit}}+\epsilon_0+\tfrac{C_{\mathrm{var}}}{1-\zeta/2}\tfrac{T^{1-\zeta/2}}{T-1}\Big)
\Bigg]
=:\mathcal B_T(\pi^\star).
\end{aligned}
$$
\end{theorem}

Converting the regularized gap to ordinary value costs an entropy bias term. The conversion turns the guarantee into an ordinary average regret bound against $\pi^\star$, comparable with unregularized methods, with an explicit bias term that scales with $\lambda$ and represents the price of the entropy floor retained on the left-hand side of the regularized bound.

\begin{corollary}[Unregularized value bound]
\label{cor:unregularized_rate}
Under the assumptions of Theorem~\ref{thm:fitted_replay_value_bound}, moving the entropy correction to the right gives
$$
    \frac1{T-1}\sum_{t=2}^T \big(J(\pi^\star)-J(\pi_{\theta_t})\big)
    \le
    \frac{\lambda\ln|\mathcal A|}{1-\gamma}
    +
    \mathcal B_T(\pi^\star),
$$
with $\mathcal B_T(\pi^\star)$ the right hand side of Theorem~\ref{thm:fitted_replay_value_bound}. With $\zeta=2/3$ and fixed $\lambda>0$ the ordinary value gap converges at rate $T^{-1/3}$ to the floor set by the entropy bias $\frac{\lambda\ln|\mathcal A|}{1-\gamma}$ and the persistent fitting and critic error $\frac{2\sqrt{C_{\mathrm{cov}}(\pi^\star)}}{1-\gamma}(\bar\epsilon_{\mathrm{fit}}+\epsilon_0)$.
\end{corollary}

\section{Experiments}
\label{sec:experiments}
We evaluate whether the replay estimator proposed in Section~\ref{sec:algorithm} gives a useful actor update in both continuous and discrete control. The experiments ask five questions: whether RDA2C improves over PPO when both methods use GAE labels, whether the same actor holds parity with SAC when labels come from twin-\(Q\) critics, whether it improves on AAPDA, the prior dual-averaging actor, at that method's released defaults, whether the same update carries over to image-based discrete control on Atari, and whether the replay window behaves like the finite replay and drift terms in Section~\ref{sec:theory}. Additional ablations and control experiments in Appendix~\ref{app:additional_ablations} examine the effects of individual algorithmic components and implementation choices.

\subsection{Experimental setup}

\begin{figure}[t]
    \centering
    \includegraphics[width=\linewidth]{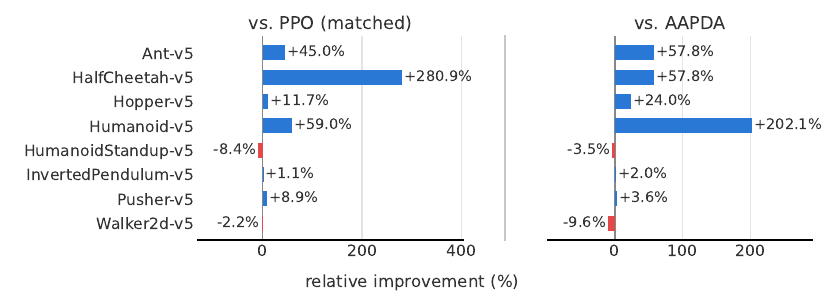}

    \caption{\textbf{RDA2C improves over matched PPO on six of eight MuJoCo tasks and over AAPDA on six of eight.} Bars show per-environment relative improvement, $100\cdot(\bar R_{\mathrm{ours}}-\bar R_{\mathrm{base}})/|\bar R_{\mathrm{base}}|$. Left: PPO matched to RDA2C's architecture and hyperparameters (GAE labels, online return). Right: AAPDA at its released defaults (deterministic evaluation). Exact values: Appendix~\ref{app:main_result_tables}.}
    \label{fig:mujoco_summary}
\end{figure}

\paragraph{Environments and evaluation.}
For the continuous control suite, we use eight Gymnasium MuJoCo v5 tasks \citep{todorov2012mujoco}.
Both the GAE and twin-$Q$ critic MuJoCo comparisons run for 1M environment steps and use 10 random seeds. The SAC comparison reports deterministic mean action evaluation over 20 episodes every 50K steps. 
For the Atari suite, we evaluate on twelve Atari 2600 games \citep{bellemare13arcade} for 10M frames with 10 random seeds, using the EnvPool vectorized environment wrapper \citep{weng2022envpool}. Across both suites, we follow the evaluation and reporting conventions used in the corresponding baseline literature, as detailed in Appendix~\ref{app:main_result_tables}.

\paragraph{Baselines and model classes.}

The GAE-based MuJoCo comparison uses PPO \citep{schulman2017proximal}
as the baseline. We match the $2\times128$ trunks, learning rate,
gradient clipping, GAE $\lambda$, batch sizes, update epochs, and
horizon, so the actor update is the intended difference.
Both methods use a value network separate from the actor. RDA2C derives its policy from a dual score that is quadratic in the actions (Appendix~\ref{subsec:continuous_actions}), whereas PPO directly parameterizes a Gaussian policy with a global log-standard-deviation.

The twin-$Q$ MuJoCo comparison uses SAC \citep{rw_haarnoja2018sac}
with matched transition-replay settings, twin-$Q$ critic architectures,
target updates, optimizer, batch size ($512$), and policy update
frequency ($1$). 
RDA2C retains the critic-learning procedure but replaces
SAC's actor update by fitting a quadratic dual score to replayed
advantage labels and obtaining the policy through the entropy mirror map.

We also compare against actor-accelerated policy dual averaging
(AAPDA) \citep{rw_ju2026aapda} on MuJoCo. AAPDA learns a separate
actor to approximate the policy dual-averaging target, whereas
RDA2C reads its policy directly from the fitted dual score.
We use AAPDA's released defaults without
per-task tuning, with both methods trained for 1M environment steps
over 10 paired training seeds. The evaluation protocol is detailed
in Appendix~\ref{app:main_result_tables}.

For Atari, RDA2C uses a discrete softmax readout and separate convolutional feature extractors for the score model and value function, while PPO keeps the shared trunk of its standard implementation \citep{huang2022ppo37}.

\subsection{Experimental results}

\paragraph{GAE labels: RDA2C improves over PPO on six of eight MuJoCo tasks} Figure~\ref{fig:mujoco_summary} summarizes the relative improvement over matched PPO, and Appendix Table~\ref{tab:mujoco_results} reports the raw mean and standard deviation over 10 seeds. RDA2C improves over matched PPO on six of eight MuJoCo tasks and gives a fixed-suite average per-environment relative improvement of $+49.5\%$ (paired seed bootstrap 95\% CI $[+42.1\%,+56.2\%]$).\footnote{\label{fn:bootstrap}Percentile bootstrap intervals can undercover at seed counts like ours: \citet{jordan2020evaluating} measure failure rates above twice the nominal level, and in a simulation matched to our data the suite-level (mujoco and atari suites) procedure used here covered a nominally 95\% target 90.5\% of the time. We therefore read these intervals as descriptive uncertainty summaries. This caveat applies to every bootstrap interval in the paper.} The largest gains are on \texttt{HalfCheetah-v5}, where RDA2C reaches $7099$ compared with PPO's $1864$, and \texttt{Humanoid-v5}, where RDA2C reaches $3768$ compared with PPO's $2370$. These gains are obtained against a strengthened PPO baseline. The matched configuration used here has higher mean returns than PPO's CleanRL defaults on five of eight tasks (Appendix~\ref{app:default_baselines}).

\paragraph{Twin-$Q$ critic labels: RDA2C matches SAC under matched settings}
\label{subsubsec:twinq}
RDA2C is higher on four tasks, ties on one, and is lower on three. Appendix Table~\ref{tab:mujoco_sac_results} reports the underlying best deterministic evaluations within the 1M-step budget over 10 seeds. The suite-level summary is a fixed-suite average per-environment relative change of $+0.64\%$ (paired seed bootstrap 95\% CI $[-2.80\%,+4.06\%]$)\textsuperscript{\ref{fn:bootstrap}}. The largest single-task difference favors RDA2C on \texttt{Walker2d-v5} ($4505$ vs $4080$), and the largest in SAC's favor is on \texttt{HalfCheetah-v5} ($9525$ vs $10119$). No taskwise difference is decisive at 10 seeds. The outcome matches the expectation in Section~\ref{subsec:baseline_connections}, and it answers the question the experiment asks: the same dual-score regression works with twin-$Q$ labels in place of GAE, at SAC's level of return. The comparison against SAC at its CleanRL-default batch size and update frequency appears in Appendix~\ref{app:default_baselines}.

\paragraph{AAPDA: RDA2C improves on the closest dual-averaging baseline}
\label{subsubsec:aapda}

AAPDA is the closest runnable dual-averaging baseline \citep{rw_ju2026aapda}, described in Appendix~\ref{sec:related-work}. We ran the authors' released defaults for 1M steps with no per-task tuning. RDA2C is higher on six of eight tasks. The right panel of Figure~\ref{fig:mujoco_summary} summarizes the comparison, and Appendix Table~\ref{tab:aapda_results} gives the per-task returns. The fixed-suite mean relative change is $+41.8\%$ (paired seed bootstrap 95\% CI $[+26.3\%,+71.2\%]$)\textsuperscript{\ref{fn:bootstrap}} and the median is $+13.8\%$ (CI $[+4.9\%,+27.0\%]$)\textsuperscript{\ref{fn:bootstrap}}. AAPDA is ahead on \texttt{HumanoidStandup-v5} and \texttt{Walker2d-v5}.

\paragraph{Atari: RDA2C improves over PPO on eight of twelve games}

Figure~\ref{fig:atari_summary} reports relative improvement over PPO after 10M frames, and Appendix Table~\ref{tab:atari_comparison} gives the raw returns. RDA2C improves over PPO on 8 of the 12 Atari games, with a fixed-suite average per-environment relative improvement of $+112.4\%$
(paired seed bootstrap 95\% CI $[+88.5\%,+141.2\%]$)\textsuperscript{\ref{fn:bootstrap}} and a median relative improvement of $+32.9\%$. The largest gains are on \texttt{Seaquest-v5}, \texttt{Enduro-v5}, and \texttt{BeamRider-v5}. PPO is higher on \texttt{Breakout-v5}, \texttt{Freeway-v5}, \texttt{Pong-v5}, and \texttt{Qbert-v5}. The Atari results show that the same average advantage parameterization can scale to image-based discrete control. The table reports mean running return and seed standard deviation over 10 random seeds.

\begin{figure}[t]
    \centering
    \includegraphics[width=\linewidth]{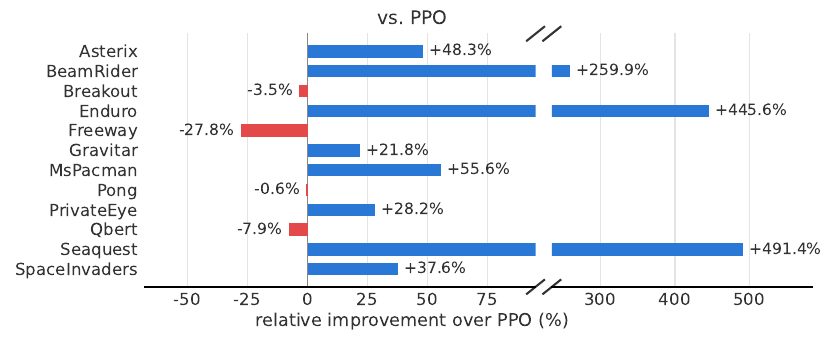}

    \caption{\textbf{RDA2C improves over PPO on eight of twelve Atari games.} Bars show relative improvement over PPO, $100\cdot(\bar R_{\mathrm{RDA2C}}-\bar R_{\mathrm{PPO}})/|\bar R_{\mathrm{PPO}}|$, after 10M frames. The value axis is broken between $+95\%$ and $+235\%$. The bars for \texttt{BeamRider-v5}, \texttt{Enduro-v5}, and \texttt{Seaquest-v5} continue in the right segment. Exact mean $\pm$ std values are reported in Appendix~\ref{app:main_result_tables}, Table~\ref{tab:atari_comparison}.}
    \label{fig:atari_summary}
\end{figure}

\section{Conclusion and Future Work}

We introduced RDA2C, a replay-based actor--critic implementation of regularized dual averaging via advantage regression. In RDA2C, replayed advantage-labeled samples define an empirical cumulative dual objective, which is summarized by a learned score model. The policy is then obtained directly through the entropy mirror map. Replay then serves both as a source of additional optimization steps and as the representation of policy-improvement feedback from which the current actor is derived.

\newpage
\label{sec:mainend}
\subsection*{AI use statement}
In this work, we used generative AI tools to implement analysis and sweep code. Additionally, we used generative AI tools to edit parts of the paper for readability and structure and to create and modify figures. We have reviewed all AI-assisted work. All experiments were run by the authors, all AI-generated code was checked and tested by the authors, and all AI-drafted text was revised by the authors. We take responsibility for the final content of this work, including text, claims, or artifacts produced with the aid of generative AI.

\subsection*{Reproducibility statement}
Appendix~\ref{app:implementation_details} gives network architectures, hyperparameters for every setting, and the label constructions. Appendix~\ref{app:main_result_tables} reports the numerical results behind every figure, with seed counts and evaluation protocols. The assumptions and complete proofs of the theoretical results are in Appendix~\ref{app:proofs}. The code for all three entrypoints follows the CleanRL single-file convention and is available at \url{https://github.com/InsultedByMaths/rda2c}.

\subsection*{Acknowledgments}
KB and NP acknowledge this work has been made possible in part by a gift from the Chan Zuckerberg Initiative Foundation to establish the Kempner Institute for the Study of Natural and Artificial Intelligence.

\bibliography{example_paper}
\bibliographystyle{iclr2027_conference}

\clearpage
\appendix

\makeatletter
\let\app@original@addcontentsline\addcontentsline
\renewcommand{\addcontentsline}[3]{%
  \app@original@addcontentsline{#1}{#2}{#3}%
  \def\app@toc{toc}%
  \def\app@target{#1}%
  \ifx\app@toc\app@target
    \app@original@addcontentsline{apc}{#2}{#3}%
  \fi
}
\makeatother

\section*{Appendix Contents}
\begingroup
\small
\setcounter{tocdepth}{2}
\makeatletter
\@starttoc{apc}
\makeatother
\endgroup

\clearpage

\FloatBarrier
\section{Implementation Details}
\label{app:implementation_details}
\FloatBarrier
\subsection{Network Architectures}
We utilize standard architectures for continuous and discrete control domains and report the relevant capacity differences explicitly. For GAE MuJoCo, RDA2C uses the quadratic dual score actor described below with a separate critic, while PPO uses the standard \href{https://github.com/vwxyzjn/cleanrl}{CleanRL} Gaussian policy/value architecture.

\paragraph{Continuous control (MuJoCo, GAE labels).}
For the GAE MuJoCo experiments, the actor is the quadratic dual score network described in Section~\ref{sec:algorithm}. It has a two-layer Tanh MLP trunk of width 128 and three heads: Gaussian mean, state dependent log standard deviation, and scalar bias $b_\theta(s)$. The value function is a separate two-layer Tanh MLP of width 128. During rollout collection we also store the old dual score prediction $Z_{\mathrm{old}}(s,a)$ for each sampled transition, which is used only by the optional clipped dual regression loss of Appendix~\ref{subsec:appendix_clipped_dual_regression}.
\begin{itemize}
    \item \textbf{Dual Score / Policy Network:} Input state $\rightarrow$ Linear(128) $\rightarrow$ Tanh $\rightarrow$ Linear(128) $\rightarrow$ Tanh $\rightarrow$ Mean / LogStd / Bias heads.
    \item \textbf{Value Network:} Input state $\rightarrow$ Linear(128) $\rightarrow$ Tanh $\rightarrow$ Linear(128) $\rightarrow$ Tanh $\rightarrow$ Linear(1).
\end{itemize}

\paragraph{Update schedule (MuJoCo, GAE labels).}
Each iteration collects one rollout of $2048$ steps from a single environment, computes GAE labels, and pushes the $2048$ tuples $(s,a,\hat A,Z_{\mathrm{old}})$ into the FIFO label buffer of capacity $W=50{,}000$, which therefore holds the most recent $24$ rollouts once full. The stored action is the unsquashed Gaussian sample. The environment receives that sample clipped to the action bounds, so the policy is a clipped Gaussian and the regression is fit on the pre-clip action. The dual regression then runs $10$ epochs of $32$ minibatches of $64$ tuples, $320$ gradient steps per iteration. The $20{,}480$ tuples are drawn uniformly with replacement from the buffer once per iteration, standardized as a batch, shuffled, and consumed sequentially, so a stored label is used about $0.4$ times per iteration on average and about $10$ times over its residence in the buffer. The value loss in the same gradient step uses the on-policy rollout minibatch of the same size. The temperature $\alpha_t=\lambda+\beta/t$ is computed once per rollout, so with $\lambda=0.05$ and $\beta=1$ it decays from $1.05$ to $\lambda+\beta/488$ over the $488$ rollouts that make up 1M steps.

\paragraph{Continuous control (MuJoCo, twin-$Q$ critic labels).}
For the SAC comparison, both SAC and RDA2C with twin-$Q$ labels use the standard CleanRL SAC backbone: two-layer MLPs with width 256 and ReLU activations, together with twin-$Q$ critics. The SAC baseline uses the standard squashed Gaussian actor and default SAC hyperparameters. Our RDA2C with twin-$Q$ labels variant reuses the same replay / critic backbone and the same width actor trunk, but replaces the SAC actor loss with regression onto the learned quadratic dual score from Section~\ref{sec:algorithm}. Concretely, the RDA2C with twin-$Q$ labels actor has mean, log standard deviation, and scalar bias heads, while each $Q$-network maps $(s,a)$ through 256-256 ReLU layers to a scalar estimate.

The reported twin-$Q$ critic RDA2C runs use \emph{automatic temperature tuning} in place of the hand-designed RDA schedule used in the GAE experiments. Specifically, for RDA2C twin-$Q$ we optimize a scalar $\log \alpha_t$ against target entropy $-|A|$. This learned temperature determines the sampling scale of the squashed Gaussian policy through $\sigma(s)\sqrt{\alpha_t}$ and enters both the soft critic bootstrap target and the actor regression labels. For $K=2$ current-policy samples $a_1,\ldots,a_K$ at a replayed state $s$, we define
$$
    x_i(s)=\min(Q_1,Q_2)(s,a_i)-\alpha_t\log\pi_\theta(a_i\mid s),
    \qquad
    g_i(s)=x_i(s)-\frac{1}{K-1}\sum_{j\ne i}x_j(s).
$$
The regression buffer stores $(s,a_i,g_i)$, where $a_i$ is the corresponding pre-Tanh Gaussian sample. We store raw labels in a FIFO label buffer of total size $W=12{,}800$, standardize them within sampled regression minibatches, clip the standardized targets to $[-10,10]$, and regress the quadratic dual score $Z_\theta(s,A_i)$ to these buffered targets. In the primary comparison both methods use batch size 512 and policy frequency 1. SAC at its CleanRL defaults, batch size 256 and policy frequency 2, is reported separately in Appendix~\ref{app:default_baselines}. Both methods use 5 parallel environments and deterministic evaluation over 20 episodes. Because autotuning is enabled throughout this comparison, the RDA schedule parameters $(\lambda,\beta)$ are inactive.

All MuJoCo runs used 2 CPU cores and no GPUs. All Atari runs used a single H100 GPU.

\FloatBarrier
\subsection{Action classes and label settings}
\paragraph{Handling discrete and continuous actions.}

\label{subsec:continuous_actions}
For discrete actions, \(Z_\theta(s)\) outputs one score per action, and the entropy readout is the corresponding softmax policy. For continuous control, we choose a quadratic dual score class,
\(Z_\theta(s,a)=-\tfrac12\sum_i \tfrac{(a_i-\mu_{\theta,i}(s))^2}{\sigma_{\theta,i}^2(s)}+b_\theta(s)\).
Under the same entropy mirror map, \(\operatorname{softmax}(Z_\theta(s, \cdot) / \alpha_t)\)
is the diagonal Gaussian \(\mathcal N\left(\mu_\theta(s), \alpha_t\,\operatorname{diag}(\sigma_\theta^2(s))\right)\). The additive term \(b_\theta(s)\) cancels in the normalization and does not affect the policy.

\paragraph{Labels, replay windows, and stabilization.}

The GAE label variants in MuJoCo and Atari compute advantage labels from a value critic using generalized advantage estimates \citep{schulman2015high}. The twin-$Q$ label variant constructs labels from the two critics by sampling actions from the current policy at replayed states and taking the smaller of the resulting soft values. The full label construction is given above. All RDA2C variants use a fixed sliding replay window for the dual regression problem. This label window is separate from the transition replay buffer that the twin-$Q$ backbone shares with SAC. Unless changed in the ablations, we use $W=50{,}000$ label tuples for the GAE MuJoCo setting, $W=12{,}800$ label tuples for the twin-$Q$ critic MuJoCo setting, and $W=500{,}000$ frames for Atari. The Atari GAE runs use clipped dual regression with clip radius $1.0$, stabilizer $\beta=0.6$, and entropy floor $\lambda=0.005$. The clipped dual regression is described in Appendix~\ref{subsec:appendix_clipped_dual_regression}.

\FloatBarrier
\subsection{Hyperparameters}

We provide the hyperparameter configurations for both domains. To ensure a fair comparison, we match the \textbf{Total Batch Size} between RDA2C and the PPO baseline. Method-specific learning-rate and stabilization settings are listed explicitly below.

\paragraph{Continuous control (MuJoCo, GAE labels).}
Table~\ref{tab:mujoco_hparams} lists the hyperparameters for the GAE MuJoCo benchmark.

\begin{table}[!htbp]
    \centering

    \caption[MuJoCo hyperparameters, GAE labels]{\textbf{MuJoCo Hyperparameters (GAE labels).} Comparison of the GAE RDA2C configuration reported in Table~\ref{tab:mujoco_results} and the PPO baseline. RDA2C uses a replay buffer of size $W=50{,}000$ and enables optional clipped prediction regression with $\epsilon_{\mathrm{clip}}=0.2$. The target-advantage clip is inactive for this configuration.}
    \label{tab:mujoco_hparams}
    \footnotesize\renewcommand{\arraystretch}{0.95}
    \begin{tabular}{lcc}
        \toprule
        \textbf{Parameter} & \textbf{RDA2C (Ours)} & \textbf{PPO (Baseline)} \\
        \midrule
        \multicolumn{3}{l}{\textit{Shared Rollout / Optimization Settings}} \\
        Optimizer & Adam & Adam \\
        Discount Factor ($\gamma$) & 0.99 & 0.99 \\
        Total Timesteps & $1{,}000{,}000$ & $1{,}000{,}000$ \\
        Total Batch Size & 2048 & 2048 \\
        Minibatch Size & 64 & 64 \\
        VF Coefficient & 0.5 & 0.5 \\
        Update Epochs & 10 & 10 \\
        Num. Environments & 1 & 1 \\
        Steps per Env & 2048 & 2048 \\
        \midrule
        \multicolumn{3}{l}{\textit{Method Specific Settings}} \\
        Actor Trunk & $2\times128$ Tanh MLP & $2\times128$ Tanh MLP \\
        Actor Heads & Mean / state LogStd / bias & Mean / global LogStd \\
        Value Network & $2\times128$ Tanh MLP & $2\times128$ Tanh MLP \\
        Learning Rate & $2 \times 10^{-4}$ & $2 \times 10^{-4}$ \\
        Hidden Width & 128 & 128 \\
        Max Gradient Norm & 3.0 & 3.0 \\
        GAE Lambda ($\lambda_{\mathrm{GAE}}$) & 0.80 & 0.80 \\
        PPO Policy Ratio Clip & -- & 0.2 \\
        Value Loss Clipping & Yes, $\epsilon=0.2$ & Yes, $\epsilon=0.2$ \\
        \midrule
        \multicolumn{3}{l}{\textit{RDA2C Dual Regression Settings}} \\
        \textbf{Replay Window Size ($W$)} & \textbf{$50{,}000$} & -- \\
        \textbf{Regularizer ($\lambda$)} & \textbf{0.05} & -- \\
        \textbf{Stabilizer ($\beta$)} & \textbf{1.0} & -- \\
        \textbf{Dual Prediction Clip ($\epsilon_{\mathrm{clip}}$)} & \textbf{0.2} & -- \\
        \bottomrule
    \end{tabular}
\end{table}

\paragraph{Continuous control (MuJoCo, twin-$Q$ critic labels).}
Table~\ref{tab:sac_hparams} lists the hyperparameters for the SAC comparison. Both methods use the standard SAC replay and twin-$Q$ architecture, but RDA2C with twin-$Q$ labels uses a tuned actor regression configuration in place of the default SAC update schedule. The table therefore separates common backbone/evaluation settings from the method specific batch, policy frequency, temperature, and label construction settings.

\begin{table}[!htbp]
    \centering

    \caption[MuJoCo hyperparameters, twin-Q critic labels]{\textbf{MuJoCo Hyperparameters (twin-$Q$ critic labels).} Comparison of the SAC baseline at matched batch size and policy update frequency and the RDA2C with twin-$Q$ labels instantiation reported in Table~\ref{tab:mujoco_sac_results}.}
    \label{tab:sac_hparams}
    \footnotesize\renewcommand{\arraystretch}{0.95}
    \begin{tabular}{lcc}
        \toprule
        \textbf{Parameter} & \textbf{RDA2C with twin-$Q$ labels (Ours)} & \textbf{SAC (Baseline)} \\
        \midrule
        \multicolumn{3}{l}{\textit{Common Backbone / Evaluation Settings}} \\
        Actor Trunk & $2\times256$ ReLU MLP & $2\times256$ ReLU MLP \\
        Actor Heads & Mean, LogStd, bias & Mean, LogStd \\
        Twin-$Q$ Networks & $2\times256$ ReLU MLP & $2\times256$ ReLU MLP \\
        Total Timesteps & $1{,}000{,}000$ & $1{,}000{,}000$ \\
        Num. Environments & 5 & 5 \\
        Replay Buffer Size & $1{,}000{,}000$ & $1{,}000{,}000$ \\
        Discount Factor ($\gamma$) & 0.99 & 0.99 \\
        Target Smoothing ($\tau$) & 0.005 & 0.005 \\
        Policy Learning Rate & $3\times 10^{-4}$ & $3\times 10^{-4}$ \\
        $Q$-Learning Rate & $1\times 10^{-3}$ & $1\times 10^{-3}$ \\
        Learning Starts & $5{,}000$ & $5{,}000$ \\
        Target Network Frequency & 1 & 1 \\
        Temperature Tuning & Autotune & Autotune \\
        Evaluation Policy & Deterministic mean action & Deterministic mean action \\
        Eval Frequency / Episodes & $50{,}000$ / 20 & $50{,}000$ / 20 \\
        \midrule
        \multicolumn{3}{l}{\textit{Method Specific SAC-Loop Settings}} \\
        Batch Size & 512 & 512 \\
        Policy Frequency & 1 & 1 \\
        Target Entropy & $-|A|$ & $-|A|$ \\
        \midrule
        \multicolumn{3}{l}{\textit{RDA2C twin-$Q$ Actor Regression Settings}} \\
        \textbf{Label Samples ($K$)} & 2 & -- \\
        \textbf{Replay States per Actor Step ($B$)} & 512 & -- \\
        \textbf{Label Window Size ($W$)} & \textbf{$12{,}800$} & -- \\
        \textbf{Advantage Clip} & 10.0 & -- \\
        \textbf{Dual Prediction Clip ($\epsilon_{\mathrm{clip}}$)} & $\infty$ (off) & -- \\
        \bottomrule
    \end{tabular}
\end{table}

\paragraph{Discrete Control (Atari).}
For Atari, we compare against PPO using the same total batch size per update, the same CleanRL preprocessing and EnvPool execution stack, and the same 10M-frame budget. The remaining rollout and optimization settings follow each method's own defaults and are tabled side by side in Appendix~\ref{app:implementation_details}. Table~\ref{tab:atari_hparams} lists the hyperparameters used for the 10M frame benchmark. Note that while the number of environments and steps per environment differ, the total batch size is fixed at $1024$ for both algorithms.

\begin{table}[!htbp]
    \centering
    \caption{\textbf{Atari Hyperparameters.} Final clipped RDA2C setting compared with PPO.}
    \label{tab:atari_hparams}
    \footnotesize\renewcommand{\arraystretch}{0.95}
    \begin{tabular}{lcc}
        \toprule
        \textbf{Parameter} & \textbf{RDA2C (Ours)} & \textbf{PPO (Baseline)} \\
        \midrule
        \textit{Shared Settings} & & \\
        CNN Backbone & Nature CNN, 512-unit output & Nature CNN, 512-unit output \\
        Actor--Critic Trunks & Independent policy and value trunks & Shared trunk \\
        Optimizer & Adam & Adam \\
        Learning Rate & $2.5 \times 10^{-4}$ & $2.5 \times 10^{-4}$ \\
        LR Schedule & Linear Decay & Linear Decay \\
        Discount Factor ($\gamma$) & 0.99 & 0.99 \\
        Total Batch Size & 1024 & 1024 \\
        Minibatch Size & 256 & 256 \\
        VF Coefficient & 0.5 & 0.5 \\
        Max Gradient Norm & 2.0 & 0.5 \\
        \midrule
        \textit{Sampling Strategy} & & \\
        Num. Environments & 4 & 8 \\
        Steps per Env & 256 & 128 \\
        \midrule
        \textit{Algorithm Specific} & & \\
        Update Epochs & 2 & 4 \\
        GAE Lambda ($\lambda_{\mathrm{GAE}}$) & 0.85 & 0.95 \\
        \textbf{Window Size ($W$)} & \textbf{500,000} & -- \\
        \textbf{Entropy Coeff ($\lambda$)} & \textbf{0.005} & 0.01 \\
        \textbf{Stabilizer ($\beta$)} & \textbf{0.6} & -- \\
        \textbf{Dual Prediction Clip ($\epsilon_{\mathrm{clip}}$)} & \textbf{1.0} & -- \\
        PPO Policy Ratio Clip & -- & 0.1 \\
        \bottomrule
    \end{tabular}
\end{table}

\FloatBarrier
\subsection{Optional clipped dual regression}
\label{subsec:appendix_clipped_dual_regression}

The main algorithm uses the squared dual regression loss in \eqref{eq:dual_regression_loss}. In the final GAE experiments, we also use a clipped prediction variant to prevent abrupt changes in the fitted dual score under repeated replay updates. When a transition enters the buffer, we store the old prediction $Z_{\mathrm{old}}(s,a)$. For clip radius $\epsilon_{\mathrm{clip}}>0$, define
$$
\begin{aligned}
Z_\theta^{\mathrm{clip}}(s,a)
&=
Z_{\mathrm{old}}(s,a)
+
\mathrm{clip}\bigl(Z_\theta(s,a)-Z_{\mathrm{old}}(s,a),-\epsilon_{\mathrm{clip}},\epsilon_{\mathrm{clip}}\bigr),\\
\mathcal L_{\mathrm{clip}}(\theta,t)
&=
\frac12\,
\mathbb E_{(s,a,\hat A,Z_{\mathrm{old}})\sim\mathcal B_t}
\left[
\max\{(Z_\theta(s,a)-\hat A)^2,(Z_\theta^{\mathrm{clip}}(s,a)-\hat A)^2\}
\right].
\end{aligned}
$$
This is a regression stabilizer analogous in spirit to PPO value function clipping \citep{schulman2017proximal}. It is not a policy ratio clip and does not change the mirror map readout in \eqref{eq:closed_form_final}. We use $\epsilon_{\mathrm{clip}}=0.2$ for GAE MuJoCo and $\epsilon_{\mathrm{clip}}=1.0$ for the final GAE Atari runs. We set $\epsilon_{\mathrm{clip}}=\infty$ when this stabilizer is disabled.

\FloatBarrier
\section{Main-Text Numerical Result Tables}
\label{app:main_result_tables}

This appendix collects the exact numerical tables corresponding to the bar-plot summaries in Figures~\ref{fig:mujoco_summary} and~\ref{fig:atari_summary}.

\paragraph{Evaluation conventions.}
The three comparisons use three return conventions, each the one used to evaluate the corresponding baseline. The PPO comparisons on MuJoCo and Atari report the online episodic return during training, the convention of the PPO implementation we match \citep{huang2022ppo37}. The SAC comparison reports the best deterministic evaluation return within the 1M-step budget. Evaluating with the mean action is the SAC protocol \citep{rw_haarnoja2018sac}, and reporting the maximum evaluation return within a 1M-step budget is the convention of \citet{fujimoto2018td3}. The AAPDA comparison reports the mean of the final five deterministic evaluations, following the final-five-epoch average of \citet{rw_ju2026aapda}, with both methods evaluated on the same ten fixed evaluation seeds at checkpoints every $20{,}000$ environment steps, RDA2C's rounded up to its next rollout boundary. The best-within-budget convention is optimistic for both methods it is applied to, and the conventions are not comparable across comparisons, so no number is compared across them.

\paragraph{Aggregate bootstrap intervals.}
The main-text figures plot per-environment relative improvement. For each environment, the point estimate is computed from the seed-averaged returns, $\Delta\%=100\cdot(\bar R_{\mathrm{ours}}-\bar R_{\mathrm{base}})/|\bar R_{\mathrm{base}}|$. Percentages are computed from unrounded seed means, so recomputing them from the rounded table entries can differ in the last digit. To summarize uncertainty at the panel level, we report fixed suite paired bootstrap intervals: each bootstrap replicate resamples matched baseline/ours seed pairs within each environment, computes the environment level percent improvements using that environment's fixed baseline mean for normalization, averages these improvements over the fixed benchmark suite, and takes the 2.5 and 97.5 percentiles across bootstrap replicates.

\paragraph{MuJoCo GAE label comparison.}

Table~\ref{tab:mujoco_results} reports the running average episodic returns underlying the MuJoCo/PPO panel in Figure~\ref{fig:mujoco_summary}. RDA2C uses GAE labels and clipped dual regression in this setting, and the table shows both the mean return and the environment-level relative improvement over PPO.

\begin{table}[htbp]
    \centering

    \caption[MuJoCo continuous control results]{\textbf{MuJoCo Continuous Control Results.} Mean episodic return ($\pm$ std) over 10 seeds. PPO is matched to RDA2C's architecture and optimization settings (Appendix~\ref{app:implementation_details}); the CleanRL-default comparison appears in Appendix~\ref{app:default_baselines}. RDA2C uses clipped dual regression with $\epsilon_{\mathrm{clip}}=0.2$ and value-loss clipping enabled. Last column reports relative change $\Delta\% \triangleq 100\cdot(\text{RDA2C}-\text{PPO})/(|\text{PPO}|)$ (higher is better).}
    \label{tab:mujoco_results}
    \begin{tabular}{lrrr}
        \toprule
        \textbf{Environment} & \textbf{PPO (Baseline)} & \textbf{RDA2C (Ours)} & \textbf{$\Delta\%$} \\
        \midrule
        Ant-v5 & $1953 \pm 372$ & $\mathbf{2831 \pm 583}$ & $\mathbf{+45.0\%}$ \\
        HalfCheetah-v5 & $1864 \pm 78$ & $\mathbf{7099 \pm 645}$ & $\mathbf{+280.9\%}$ \\
        Hopper-v5 & $2333 \pm 590$ & $\mathbf{2607 \pm 272}$ & $\mathbf{+11.7\%}$ \\
        Humanoid-v5 & $2370 \pm 976$ & $\mathbf{3768 \pm 1290}$ & $\mathbf{+59.0\%}$ \\
        HumanoidStandup-v5 & $\mathbf{170412 \pm 8967}$ & $156068 \pm 9614$ & $-8.4\%$ \\
        InvertedPendulum-v5 & $989 \pm 36$ & $\mathbf{999.2 \pm 2.6}$ & $\mathbf{+1.1\%}$ \\
        Pusher-v5 & $-33.5 \pm 2.2$ & $\mathbf{-30.6 \pm 3.0}$ & $\mathbf{+8.9\%}$ \\
        Walker2d-v5 & $\mathbf{3626 \pm 604}$ & $3546 \pm 659$ & $-2.2\%$ \\
        \bottomrule
    \end{tabular}
\end{table}

\paragraph{MuJoCo twin-$Q$ critic comparison.}
Table~\ref{tab:mujoco_sac_results} reports deterministic evaluation returns for the SAC-style twin-$Q$ critic backbone. RDA2C and SAC share the replay and critic components, so the table isolates the effect of replacing the SAC actor improvement step with dual-score regression and entropy readout.

\begin{table}[htbp]
    \centering

    \caption[MuJoCo twin-Q critic results]{\textbf{MuJoCo twin-$Q$ critic results.} Best deterministic evaluation return ($\pm$ seed std) within the 1M-step budget over 10 seeds. Both methods use batch size $512$, policy update frequency $1$, 5 parallel environments, and 20 evaluation episodes; the CleanRL-default SAC comparison appears in Appendix~\ref{app:default_baselines}. Last column reports relative percent change against SAC (higher is better).}
    \label{tab:mujoco_sac_results}
    \begin{tabular}{lrrr}
        \toprule
        \textbf{Environment} & \textbf{SAC (Baseline)} & \textbf{RDA2C with twin-$Q$ labels (Ours)} & \textbf{$\Delta\%$} \\
        \midrule
        Ant-v5 & $\mathbf{4389.60 \pm 1100.19}$ & $4240.94 \pm 877.82$ & $-3.4\%$ \\
        HalfCheetah-v5 & $\mathbf{10118.79 \pm 894.32}$ & $9524.52 \pm 1388.40$ & $-5.9\%$ \\
        Hopper-v5 & $3345.29 \pm 204.84$ & $\mathbf{3408.27 \pm 121.76}$ & $\mathbf{+1.9\%}$ \\
        Humanoid-v5 & $\mathbf{5376.44 \pm 175.72}$ & $5195.96 \pm 314.71$ & $-3.4\%$ \\
        HumanoidStandup-v5 & $158309.50 \pm 4088.25$ & $\mathbf{158784.44 \pm 6868.90}$ & $\mathbf{+0.3\%}$ \\
        InvertedPendulum-v5 & $\mathbf{1000.00 \pm 0.00}$ & $\mathbf{1000.00 \pm 0.00}$ & $0.0\%$ \\
        Pusher-v5 & $-29.58 \pm 1.55$ & $\mathbf{-28.05 \pm 1.51}$ & $\mathbf{+5.2\%}$ \\
        Walker2d-v5 & $4079.87 \pm 371.70$ & $\mathbf{4504.91 \pm 788.33}$ & $\mathbf{+10.4\%}$ \\
        \bottomrule
    \end{tabular}
\end{table}

\newpage
\paragraph{AAPDA comparison.}
Table~\ref{tab:aapda_results} reports the comparison with AAPDA described in Section~\ref{sec:experiments}: released defaults, 1M steps, ten paired training seeds, the same ten fixed evaluation seeds at every checkpoint, and the mean of the final five aligned deterministic checkpoints.

\begin{table}[htbp]
    \centering
    \caption[AAPDA comparison]{\textbf{MuJoCo results against AAPDA.} Mean of the final five aligned deterministic checkpoints ($\pm$ seed std) over 10 paired seeds. Last column reports relative change $\Delta\% \triangleq 100\cdot(\text{RDA2C}-\text{AAPDA})/(|\text{AAPDA}|)$.}
    \label{tab:aapda_results}
    \begin{tabular}{lrrr}
        \toprule
        \textbf{Environment} & \textbf{AAPDA} & \textbf{RDA2C (Ours)} & \textbf{$\Delta\%$} \\
        \midrule
        Ant-v5 & $2562 \pm 374$ & $\mathbf{4043 \pm 691}$ & $\mathbf{+57.8\%}$ \\
        HalfCheetah-v5 & $4948 \pm 705$ & $\mathbf{7808 \pm 629}$ & $\mathbf{+57.8\%}$ \\
        Hopper-v5 & $2594 \pm 794$ & $\mathbf{3218 \pm 196}$ & $\mathbf{+24.0\%}$ \\
        Humanoid-v5 & $1209 \pm 1335$ & $\mathbf{3653 \pm 1459}$ & $\mathbf{+202.1\%}$ \\
        HumanoidStandup-v5 & $\mathbf{162208 \pm 2025}$ & $156485 \pm 9003$ & $-3.5\%$ \\
        InvertedPendulum-v5 & $980.7 \pm 49.6$ & $\mathbf{1000.0 \pm 0.0}$ & $\mathbf{+2.0\%}$ \\
        Pusher-v5 & $-28.2 \pm 1.1$ & $\mathbf{-27.1 \pm 1.7}$ & $\mathbf{+3.6\%}$ \\
        Walker2d-v5 & $\mathbf{4424 \pm 251}$ & $4001 \pm 558$ & $-9.6\%$ \\
        \bottomrule
    \end{tabular}
\end{table}

\paragraph{Baseline default configurations.}
\label{app:default_baselines}

The main-text comparisons use baselines matched to RDA2C's shared settings. For reference, this section reports both baselines at their CleanRL default configurations. For PPO the defaults are $2\times64$ trunks, learning rate $3\times10^{-4}$, gradient clip $0.5$, and GAE $\lambda=0.85$. For SAC they are batch size $256$ and policy update frequency $2$. Matching strengthens the baselines: the matched configuration improves PPO over its defaults on five of eight tasks, and the batch-$512$, frequency-$1$ SAC matches or improves the default SAC on seven of eight tasks.

\begin{table}[htbp]
    \centering
    \caption[MuJoCo results, CleanRL-default PPO]{\textbf{MuJoCo results against CleanRL-default PPO.} Mean episodic return ($\pm$ std) over 10 seeds. Same RDA2C runs as Table~\ref{tab:mujoco_results}.}
    \label{tab:mujoco_results_default}
    \begin{tabular}{lrrr}
        \toprule
        \textbf{Environment} & \textbf{PPO (default)} & \textbf{RDA2C (Ours)} & \textbf{$\Delta\%$} \\
        \midrule
        Ant-v5 & $1855 \pm 519$ & $\mathbf{2831 \pm 583}$ & $\mathbf{+52.6\%}$ \\
        HalfCheetah-v5 & $1884 \pm 78$ & $\mathbf{7099 \pm 645}$ & $\mathbf{+276.8\%}$ \\
        Hopper-v5 & $2588 \pm 831$ & $\mathbf{2607 \pm 272}$ & $\mathbf{+0.7\%}$ \\
        Humanoid-v5 & $1100 \pm 388$ & $\mathbf{3768 \pm 1290}$ & $\mathbf{+242.5\%}$ \\
        HumanoidStandup-v5 & $\mathbf{158572 \pm 12341}$ & $156068 \pm 9614$ & $-1.6\%$ \\
        InvertedPendulum-v5 & $\mathbf{1000.0 \pm 0.0}$ & $999.2 \pm 2.6$ & $-0.1\%$ \\
        Pusher-v5 & $-36.2 \pm 2.0$ & $\mathbf{-30.6 \pm 3.0}$ & $\mathbf{+15.6\%}$ \\
        Walker2d-v5 & $3450 \pm 475$ & $\mathbf{3546 \pm 659}$ & $\mathbf{+2.8\%}$ \\
        \bottomrule
    \end{tabular}
\end{table}

\begin{table}[htbp]
    \centering
    \caption[MuJoCo twin-Q results, CleanRL-default SAC]{\textbf{MuJoCo twin-$Q$ results against CleanRL-default SAC} (batch $256$, policy frequency $2$). Best deterministic evaluation return ($\pm$ seed std) within the 1M-step budget over 10 seeds. Same RDA2C runs as Table~\ref{tab:mujoco_sac_results}.}
    \label{tab:mujoco_sac_results_default}
    \begin{tabular}{lrrr}
        \toprule
        \textbf{Environment} & \textbf{SAC (default)} & \textbf{RDA2C with twin-$Q$ labels (Ours)} & \textbf{$\Delta\%$} \\
        \midrule
        Ant-v5 & $4048.10 \pm 859.12$ & $\mathbf{4240.94 \pm 877.82}$ & $\mathbf{+4.8\%}$ \\
        HalfCheetah-v5 & $\mathbf{9844.07 \pm 914.73}$ & $9524.52 \pm 1388.40$ & $-3.2\%$ \\
        Hopper-v5 & $3348.58 \pm 157.32$ & $\mathbf{3408.27 \pm 121.76}$ & $\mathbf{+1.8\%}$ \\
        Humanoid-v5 & $\mathbf{5228.44 \pm 241.40}$ & $5195.96 \pm 314.71$ & $-0.6\%$ \\
        HumanoidStandup-v5 & $156360.89 \pm 7197.51$ & $\mathbf{158784.44 \pm 6868.90}$ & $\mathbf{+1.5\%}$ \\
        InvertedPendulum-v5 & $\mathbf{1000.00 \pm 0.00}$ & $\mathbf{1000.00 \pm 0.00}$ & $0.0\%$ \\
        Pusher-v5 & $-29.80 \pm 1.55$ & $\mathbf{-28.05 \pm 1.51}$ & $\mathbf{+5.9\%}$ \\
        Walker2d-v5 & $3871.58 \pm 736.89$ & $\mathbf{4504.91 \pm 788.33}$ & $\mathbf{+16.4\%}$ \\
        \bottomrule
    \end{tabular}
\end{table}

\newpage
\paragraph{Atari GAE label comparison.}
Table~\ref{tab:atari_comparison} reports the twelve Atari games underlying Figure~\ref{fig:atari_summary}.

\begin{table}[htbp]
    \centering
    \caption[Atari results]{\textbf{Atari results.} Mean running return ($\pm$ seed std) over 10 seeds on twelve Atari games with clipped dual regression, $\epsilon_{\mathrm{clip}}=1.0$, $\beta=0.6$, and $\lambda=0.005$. Last column reports relative change $\Delta\% \triangleq 100\cdot(\text{RDA2C}-\text{PPO})/(|\text{PPO}|)$ when PPO $\neq 0$.}
    \label{tab:atari_comparison}
    \begin{tabular}{lrrr}
        \toprule
        \textbf{Environment} & \textbf{PPO (Baseline)} & \textbf{RDA2C (Ours)} & \textbf{$\Delta\%$} \\
        \midrule
        Asterix & $4061 \pm 1014$ & $\mathbf{6024 \pm 1336}$ & $\mathbf{+48.3\%}$ \\
        BeamRider & $2540 \pm 316$ & $\mathbf{9142 \pm 786}$ & $\mathbf{+259.9\%}$ \\
        Breakout & $\mathbf{422 \pm 24}$ & $408 \pm 29$ & $-3.5\%$ \\
        Enduro & $980 \pm 135$ & $\mathbf{5348 \pm 1720}$ & $\mathbf{+445.6\%}$ \\
        Freeway & $\mathbf{32.7 \pm 0.6}$ & $23.6 \pm 16.3$ & $-27.8\%$ \\
        Gravitar & $686 \pm 187$ & $\mathbf{835 \pm 376}$ & $\mathbf{+21.8\%}$ \\
        MsPacman & $2085 \pm 268$ & $\mathbf{3243 \pm 936}$ & $\mathbf{+55.6\%}$ \\
        Pong & $\mathbf{20.6 \pm 0.7}$ & $20.5 \pm 1.4$ & $-0.6\%$ \\
        PrivateEye & $70.1 \pm 47.1$ & $\mathbf{89.8 \pm 31.6}$ & $\mathbf{+28.2\%}$ \\
        Qbert & $\mathbf{15804 \pm 1993}$ & $14562 \pm 1283$ & $-7.9\%$ \\
        Seaquest & $1541 \pm 427$ & $\mathbf{9114 \pm 7642}$ & $\mathbf{+491.4\%}$ \\
        SpaceInvaders & $998 \pm 187$ & $\mathbf{1373 \pm 220}$ & $\mathbf{+37.6\%}$ \\
        \bottomrule
    \end{tabular}
\end{table}

\newpage
\section{Additional Experimental Results}
\label{app:additional_results}

\FloatBarrier
\subsection{MuJoCo GAE label learning curves}
\label{app:mujoco_curves}

Figure~\ref{fig:mujoco_curves} displays the 10-seed training dynamics for the PPO versus clipped-dual-regression RDA2C comparison across the 8 continuous control environments. The curves highlight large gains on \texttt{HalfCheetah} and \texttt{Humanoid}, a clear improvement on \texttt{Ant}, and PPO-level behavior on the remaining tasks, with \texttt{HumanoidStandup} remaining slightly below PPO in the final aggregate table.

\begin{figure}[htbp]
    \centering
    \includegraphics[width=\textwidth]{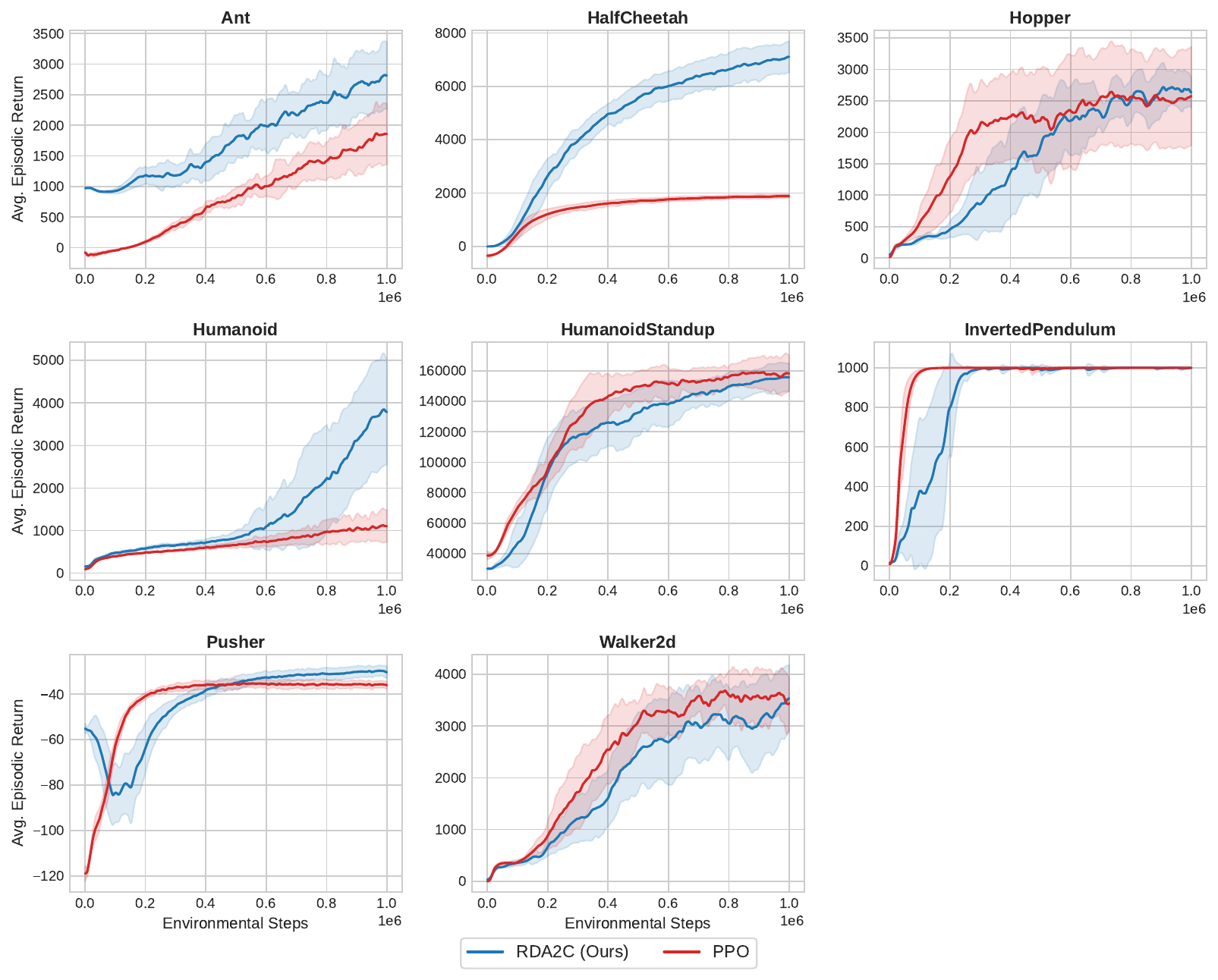}
    \caption{\textbf{PPO versus RDA2C MuJoCo learning curves.} Blue denotes RDA2C and red denotes PPO. Shaded regions show standard deviation over 10 seeds. The clipped-dual-regression configuration yields large gains on \texttt{HalfCheetah}, \texttt{Humanoid}, and \texttt{Ant}, while remaining close to PPO on the other environments.}
    \label{fig:mujoco_curves}
\end{figure}

\FloatBarrier
\subsection{MuJoCo twin-$Q$ critic label learning curves}
\label{app:mujoco_sac_curves}
Figure~\ref{fig:mujoco_sac_curves} presents the 10-seed online training-return curves for the SAC versus RDA2C with twin-$Q$ labels comparison. These curves correspond to the logged running reward summaries in Table~\ref{tab:mujoco_sac_running_reward}, not to the deterministic evaluation returns in Table~\ref{tab:mujoco_sac_results}. Across environments, the off-policy RDA actor update with soft twin-$Q$ RLOO labels is generally stronger during data collection: RDA2C with twin-$Q$ labels has higher final running reward on seven of eight tasks and lower seed standard deviation on six of eight tasks. The largest online return gaps appear on \texttt{Hopper-v5} and \texttt{Walker2d-v5}. \texttt{HalfCheetah-v5} remains the main environment where SAC is clearly stronger.

\begin{table}[htbp]
    \centering
    \caption[MuJoCo online training return with twin-Q critic labels]{\textbf{MuJoCo online training return with twin-$Q$ critic labels.} Final logged running reward ($\pm$ seed std) over 10 seeds at $1{,}000{,}000$ environment steps. This is an online data collection metric, distinct from the deterministic evaluation return in Table~\ref{tab:mujoco_sac_results}. Last column reports relative percent change against SAC (higher is better).}
    \label{tab:mujoco_sac_running_reward}
    \begin{tabular}{lrrr}
        \toprule
        \textbf{Environment} & \textbf{SAC (Baseline)} & \textbf{RDA2C with twin-$Q$ labels (Ours)} & \textbf{$\Delta\%$} \\
        \midrule
        Ant-v5 & $3414.91 \pm 873.59$ & $\mathbf{3683.40 \pm 842.94}$ & $\mathbf{+7.9\%}$ \\
        HalfCheetah-v5 & $\mathbf{9182.51 \pm 888.07}$ & $8538.56 \pm 1167.47$ & $-7.0\%$ \\
        Hopper-v5 & $2289.03 \pm 637.75$ & $\mathbf{2948.40 \pm 357.63}$ & $\mathbf{+28.8\%}$ \\
        Humanoid-v5 & $4771.82 \pm 320.91$ & $\mathbf{4979.12 \pm 215.64}$ & $\mathbf{+4.3\%}$ \\
        HumanoidStandup-v5 & $152428.66 \pm 7472.47$ & $\mathbf{155100.32 \pm 5994.39}$ & $\mathbf{+1.8\%}$ \\
        InvertedPendulum-v5 & $935.64 \pm 72.85$ & $\mathbf{969.61 \pm 33.70}$ & $\mathbf{+3.6\%}$ \\
        Pusher-v5 & $-30.01 \pm 1.24$ & $\mathbf{-28.87 \pm 1.50}$ & $\mathbf{+3.8\%}$ \\
        Walker2d-v5 & $3124.79 \pm 1354.77$ & $\mathbf{4067.66 \pm 1117.91}$ & $\mathbf{+30.2\%}$ \\
        \bottomrule
    \end{tabular}
\end{table}

\begin{figure}[htbp]
    \centering
    \includegraphics[width=\textwidth]{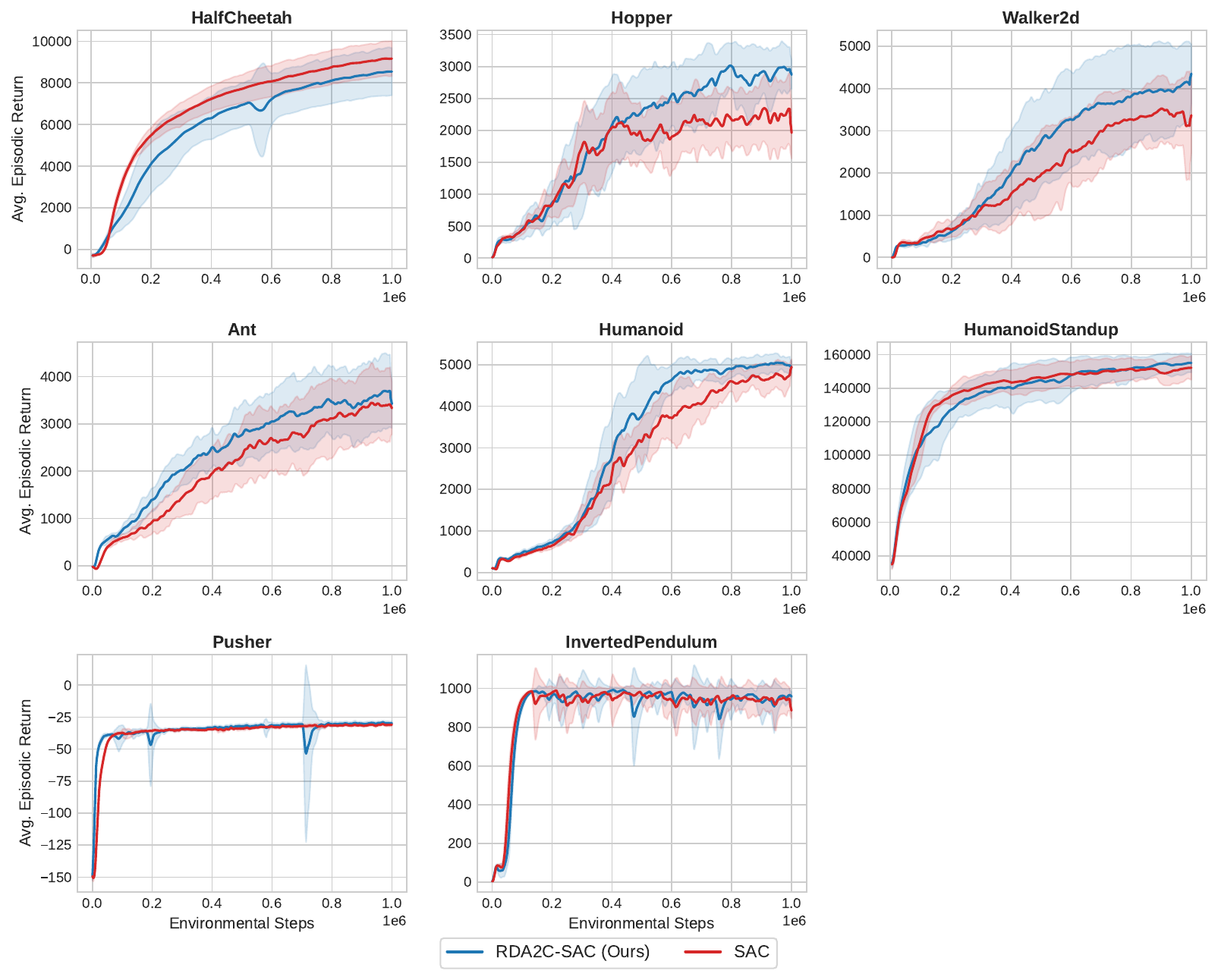}
    \caption{\textbf{MuJoCo online training curves with twin-$Q$ critic labels.} Blue denotes RDA2C twin-$Q$ and red denotes SAC. Shaded regions show standard deviation over 10 seeds. The curves report the online running reward metric logged during data collection. RDA2C twin-$Q$ achieves higher final running reward on Ant, Hopper, Humanoid, HumanoidStandup, InvertedPendulum, Pusher, and Walker2d, while SAC remains stronger on HalfCheetah.}
    \label{fig:mujoco_sac_curves}
\end{figure}

\FloatBarrier
\subsection{Atari Learning Curves}
\label{app:atari_curves}
Figure \ref{fig:atari_full_curves} presents the learning dynamics for the twelve Atari games used in the main comparison. Curves represent the average episodic return over 10 random seeds with standard deviation shading.

\paragraph{Analysis.}
The curves give a more detailed view of the summary in the main text. RDA2C improves substantially over PPO on \texttt{Seaquest}, \texttt{BeamRider}, \texttt{Enduro}, and \texttt{Asterix}. In these runs, the replayed dual averaging update continues to improve after the PPO curve has largely flattened.

The remaining games include threshold and saturated regimes. \texttt{Pong} is near saturated for both methods, while \texttt{Freeway} remains a high variance threshold game where PPO is higher in the final 10-seed aggregate. RDA2C improves \texttt{PrivateEye}, but it does not add a directed exploration mechanism.

\begin{figure}[htbp]
    \centering
    \includegraphics[width=\textwidth]{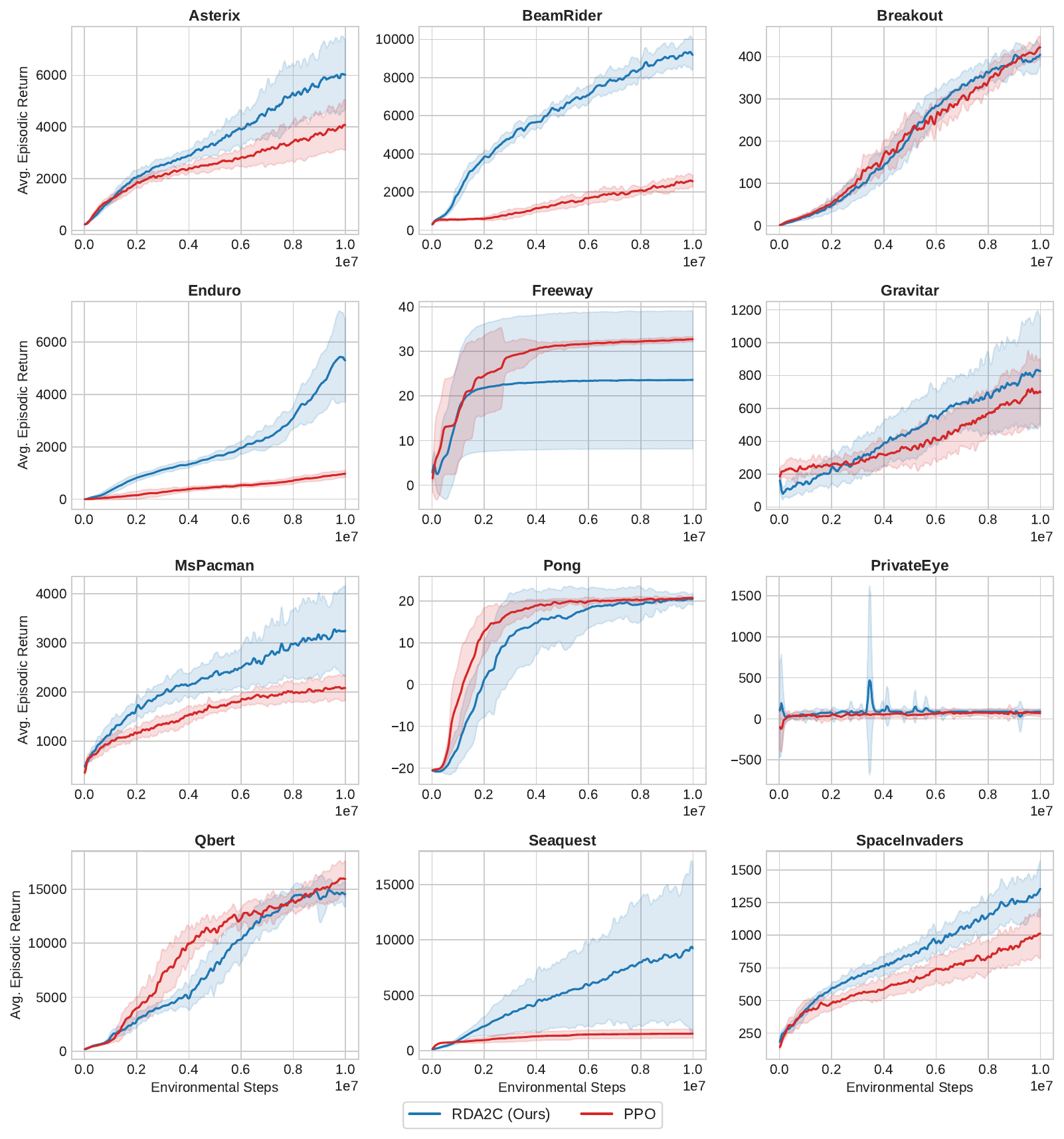}
    \caption{\textbf{Full Training Dynamics (10M Frames).} Comparison of RDA2C and PPO across the twelve Atari games used in the main comparison.}
    \label{fig:atari_full_curves}
\end{figure}

\FloatBarrier
\section{Ablation Studies}
\label{app:additional_ablations}
This section reports the ablations: the replay window (Section~\ref{app:ablation_window}), the RDA stabilizer (Section~\ref{app:ablation_beta}), the entropy floor (Section~\ref{app:ablation_lambda}), the advantage-label loss (Section~\ref{app:ablation_adv_loss}), the SAC replay buffer (Section~\ref{app:ablation_sac_buffer}), and a PPO capacity control (Section~\ref{app:ablation_ppo_capacity}).
\FloatBarrier
\subsection{Replay Window Size}
\label{app:ablation_window}

We evaluate sensitivity to the replay / label window size $W$ in both the GAE Atari setting and the twin-$Q$ critic label MuJoCo setting. The two sweeps differ because the Atari implementation stores replayed transitions, whereas the twin-$Q$ critic implementation stores labeled $(s,a,g)$ tuples in a dedicated FIFO actor regression buffer.

This sweep should be interpreted qualitatively. The theoretical discussion in Section~\ref{sec:theory} predicts a variance--staleness tradeoff and, in a simplified asymptotic bound, balances the replay error terms with $W_t\asymp t^{2/3}$. Our finite budget sweeps do not estimate this exponent: they vary fixed window sizes and report achieved return, and they do not measure score estimation error directly. The empirical question is therefore whether small windows are too noisy and whether very long memories can become stale.

\label{subsec:ablation_window_main}
The replay window controls how much past dual evidence defines the actor objective. Small windows use fresher labels but provide fewer samples for fitting the dual score, while large windows reduce label variance but mix in older policy-improvement signals. The Atari sweep in Appendix~\ref{app:ablation_window} shows this finite-budget tradeoff, with intermediate windows often strongest. The best window varies across games and is not monotone in \(W\). This behavior qualitatively matches the finite replay and local drift terms in Theorem~\ref{thm:fitted_replay_value_bound}. The twin-$Q$ MuJoCo sweep (Appendix~\ref{app:ablation_window_mujoco}) shows the same nonmonotone pattern, and the one-round current-label setting is never the best window. A current-label control for the GAE setting (Appendix~\ref{app:ablation_window_currentlabel}) isolates retention alone: the historical target is higher on seven of eight tasks.

\FloatBarrier
\subsubsection{Atari}
\label{app:ablation_window_atari}

We evaluate the sensitivity of RDA2C to the Atari replay window size $W$. Figure \ref{fig:window_sweeps} shows the learning curves, and Table \ref{tab:window_full} provides the final asymptotic performance. The window sweep uses the plain squared dual regression loss, whereas the main Atari comparison in Table~\ref{tab:atari_comparison} uses the clipped loss of Appendix~\ref{subsec:appendix_clipped_dual_regression}, so the two tables are not comparable at $W=500{,}000$. Only the comparison across windows within the sweep is meaningful.

\begin{table}[htbp]
    \centering
    \caption{\textbf{Atari Window Size Sweep Results.} Mean episodic return at 10M frames, plain squared loss. \textbf{Bold} indicates the best performing configuration.}
    \label{tab:window_full}
    \begin{tabular}{lcccc}
        \toprule
        \textbf{Environment} & \textbf{$W=300,000$} & \textbf{$W=500,000$} & \textbf{$W=700,000$} & \textbf{$W=1,000,000$} \\
        \midrule
        Asterix & $\mathbf{5448 \pm 1063}$ & $5305 \pm 204$ & $4758 \pm 929$ & $4122 \pm 505$ \\
        BeamRider & $9768 \pm 429$ & $\mathbf{10509 \pm 874}$ & $9823 \pm 317$ & $9831 \pm 832$ \\
        Breakout & $\mathbf{369 \pm 17}$ & $353 \pm 6$ & $319 \pm 22$ & $314 \pm 31$ \\
        Enduro & $\mathbf{3474 \pm 372}$ & $3374 \pm 273$ & $3077 \pm 520$ & $3076 \pm 340$ \\
        Freeway & $33.8 \pm 0.0$ & $33.8 \pm 0.0$ & $33.8 \pm 0.0$ & $\mathbf{33.9 \pm 0.0}$ \\
        Gravitar & $690 \pm 93$ & $\mathbf{830 \pm 169}$ & $675 \pm 40$ & $670 \pm 65$ \\
        MsPacman & $2788 \pm 576$ & $\mathbf{3974 \pm 669}$ & $3056 \pm 802$ & $3245 \pm 428$ \\
        Pong & $20.9 \pm 0.0$ & $\mathbf{21.0 \pm 0.0}$ & $\mathbf{21.0 \pm 0.0}$ & $20.9 \pm 0.1$ \\
        PrivateEye & $66.6 \pm 47.1$ & $66.6 \pm 47.1$ & $\mathbf{100.0 \pm 0.0}$ & $99.9 \pm 0.0$ \\
        Qbert & $15341 \pm 392$ & $15690 \pm 309$ & $14310 \pm 1150$ & $\mathbf{15892 \pm 929}$ \\
        Seaquest & $5484 \pm 1819$ & $\mathbf{14158 \pm 5832}$ & $13335 \pm 4911$ & $13843 \pm 4220$ \\
        SpaceInvaders & $1359 \pm 44$ & $\mathbf{1457 \pm 237}$ & $1361 \pm 188$ & $1405 \pm 172$ \\
        \bottomrule
    \end{tabular}
\end{table}

\begin{figure}[htbp]
    \centering
   \includegraphics[width=\textwidth]{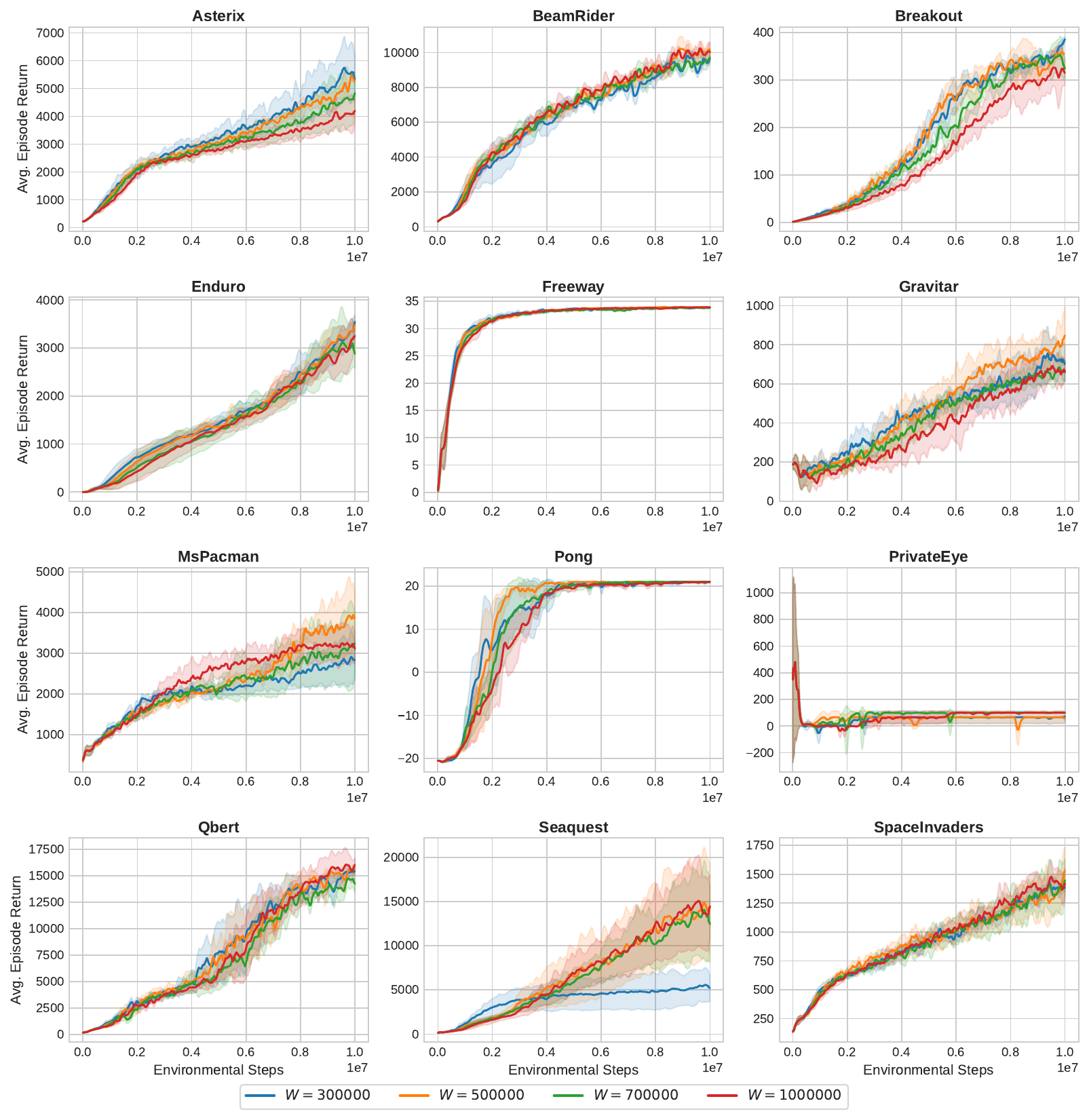}
    \caption{\textbf{Effect of Window Size ($W$) on Atari.} Training curves for $W \in \{300\text{k}, 500\text{k}, 700\text{k}, 1\text{M}\}$ on the twelve Atari games used in the main comparison. $W=500\text{k}$ offers the most consistent performance across tasks.}
    \label{fig:window_sweeps}
\end{figure}

\newpage
\subsubsection{MuJoCo with twin-$Q$ Critics}
\label{app:ablation_window_mujoco}

We additionally sweep the RDA label buffer size on five representative MuJoCo tasks: \texttt{Ant-v5}, \texttt{HalfCheetah-v5}, \texttt{Hopper-v5}, \texttt{Humanoid-v5}, and \texttt{Walker2d-v5}. We summarize each setting by the best deterministic evaluation mean return achieved during training, averaged over 10 seeds. In this implementation, each actor trigger samples $B=512$ replay states and generates $K=2$ labels per state, so it inserts $KB=1024$ fresh $(s,a,g)$ tuples into the FIFO label buffer. Thus $W=1024$ acts as a one round current-label control. Table~\ref{tab:window_sac_full} shows a nonmonotone mean return pattern: the current-label control is not the top mean setting on any of the five tasks, larger windows substantially improve \texttt{HalfCheetah-v5} and moderately improve \texttt{Walker2d-v5}, and \texttt{Ant-v5}, \texttt{Hopper-v5}, and \texttt{Humanoid-v5} are comparatively flat once $W>1024$. This is consistent with the qualitative bias--variance/staleness tradeoff from Section~\ref{sec:theory}: averaging recent actor-improvement labels can reduce noise, but overly long memories can become stale.

\begin{table}[htbp]
    \centering
    \caption[MuJoCo window size sweep with twin-Q critic labels]{\textbf{MuJoCo window size sweep with twin-$Q$ critic labels.} Mean best deterministic evaluation return ($\pm$ seed std) over 10 seeds. Here $W$ is the FIFO label buffer size in stored $(s,a,g)$ tuples. With batch size $512$ and $K=2$ label samples, $W=1024$ retains approximately one actor trigger's fresh labels and therefore acts as a current label control. \textbf{Bold} indicates the best window per environment.}
    \label{tab:window_sac_full}
    \resizebox{\textwidth}{!}{%
    \begin{tabular}{lcccccc}
        \toprule
        \textbf{Environment} & \textbf{$W=1024$} & \textbf{$W=2048$} & \textbf{$W=4096$} & \textbf{$W=8192$} & \textbf{$W=12{,}800$} & \textbf{$W=25{,}600$} \\
        \midrule
        Ant-v5 & $4755 \pm 985$ & $4008 \pm 953$ & $4314 \pm 909$ & $\mathbf{4798 \pm 747}$ & $4241 \pm 878$ & $4672 \pm 946$ \\
        HalfCheetah-v5 & $7068 \pm 3653$ & $8800 \pm 2565$ & $9387 \pm 1542$ & $8081 \pm 3410$ & $9525 \pm 1388$ & $\mathbf{9819 \pm 1029}$ \\
        Hopper-v5 & $3404 \pm 195$ & $3330 \pm 111$ & $3365 \pm 207$ & $3351 \pm 133$ & $3408 \pm 122$ & $\mathbf{3430 \pm 150}$ \\
        Humanoid-v5 & $5280 \pm 110$ & $5237 \pm 159$ & $5279 \pm 200$ & $\mathbf{5294 \pm 127}$ & $5196 \pm 315$ & $5218 \pm 107$ \\
        Walker2d-v5 & $4444 \pm 459$ & $3810 \pm 1291$ & $\mathbf{4724 \pm 522}$ & $3566 \pm 1302$ & $4505 \pm 788$ & $4179 \pm 977$ \\
        \bottomrule
    \end{tabular}%
    }
\end{table}

\FloatBarrier
\subsubsection{Current-label control (GAE MuJoCo)}
\label{app:ablation_window_currentlabel}

This control changes label retention and nothing else. The buffer size $W=2048$ retains exactly one rollout of labels, so every regression target comes from the current policy. The main-results setting $W=50{,}000$ retains about 24 rollouts. Architecture, optimizer state, and training budget are fixed. Table~\ref{tab:current_label_gae} shows the result over ten paired seeds: the historical target is higher on seven of eight tasks. The gains are large where they appear: \texttt{Humanoid-v5} goes from $500$ to $3768$ and \texttt{Ant-v5} from $884$ to $2831$. This is the most direct evidence that the improvement comes from aggregating labels across several recent policies. Score regression on current labels alone does not produce it.

\begin{table}[htbp]
    \centering
    \caption[Current-label control, GAE MuJoCo]{\textbf{Current-label control (GAE labels, MuJoCo).} Mean running episodic return ($\pm$ seed std) over 10 paired seeds. Last column reports relative change $\Delta\% \triangleq 100\cdot(\text{history}-\text{current})/(|\text{current}|)$ (higher is better).}
    \label{tab:current_label_gae}
    \begin{tabular}{lrrr}
        \toprule
        \textbf{Environment} & \textbf{current ($W=2048$)} & \textbf{history ($W=50{,}000$)} & \textbf{$\Delta\%$} \\
        \midrule
        Ant-v5 & $884 \pm 412$ & $\mathbf{2831 \pm 583}$ & $\mathbf{+220.1\%}$ \\
        HalfCheetah-v5 & $5162 \pm 1298$ & $\mathbf{7099 \pm 645}$ & $\mathbf{+37.5\%}$ \\
        Hopper-v5 & $2112 \pm 802$ & $\mathbf{2607 \pm 272}$ & $\mathbf{+23.4\%}$ \\
        Humanoid-v5 & $500 \pm 75$ & $\mathbf{3768 \pm 1290}$ & $\mathbf{+654.0\%}$ \\
        HumanoidStandup-v5 & $94062 \pm 24669$ & $\mathbf{156068 \pm 9614}$ & $\mathbf{+65.9\%}$ \\
        InvertedPendulum-v5 & $\mathbf{1000.0 \pm 0.0}$ & $999.2 \pm 2.6$ & $-0.1\%$ \\
        Pusher-v5 & $-31.9 \pm 1.3$ & $\mathbf{-30.6 \pm 3.0}$ & $\mathbf{+4.3\%}$ \\
        Walker2d-v5 & $2973 \pm 707$ & $\mathbf{3546 \pm 659}$ & $\mathbf{+19.3\%}$ \\
        \bottomrule
    \end{tabular}
\end{table}

\newpage
\subsection{RDA Stabilizer}
\label{app:ablation_beta}
\label{subsec:main_text_ablations}

Figure~\ref{fig:beta_sensitivity_main} varies the RDA stabilizer \(\beta\) while holding the entropy floor and replay window fixed. Larger values damp early policy changes more strongly. The aggregate results show that some transient stabilization is useful, with moderate values performing best, while the per-task heatmaps show that the best \(\beta\) is environment dependent. Very large \(\beta\) can make the early readout too conservative, whereas \(\beta=0\) removes this transient damping entirely.

\begin{figure}[htbp]
    \centering
    \includegraphics[width=0.98\textwidth]{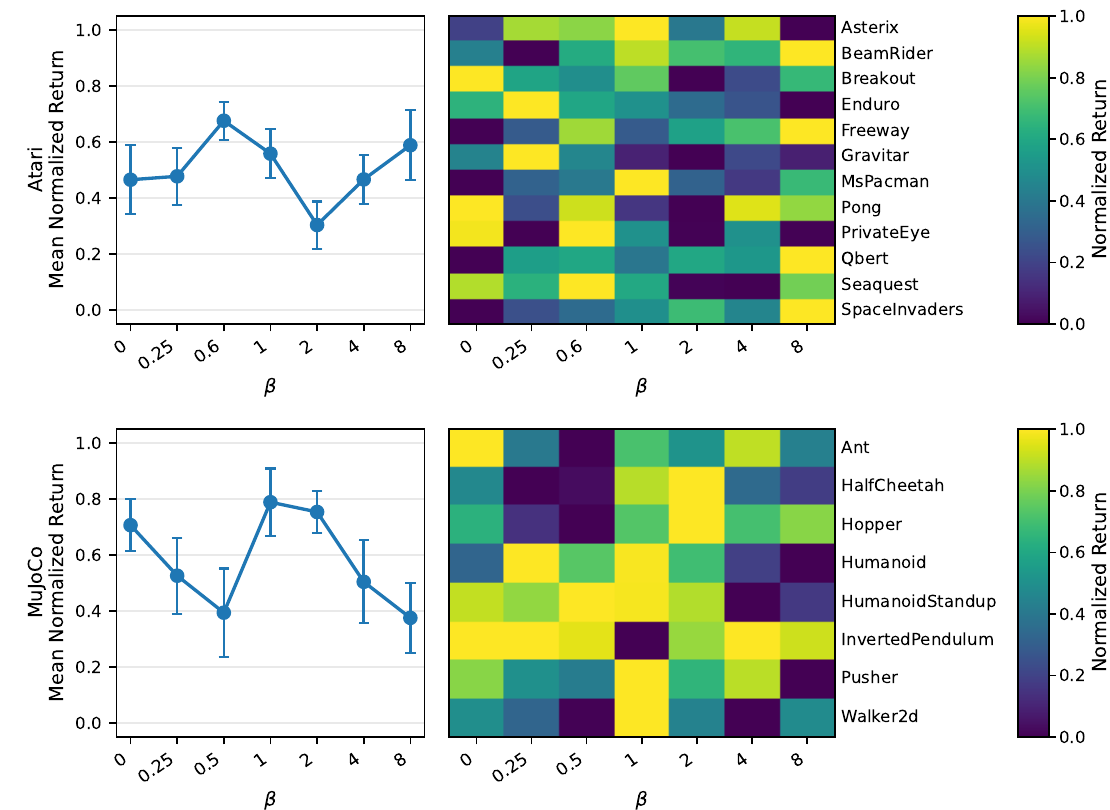}
    \caption{\textbf{Sensitivity to the RDA stabilizer $\beta$.} Returns are normalized independently within each environment across the swept $\beta$ values before aggregation. Error bars in the aggregate panels denote task SEM. For the reason in footnote~\ref{fn:bootstrap}, they do not denote confidence intervals. They illustrate the scale of variation across tasks.}
    \label{fig:beta_sensitivity_main}
\end{figure}

\newpage
\subsection{Entropy Floor Sensitivity}
\label{app:ablation_lambda}

We sweep the persistent entropy floor $\lambda$ in the RDA temperature $\alpha_t=\lambda+\beta/t$ while holding the replay window and $\beta$ fixed. Figure~\ref{fig:rda_lambda_sensitivity_appendix} shows a broad small $\lambda$ plateau on Atari and a moderate $\lambda$ plateau on MuJoCo. Very large entropy floors keep the readout too diffuse and hurt return.

\begin{figure}[htbp]
    \centering
    \includegraphics[width=0.98\textwidth]{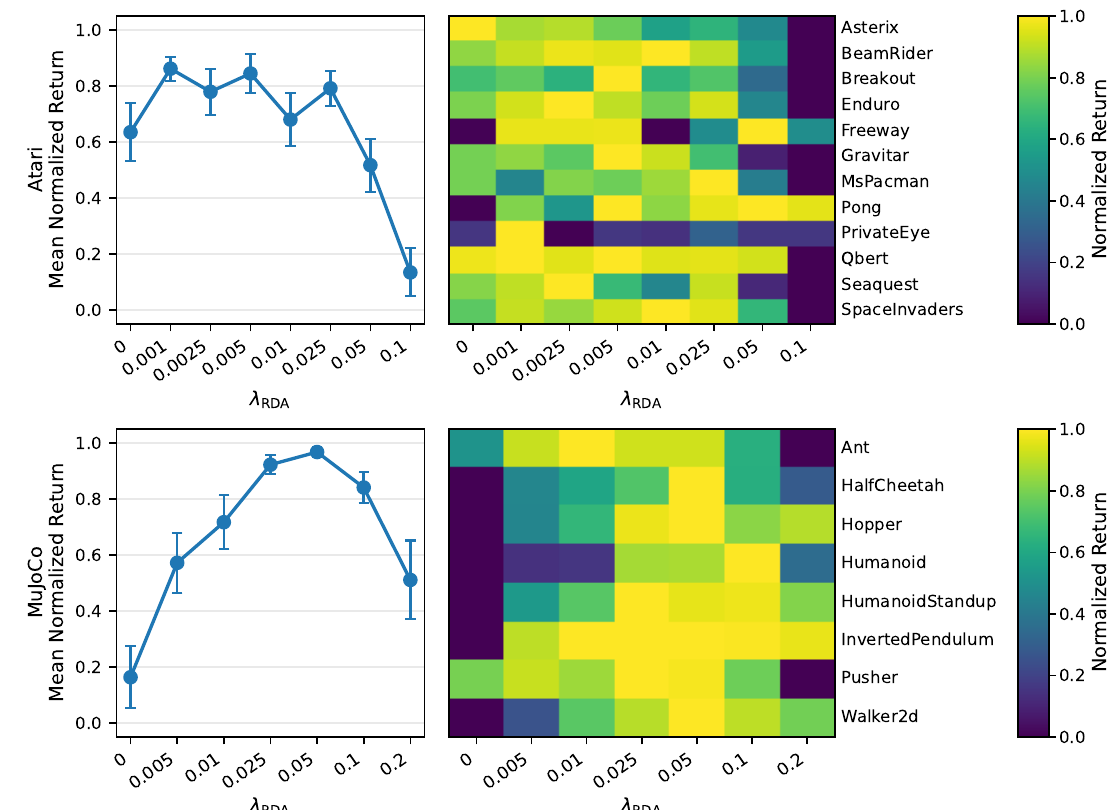}
    \caption{\textbf{Sensitivity to the entropy floor $\lambda$.} Returns are min-max normalized independently within each environment across the swept $\lambda$ values before aggregation. Error bars in the aggregate panels denote task SEM. For the reason in footnote~\ref{fn:bootstrap}, they do not denote confidence intervals. They illustrate the scale of variation across tasks.}
    \label{fig:rda_lambda_sensitivity_appendix}
\end{figure}

\newpage
\subsection{Advantage-Label Loss}
\label{app:ablation_adv_loss}

We compare four dual regression losses for GAE advantage labels: plain MSE, Huber regression, normalized target clipping, and PPO style prediction clipping. Figure~\ref{fig:adv_loss_ablation_appendix} reports Atari and MuJoCo separately. Prediction clipping is the strongest default, especially on MuJoCo. Target clipping is competitive on Atari but is less reliable in continuous control. Huber regression alone does not recover the performance of prediction clipping.

\begin{figure}[htbp]
    \centering
    \includegraphics[width=0.92\textwidth]{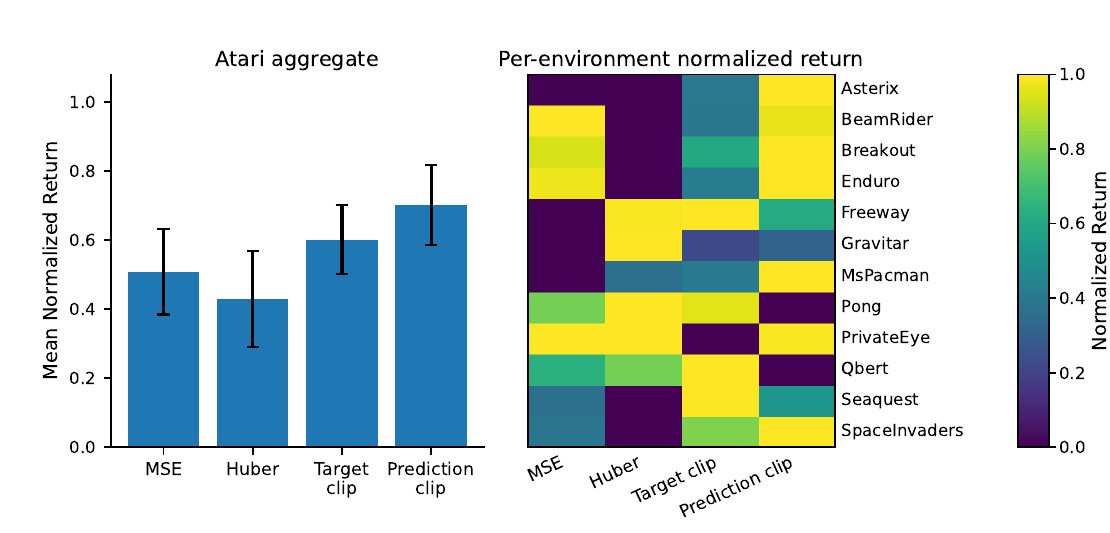}
    \vspace{0.5em}
    \includegraphics[width=0.92\textwidth]{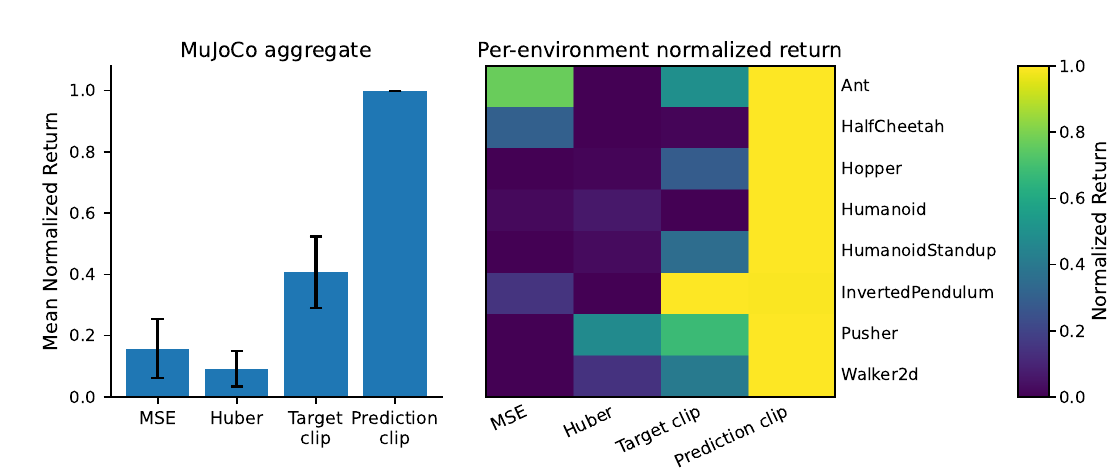}
    \caption{\textbf{Advantage label loss ablation.} Atari results are shown above and MuJoCo results below. Returns are min-max normalized within each environment across loss types before aggregation. Error bars in the aggregate panels denote task SEM. For the reason in footnote~\ref{fn:bootstrap}, they do not denote confidence intervals. They illustrate the scale of variation across tasks. Prediction clipping gives the best aggregate, while Huber alone is weaker.}
    \label{fig:adv_loss_ablation_appendix}
\end{figure}

\newpage
\subsection{SAC Replay Buffer Size}
\label{app:ablation_sac_buffer}

Sensitivity to the shared transition replay buffer. We examine sensitivity to the capacity of the ordinary transition replay buffer used by the shared SAC/twin-Q backbone. In the main comparison, we retain the 1M capacity from the default SAC configuration used in our implementation and apply the same capacity to both SAC and RDA2C. This setting preserves comparability with the standard baseline. It was not selected using the present sweep. For this diagnostic ablation, we vary the shared transition-buffer capacity over 50k, 100k, 250k, 500k, and 1M, while holding RDA2C’s separate advantage-label buffer fixed at 12{,}800. We normalize jointly across both methods and all tested capacities within each environment, so the aggregate summarizes sensitivity to the shared replay capacity and the method gap on a common scale. Figure~\ref{fig:sac_buffer_sensitivity_appendix} shows that aggregate normalized performance peaks near 100k for both methods and decreases for larger capacities. The similar dependence across SAC and RDA2C suggests that this effect comes primarily from recency in the shared transition-replay backbone and not from the RDA-specific label buffer, which is held fixed. We retain 1M in the main comparison to report results under the default SAC configuration.

\begin{figure}[htbp]
    \centering
    \includegraphics[width=\textwidth]{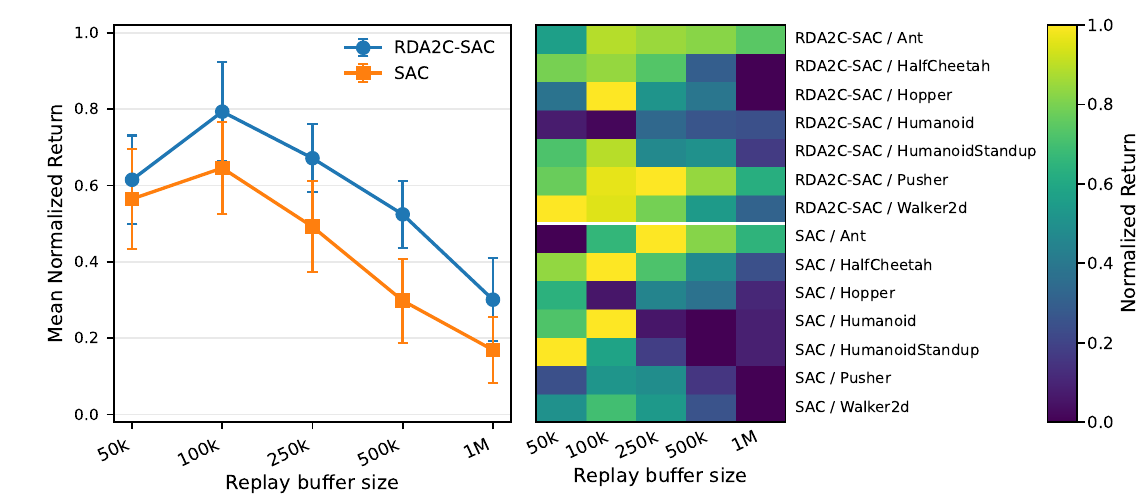}
    \caption{\textbf{Sensitivity to standard SAC replay buffer size.} We vary the ordinary SAC replay buffer and keep the RDA label window fixed at $12{,}800$. Returns are min-max normalized independently within each environment across all method and buffer configurations before aggregation. Error bars denote task SEM. InvertedPendulum is omitted from the normalized aggregate because it is saturated for every setting.}
    \label{fig:sac_buffer_sensitivity_appendix}
\end{figure}

\newpage
\subsection{PPO Network Capacity Control (MuJoCo, GAE)}
\label{app:ablation_ppo_capacity}

The matched comparison in Table~\ref{tab:mujoco_results} sets PPO to RDA2C's width and learning rate. This ablation asks the complementary question: does capacity alone close the gap when PPO keeps its own defaults? Starting from the CleanRL default $2\times64$ trunk (Appendix~\ref{app:default_baselines}), we retrain PPO with a $2\times128$ trunk, keeping the global log-standard-deviation and the remaining default hyperparameters. We report both the default learning rate $3\times10^{-4}$ and the RDA2C rate $2\times10^{-4}$. Table~\ref{tab:ppo_capacity_control} and Figure~\ref{fig:ppo_capacity_control_appendix} show that enlarging PPO does not close the gap: at $3\times10^{-4}$ the wider PPO is often worse than the $2\times64$ baseline (especially on \texttt{Ant-v5} and \texttt{Hopper-v5}), and lowering the learning rate recovers some of that loss but still leaves RDA2C ahead on the tasks with the largest gains.

\begin{table}[htbp]
    \centering
    \caption[PPO capacity control on MuJoCo GAE]{\textbf{PPO network capacity control (MuJoCo, GAE labels).} Mean running episodic return ($\pm$ seed std) over 10 seeds. All PPO variants use global log-standard-deviation. \textbf{Bold} marks the best mean in each row.}
    \label{tab:ppo_capacity_control}
    \resizebox{\textwidth}{!}{%
    \begin{tabular}{lcccc}
        \toprule
        \textbf{Environment} & \textbf{PPO $2\times64$} & \textbf{PPO $2\times128$} & \textbf{PPO $2\times128$} & \textbf{RDA2C $2\times128$} \\
         & \textbf{lr $3\times10^{-4}$} & \textbf{lr $3\times10^{-4}$} & \textbf{lr $2\times10^{-4}$} & \textbf{lr $2\times10^{-4}$} \\
        \midrule
        Ant-v5 & $1855 \pm 519$ & $1260 \pm 441$ & $1537 \pm 757$ & $\mathbf{2831 \pm 583}$ \\
        HalfCheetah-v5 & $1884 \pm 78$ & $1940 \pm 69$ & $2190 \pm 1059$ & $\mathbf{7099 \pm 645}$ \\
        Hopper-v5 & $2588 \pm 831$ & $2154 \pm 719$ & $2363 \pm 464$ & $\mathbf{2607 \pm 272}$ \\
        Humanoid-v5 & $1100 \pm 388$ & $1187 \pm 529$ & $1573 \pm 901$ & $\mathbf{3768 \pm 1290}$ \\
        HumanoidStandup-v5 & $158572 \pm 12341$ & $\mathbf{161394 \pm 6265}$ & $159009 \pm 8458$ & $156068 \pm 9614$ \\
        InvertedPendulum-v5 & $\mathbf{1000.0 \pm 0.0}$ & $952.6 \pm 149.9$ & $997.3 \pm 8.4$ & $999.2 \pm 2.6$ \\
        Pusher-v5 & $-36.2 \pm 2.0$ & $-36.2 \pm 2.1$ & $-36.5 \pm 2.4$ & $\mathbf{-30.6 \pm 3.0}$ \\
        Walker2d-v5 & $3450 \pm 475$ & $3262 \pm 951$ & $\mathbf{3812 \pm 675}$ & $3546 \pm 659$ \\
        \bottomrule
    \end{tabular}%
    }
\end{table}

\begin{figure}[htbp]
    \centering
    \includegraphics[width=0.92\textwidth]{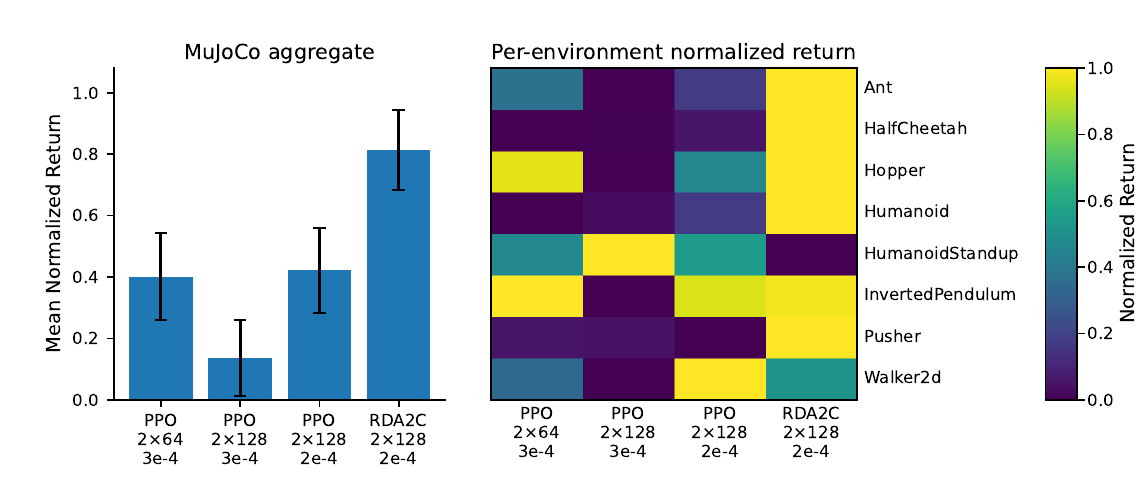}
    \caption{\textbf{PPO network capacity control on MuJoCo (GAE labels).} Returns are min--max normalized within each environment across the four configurations before aggregation. Error bars in the aggregate panel denote task SEM. Matching PPO's width to RDA2C does not reproduce the RDA2C gains.}
    \label{fig:ppo_capacity_control_appendix}
\end{figure}

\FloatBarrier
\section{RDA2C Gradient Boosting with XGBoost}
\label{app:xgb_trees}

\FloatBarrier
\subsection{Overview}

To complement our neural network instantiation of the dual predictor $Z_\theta$, we performed a small study in which $Z_\theta$ is instantiated as a gradient boosted decision tree regressor (XGBoost).
This setting makes the connection between dual averaging and additive function approximation particularly explicit, and it allows us to compare two concrete approximations of the running dual average:
\texttt{xgb\_boosting} (incremental boosting of the cumulative dual) and
\texttt{xgb\_dagger} (DAgger style dataset aggregation with a FIFO buffer).

\paragraph{Supervised representation.}
In both variants, we train a regressor on features
$\phi(s,a)=[s,\mathrm{onehot}(a)]$ and targets given by an advantage label derived
from a learned value function.
The resulting predictor $\widehat{Z}(s,a)$ is converted into a policy by the
closed-form mirror map:
$$
\pi(a\mid s) \propto \exp\!\left(\frac{\widehat{Z}(s,a)}{\alpha}\right),
$$
using the same entropy regularized Boltzmann form as in the main algorithm.
In all runs reported here, we apply clipped importance sampling (IPS) by
reweighting advantage targets by $1/\pi_{\mathrm{beh}}(a\mid s)$ with a maximum
weight (\texttt{iw\_max}). See the configuration tables below.

\paragraph{Incremental boosting (\texttt{xgb\_boosting}).}
We maintain a single XGBoost booster whose prediction represents an
\emph{unnormalized} cumulative dual
$U(s,a)\approx \sum_{k=1}^{K} g_k(s,a)$, where $g_k$ is the (normalized and
optionally IPS weighted) advantage signal from update $k$.
Operationally, after each minibatch we add one tree (warm start) using a
\emph{dual target} $y = U_{\mathrm{prev}} + g$ so that the new tree fits the
residual advantage.
At action selection time, logits are formed as
$U(s,a)/(K\lambda+\beta)$, where $K$ is the current number of boosted rounds.
This matches the scaling implied by the ratio $Z_t/\alpha_t$ in
\eqref{eq:closed_form_final} when identifying $K$ with the number of dual updates.

\paragraph{Dataset aggregation (\texttt{xgb\_dagger}).}
We maintain a FIFO buffer of advantage labeled tuples $(s,a,\hat A)$ collected
from all past rollouts (up to a fixed capacity), and periodically continue
boosting on the full buffer.
This corresponds to fitting a single predictor to the pooled dataset, which
approximates the average dual $Z_t$ under the mixture of visited state--action
distributions.
The policy uses the standard temperature schedule
$\alpha_t=\lambda+\beta/t$ and applies the
same softmax mapping.

\FloatBarrier
\subsection{Hyperparameters}

\paragraph{Experimental protocol.}
We ran both variants for $10^6$ environment steps on \texttt{CartPole-v1} and
\texttt{Acrobot-v1}, using 8 parallel environments with rollout length 256 (batch
size 2048).
We report mean $\pm$ standard deviation of the final episodic return over 5 random
seeds.

\paragraph{Hyperparameters.}
Tables~\ref{tab:xgb_dagger_hparams} and~\ref{tab:xgb_boosting_hparams} list the
configuration used for each environment.

\begin{table}[htbp]
    \centering
    \caption{\textbf{XGBoost dataset aggregation (\texttt{xgb\_dagger}) configurations.}}
    \label{tab:xgb_dagger_hparams}
    \small
    \begin{tabular}{lcc}
        \toprule
        \textbf{Parameter} & \textbf{CartPole-v1} & \textbf{Acrobot-v1} \\
        \midrule
        Total timesteps & $1{,}000{,}000$ & $1{,}000{,}000$ \\
        Num envs / steps & $8$ / $256$ & $8$ / $256$ \\
        Batch size & $2048$ & $2048$ \\
        $\gamma$ / GAE $\lambda_{\mathrm{GAE}}$ & $0.99$ / $0.8$ & $0.99$ / $0.8$ \\
        Value LR / epochs & $2.5{\times}10^{-3}$ / $4$ & $2.5{\times}10^{-3}$ / $4$ \\
        Anneal LR & \texttt{True} & \texttt{False} \\
        Max grad norm & $0.5$ & $0.5$ \\
        FIFO capacity & $100{,}000$ & $200{,}000$ \\
        Fit every (iters) & $4$ & $2$ \\
        Advantage clip & $10$ & $25$ \\
        IPS / \texttt{iw\_max} & \texttt{True} / $10$ & \texttt{True} / $5$ \\
        $\alpha_t$ schedule & $\lambda{=}0.15,\,\beta{=}2$ & $\lambda{=}0.2,\,\beta{=}2$ \\
        Trees per fit & $1$ & $2$ \\
        Tree LR (\texttt{eta}) & $0.05$ & $0.05$ \\
        Max depth & $3$ & $7$ \\
        Min child weight & $3$ & $3$ \\
        Subsample / colsample & $1.0$ / $1.0$ & $1.0$ / $1.0$ \\
        \bottomrule
    \end{tabular}
\end{table}

\begin{table}[htbp]
    \centering
    \caption{\textbf{XGBoost incremental boosting (\texttt{xgb\_boosting}) configurations.}}
    \label{tab:xgb_boosting_hparams}
    \small
    \begin{tabular}{lcc}
        \toprule
        \textbf{Parameter} & \textbf{CartPole-v1} & \textbf{Acrobot-v1} \\
        \midrule
        Total timesteps & $1{,}000{,}000$ & $1{,}000{,}000$ \\
        Num envs / steps & $8$ / $256$ & $8$ / $256$ \\
        Batch size & $2048$ & $2048$ \\
        $\gamma$ / GAE $\lambda_{\mathrm{GAE}}$ & $0.99$ / $0.8$ & $0.99$ / $0.8$ \\
        Value LR / epochs & $3{\times}10^{-3}$ / $1$ & $10^{-2}$ / $1$ \\
        Anneal LR & \texttt{True} & \texttt{True} \\
        Max grad norm & $10$ & $10$ \\
        Normalize adv & \texttt{True} & \texttt{True} \\
        Advantage clip & $5$ & $5$ \\
        IPS / \texttt{iw\_max} & \texttt{True} / $10$ & \texttt{True} / $10$ \\
        $\alpha_t$ schedule & $\lambda{=}10^{-4},\,\beta{=}1$ & $\lambda{=}10^{-4},\,\beta{=}1$ \\
        Tree LR (\texttt{eta}) & $0.5$ & $0.45$ \\
        Max depth & $3$ & $3$ \\
        Min child weight & $5$ & $5$ \\
        Subsample / colsample & $1.0$ / $1.0$ & $1.0$ / $1.0$ \\
        \bottomrule
    \end{tabular}
\end{table}

\FloatBarrier
\subsection{Results}

\begin{table}[htbp]
    \centering
    \caption{\textbf{RDA2C gradient boosting results.} Final episodic return (mean $\pm$ std over 5 seeds).}
    \label{tab:xgb_results}
    \begin{tabular}{llc}
        \toprule
        \textbf{Environment} & \textbf{Method} & \textbf{Final return} \\
        \midrule
        CartPole-v1 & \texttt{xgb\_dagger} & $500.00 \pm 0.00$ \\
        CartPole-v1 & \texttt{xgb\_boosting} & $500.00 \pm 0.00$ \\
        Acrobot-v1 & \texttt{xgb\_dagger} & $-84.15 \pm 4.57$ \\
        Acrobot-v1 & \texttt{xgb\_boosting} & $-88.98 \pm 2.43$ \\
        \bottomrule
    \end{tabular}
\end{table}

\begin{figure}[htbp]
    \centering
    \begin{subfigure}[t]{0.47\textwidth}
        \centering
        \includegraphics[width=\textwidth]{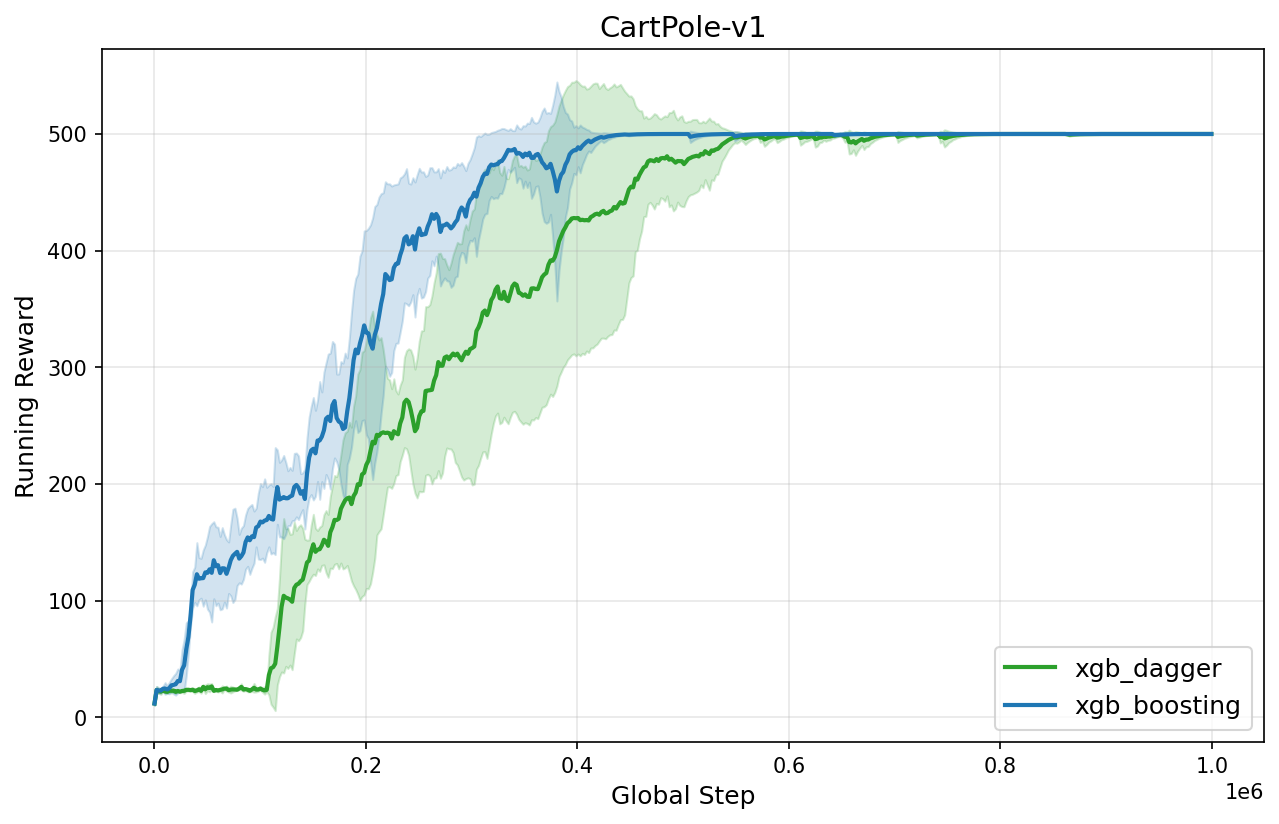}
    \end{subfigure}\hfill
    \begin{subfigure}[t]{0.47\textwidth}
        \centering
        \includegraphics[width=\textwidth]{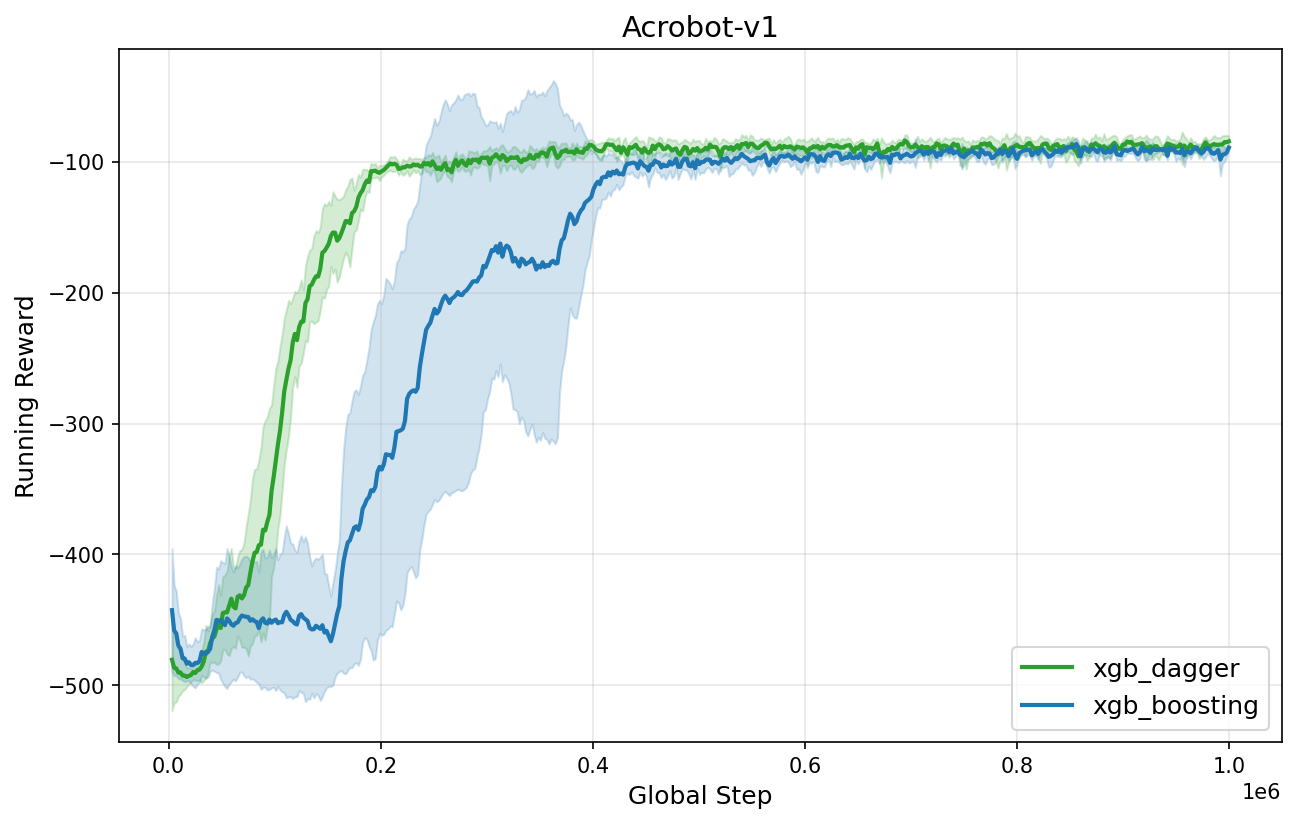}
    \end{subfigure}
    \caption{\textbf{RDA2C gradient boosting learning curves.} On CartPole-v1 (left), both \texttt{xgb\_dagger} and \texttt{xgb\_boosting} reach $500$ within $10^6$ steps. On Acrobot-v1 (right), \texttt{xgb\_dagger} is slightly more sample efficient early, and both methods converge similarly by $10^6$ steps.}
    \label{fig:xgb_learning_curves}
\end{figure}

\newpage
\section{Proofs and implications for policy performance}
\label{app:proofs}

The proof below analyzes the learned replay actor through a direct regularized score to value argument.

The three layers are:
\begin{enumerate}
    \item \emph{Replay window Bayes target and local drift.} The replay regression Bayes target $Z_t^A$ is a convex combination of recent advantages at each state action pair, so it tracks the current advantage $A_t$ up to a local window drift term.
    \item \emph{Direct regularized learned readout bound.} Since the learned policy $\pi_{\theta_t}$ is the exact entropy regularized maximizer of the learned score $Z_{\theta_{t-1}}$, we prove regularized approximate optimality for $\pi_{\theta_t}$. The residual $Z_{\theta_{t-1}}-Z_t^A$ against the regression Bayes target $Z_t^A$ enters through a direct inner product error under the comparator occupancy.
    \item \emph{Replay decomposition.} The learned score error $\|Z_{\theta_{t-1}}-Z_t^A\|_{2,\mu_t}$ is decomposed into supervised fitting error, finite replay concentration, and critic label bias.
\end{enumerate}

\paragraph{Estimation chain.}
The estimation chain summarizes the construction, and each arrow is one estimation step:
\begin{equation*}
Z_t \;\xrightarrow{\text{replay weighting}}\; \widehat Z_t \;\xrightarrow{\text{regression}}\; Z_\theta \;\xrightarrow{\text{mirror map}}\; \pi_{\theta_{t+1}}.
\end{equation*}
The exact score $Z_t$ is the state-weighted average advantage. Replay sampling replaces it with the Bayes target $\widehat Z_t$, which carries the behavior action weights of the stored tuples. Fitting gives $Z_\theta$, and the entropy mirror map turns $Z_\theta$ into the next policy. Section~\ref{sec:theory} bounds each gap in this chain. For states outside the support of $\bar d_t$, $Z_t$ can be assigned any fixed value, since those states have zero weight in the objective in \eqref{eq:rda_update} until later data cover them.

\FloatBarrier
\subsection{Performance difference lemma}
\label{app:performance_difference_lemma}

With normalized discounted occupancies, the following lemma converts policy space advantage feedback into value differences.

\begin{lemma}[Performance difference lemma \citep{kakade2002approximately}]
\label{lem:performance_difference_lemma}
For any stationary comparator $\pi^\star$ and any policy $\pi_{\theta_t}$, with $A_t=A^{\pi_{\theta_t}}$,
\begin{equation}
\label{eq:pdl_advantage_feedback}
    J(\pi^\star)-J(\pi_{\theta_t})
    =
    \frac{1}{1-\gamma}
    \mathbb E_{s\sim d^{\pi^\star}}
    \left[
        \left\langle
        \pi^\star(\cdot\mid s)-\pi_{\theta_t}(\cdot\mid s),
        A_t(s,\cdot)
        \right\rangle
    \right].
\end{equation}
\end{lemma}

\FloatBarrier
\subsection{Window notation and replay Bayes target}

Fix an exponent $\zeta\in(0,1)$. For every $t\ge2$, let
$$
    W_t=\max\{1,\lfloor t^\zeta\rfloor\}
$$
be the number of policy iterations represented in the replay window, and define
$$
    \mathcal W_t
    =
    \{t-W_t,\ldots,t-1\}.
$$
Thus $|\mathcal W_t|=W_t$ and $W_t\le t^\zeta$. Replacing $\lfloor t^\zeta\rfloor$ by $\lceil t^\zeta\rceil$ only changes constants and has no effect on the rate. If the implementation measures replay capacity in samples rather than iterations, a fixed batch size converts sample capacity into a constant multiple of $W_t$, and that constant is absorbed into $C_{\mathrm{var}}$ below.

Throughout this section,
$$
    J(\pi)=J_\rho(\pi):=V^\pi(\rho)=\mathbb E_{s_0\sim\rho}[V^\pi(s_0)]
$$
denotes the scalar start distribution value from the fixed initial distribution $\rho$. We also assume that rewards are uniformly bounded by a global constant $R_{\max}<\infty$:
$$
    |r(s,a)|\le R_{\max}
    \qquad\text{for all }(s,a).
$$
We write $G_A$ for a uniform advantage bound. Under this two sided reward bound one may take $G_A=2R_{\max}/(1-\gamma)$, although the analysis below only uses $\|A_t\|_\infty\le G_A$.

Write
$$
    A_k(s,a)=A^{\pi_{\theta_k}}(s,a)
$$
for the advantage of the learned policy at round $k$.

When we need the population replay distribution explicitly, we reuse the occupancy notation from Section~\ref{sec:preliminaries}:
$$
    \mu_k(s,a)
    =
    \mu^{\pi_{\theta_k}}(s,a)
    =
    d^{\pi_{\theta_k}}(s)\pi_{\theta_k}(a\mid s)
$$
for the state action occupancy of the data generated at round $k$, and define the window replay marginal
$$
    \mu_t(s,a)
    =
    \frac1{W_t}
    \sum_{k\in\mathcal W_t}\mu_k(s,a).
$$
Thus $\mu_t$ is the population state action distribution represented by the current replay window. Ratios involving $\mu_t(s,a)$ are always understood on the support where $\mu_t(s,a)>0$.

The score learned by the squared loss replay regression has, with true advantage labels, the population conditional mean target
$$
    Z_t^A(s,a)
    =
    \begin{cases}
    \displaystyle
    \frac{
        \sum_{k\in\mathcal W_t}\mu_k(s,a)A_k(s,a)
    }{
        \sum_{k\in\mathcal W_t}\mu_k(s,a)
    },
    &\displaystyle \sum_{k\in\mathcal W_t}\mu_k(s,a)>0,\\[2.25ex]
    \displaystyle
    \frac1{W_t}
    \sum_{k\in\mathcal W_t}A_k(s,a),
    &\displaystyle \sum_{k\in\mathcal W_t}\mu_k(s,a)=0.
    \end{cases}
$$
On the support of $\mu_t$ this is the replay regression Bayes target. The second line is a harmless tie breaker on state action pairs with zero window replay mass. It is included only so that the pointwise supremum $\|A_t-Z_t^A\|_\infty$ used in Layer~1 is defined on every $(s,a)$, and it does not affect the $L_2(\mu_t)$ score error term, which never sees the off support branch. At every $(s,a)$, $Z_t^A(s,a)$ is therefore a convex combination of the recent advantages $\{A_k(s,a):k\in\mathcal W_t\}$, with coefficients proportional to the joint occupancy $\mu_k(s,a)$ on the support and uniform off it. This convex combination property is all that the local drift bound uses.

The algorithm uses a fixed stabilizer $\beta\ge0$. The temperature of the readout that produces the round-$t$ policy is that of \eqref{eq:closed_form_final},
$$
    \alpha_{t-1}
    =
    \lambda+\frac{\beta}{t-1},
    \qquad
    \lambda>0.
$$

For $t\ge2$, define the learned readout by
$$
    \pi_{\theta_t}(\cdot\mid s)
    =
    \operatorname{softmax}\!\left(
        \frac{Z_{\theta_{t-1}}(s,\cdot)}{\alpha_{t-1}}
    \right).
$$
Here $Z_{\theta_{t-1}}$ denotes the score learned using data available before round $t$. The proof is stated for the average over $t=2,\ldots,T$, so every policy in the average is read out from a nonempty past window.

\FloatBarrier
\subsection{Layer 1: replay window Bayes target and local drift}

The role of Layer~1 is now only to connect the window target appearing in the score residual to the current advantage in the performance difference lemma. The main bridge is the regression Bayes target $Z_t^A$. Since $Z_t^A(s,a)$ is a convex combination of recent advantages at each fixed $(s,a)$, the local policy path stability assumption implies $\|A_t-Z_t^A\|_\infty=O(t^{\zeta-1})$.

\begin{assumption}[Learned policy path stability]
\label{ass:learned_policy_path_stability_clean}
There exists $C_\pi>0$ such that the learned policies satisfy
$$
    \sup_s
    \left\|
        \pi_{\theta_{k+1}}(\cdot\mid s)
        -
        \pi_{\theta_k}(\cdot\mid s)
    \right\|_1
    \le
    \frac{C_\pi}{k}
$$
for every $k\ge1$.
\end{assumption}

This is the algorithm side stability condition needed because the replay buffer is generated by the learned policies $\pi_{\theta_k}$, not by an auxiliary exact windowed readout. For learned policies, this need not follow from average fitting error alone: consecutive learned policies can move by a constant amount while maintaining small average replay error. We therefore state this assumption explicitly. It is natural for warm started actor updates with $1/k$-scale step sizes, trust region constraints, clipping, or projection onto a path stable policy class.

\begin{lemma}[Value and advantage Lipschitzness in policy]
\label{lem:value_advantage_lipschitz_clean}
For any two stationary policies $\pi$ and $\pi'$,
$$
    \left\|V^\pi-V^{\pi'}\right\|_\infty
    \le
    \frac{R_{\max}}{(1-\gamma)^2}
    \sup_s
    \left\|
        \pi(\cdot\mid s)-\pi'(\cdot\mid s)
    \right\|_1.
$$
Moreover,
$$
    \left\|A^\pi-A^{\pi'}\right\|_\infty
    \le
    \frac{(1+\gamma)R_{\max}}{(1-\gamma)^2}
    \sup_s
    \left\|
        \pi(\cdot\mid s)-\pi'(\cdot\mid s)
    \right\|_1.
$$
\end{lemma}

\begin{proof}
For any state $s$,
$$
\begin{aligned}
V^\pi(s)-V^{\pi'}(s)
&=
\left\langle
    \pi(\cdot\mid s)-\pi'(\cdot\mid s),
    Q^{\pi'}(s,\cdot)
\right\rangle \\
&\qquad+
\mathbb E_{a\sim\pi(\cdot\mid s)}
\left[
    Q^\pi(s,a)-Q^{\pi'}(s,a)
\right].
\end{aligned}
$$
By the Bellman relation,
$$
    Q^\pi(s,a)-Q^{\pi'}(s,a)
    =
    \gamma
    \mathbb E_{s'\sim P(\cdot\mid s,a)}
    \left[
        V^\pi(s')-V^{\pi'}(s')
    \right].
$$
Therefore, by H\"older's inequality,
$$
\begin{aligned}
\left|V^\pi(s)-V^{\pi'}(s)\right|
&\le
\left\|\pi(\cdot\mid s)-\pi'(\cdot\mid s)\right\|_1
\left\|Q^{\pi'}\right\|_\infty
+
\gamma
\left\|V^\pi-V^{\pi'}\right\|_\infty.
\end{aligned}
$$
Taking the supremum over $s$ and using $\left\|Q^{\pi'}\right\|_\infty\le R_{\max}/(1-\gamma)$ gives
$$
    \left\|V^\pi-V^{\pi'}\right\|_\infty
    \le
    \frac{R_{\max}}{1-\gamma}
    \sup_s
    \left\|
        \pi(\cdot\mid s)-\pi'(\cdot\mid s)
    \right\|_1
    +
    \gamma
    \left\|V^\pi-V^{\pi'}\right\|_\infty.
$$
Rearranging proves the value bound. For the advantage bound,
$$
    \left\|Q^\pi-Q^{\pi'}\right\|_\infty
    \le
    \gamma
    \left\|V^\pi-V^{\pi'}\right\|_\infty,
$$
and since $A^\pi=Q^\pi-V^\pi$,
$$
    \left\|A^\pi-A^{\pi'}\right\|_\infty
    \le
    \left\|Q^\pi-Q^{\pi'}\right\|_\infty
    +
    \left\|V^\pi-V^{\pi'}\right\|_\infty
    \le
    (1+\gamma)
    \left\|V^\pi-V^{\pi'}\right\|_\infty.
$$
Substituting the value bound completes the proof.
\end{proof}

\begin{corollary}[Local advantage drift]
\label{cor:local_advantage_drift_clean}
Suppose Assumption~\ref{ass:learned_policy_path_stability_clean} holds. If $A_k=A^{\pi_{\theta_k}}$, then for every $k\ge1$,
$$
    \left\|A_{k+1}-A_k\right\|_\infty
    \le
    \frac{L_A}{k},
    \qquad
    L_A
    =
    \frac{(1+\gamma)R_{\max}C_\pi}{(1-\gamma)^2}.
$$
\end{corollary}

\begin{proof}
By Lemma~\ref{lem:value_advantage_lipschitz_clean},
$$
\|A_{k+1}-A_k\|_\infty
\le
\frac{(1+\gamma)R_{\max}}{(1-\gamma)^2}
\sup_s\|\pi_{\theta_{k+1}}(\cdot\mid s)-\pi_{\theta_k}(\cdot\mid s)\|_1.
$$
Applying Assumption~\ref{ass:learned_policy_path_stability_clean} to the right hand side gives the claim with $L_A = (1+\gamma) R_{\max}C_\pi/(1-\gamma)^2$.
\end{proof}

\begin{lemma}[The recent window tracks the current advantage]
\label{lem:recent_window_tracks_current_advantage_clean}
Under Corollary~\ref{cor:local_advantage_drift_clean}, for every $t\ge2$,
$$
    \left\|A_t-Z_t^A\right\|_\infty
    \le
    C_{\mathrm{loc}}t^{\zeta-1},
$$
where one may take
$$
    C_{\mathrm{loc}}
    =
    \frac{L_A}{1-2^{\zeta-1}}.
$$
Consequently, for every $T\ge2$,
$$
    \frac{2}{T-1}\sum_{t=2}^T
    \left\|A_t-Z_t^A\right\|_\infty
    \le
    \frac{2C_{\mathrm{loc}}}{\zeta}
    \frac{T^\zeta}{T-1}.
$$
\end{lemma}

\begin{proof}
For $k\in\mathcal W_t$, Corollary~\ref{cor:local_advantage_drift_clean} gives
$$
\begin{aligned}
\left\|A_t-A_k\right\|_\infty
&\le
\sum_{j=k}^{t-1}
\left\|A_{j+1}-A_j\right\|_\infty \\
&\le
L_A\sum_{j=k}^{t-1}\frac1j.
\end{aligned}
$$

Since $k\ge t-W_t$ and $W_t\le t^\zeta$, we have $k\ge t-t^\zeta$. Also, since $t^\zeta/t=t^{\zeta-1}\le2^{\zeta-1}$ for every $t\ge2$,
$$
    t-W_t
    \ge
    t-t^\zeta
    \ge
    \left(1-2^{\zeta-1}\right)t.
$$
Therefore,
$$
    \sum_{j=k}^{t-1}\frac1j
    \le
    \frac{t-k}{t-W_t}
    \le
    \frac{W_t}{(1-2^{\zeta-1})t}
    \le
    \frac{t^{\zeta-1}}{1-2^{\zeta-1}}.
$$
Thus, uniformly over $k\in\mathcal W_t$,
$$
    \left\|A_t-A_k\right\|_\infty
    \le
    C_{\mathrm{loc}}t^{\zeta-1}.
$$
For each fixed $(s,a)$, $Z_t^A(s,a)$ is a convex combination of the values $\{A_k(s,a):k\in\mathcal W_t\}$ (coefficients proportional to $\mu_k(s,a)$ on the support of $\mu_t$, and uniform off it). Hence
$$
\begin{aligned}
\left|A_t(s,a)-Z_t^A(s,a)\right|
&\le
\max_{k\in\mathcal W_t}
\left|A_t(s,a)-A_k(s,a)\right| \\
&\le
C_{\mathrm{loc}}t^{\zeta-1}.
\end{aligned}
$$
Taking the supremum over $(s,a)$ proves the pointwise bound.
For the average, multiply the pointwise bound by $2$ and use
$$
    \sum_{t=2}^T t^{\zeta-1}
    \le
    \int_1^T x^{\zeta-1}\,dx
    =
    \frac{T^\zeta-1}{\zeta}
    \le
    \frac{T^\zeta}{\zeta}.
$$
\end{proof}

\FloatBarrier
\subsection{Layer 2: direct regularized learned readout value transfer}
\label{app:layer2_coverage}

Layer~2 proves regularized approximate optimality directly for the learned policy $\pi_{\theta_t}$. The only transfer needed is from an $L_2(\mu_t)$ score error bound to a comparator state expectation. This is where the quadratic coverage coefficient enters.

\paragraph{Where the action sampling mismatch is paid.}
The replay Bayes target weights each action by how often the retained behavior policies took it, so the fitted score is accurate where replay puts probability mass and is unconstrained elsewhere. The analysis does not estimate this mismatch. It bounds it through the assumption that the coverage coefficient below is finite, which holds exactly when the replay occupancy $\mu_t$ reaches every action at every state the comparator visits. The burden of the mismatch therefore sits on this coefficient. It grows when the comparator prefers actions the replayed policies rarely take, and the guarantee weakens accordingly. Removing the assumption, for example through importance weighting on the stored labels or through a drift argument, is open.

For a comparator $\pi^\star$, define the quadratic replay to comparator coverage coefficient
$$
    C_t^{\mathrm{cov}}(\pi^\star)
    =
    \operatorname*{ess\,sup}_{s,a}
    \frac{d^{\pi^\star}(s)}{\mu_t(s,a)},
$$
with the convention that the ratio is $+\infty$ whenever $d^{\pi^\star}(s)>0$ and $\mu_t(s,a)=0$. Let
$$
    \mu_t^S(s)
    =
    \sum_a \mu_t(s,a)
$$
be the replay state marginal. In finite action spaces this coefficient is easy to interpret. If the replay conditional action distribution satisfies
$$
    \rho_t(a\mid s)
    =
    \frac{\mu_t(s,a)}{\mu_t^S(s)}
    \ge
    \eta
$$
whenever $\mu_t^S(s)>0$, then
$$
    C_t^{\mathrm{cov}}(\pi^\star)
    \le
    \frac1\eta
    \operatorname*{ess\,sup}_{s}
    \frac{d^{\pi^\star}(s)}{\mu_t^S(s)}.
$$
Thus action exploration still enters the theory through the coverage transfer, even though it no longer enters the replay reweighting lemma.

\begin{lemma}[Direct regularized optimality of the learned readout]
\label{lem:direct_regularized_learned_readout_clean}
For every comparator policy $\pi^\star$, every state $s$, and every $t\ge2$,
$$
\begin{aligned}
&\left\langle
    \pi^\star(\cdot\mid s)-\pi_{\theta_t}(\cdot\mid s),
    A_t(s,\cdot)
\right\rangle
+
\lambda
\left(
    \varphi(\pi_{\theta_t}(\cdot\mid s))
    -
    \varphi(\pi^\star(\cdot\mid s))
\right) \\
&\qquad\le
\frac{\beta}{t-1}\ln|\mathcal A|
+
2\left\|A_t(s,\cdot)-Z_t^A(s,\cdot)\right\|_\infty
+
2\left\|Z_{\theta_{t-1}}(s,\cdot)-Z_t^A(s,\cdot)\right\|_2.
\end{aligned}
$$
\end{lemma}

\begin{proof}
Fix $s$ and $t\ge2$. By definition of the readout, with $\alpha_{t-1}=\lambda+\beta/(t-1)$,
$$
    \pi_{\theta_t}(\cdot\mid s)
    \in
    \arg\max_{p\in\Delta(\mathcal A)}
    \left\{
        \left\langle p,Z_{\theta_{t-1}}(s,\cdot)\right\rangle
        -
        \alpha_{t-1}\varphi(p)
    \right\}.
$$
Since this objective is concave, evaluating it at $\pi_{\theta_t}(\cdot\mid s)$ and at $\pi^\star(\cdot\mid s)$ and using maximality gives
$$
    \left\langle
        \pi^\star(\cdot\mid s)-\pi_{\theta_t}(\cdot\mid s),
        Z_{\theta_{t-1}}(s,\cdot)
    \right\rangle
    +
    \alpha_{t-1}
    \left(
        \varphi(\pi_{\theta_t}(\cdot\mid s))
        -
        \varphi(\pi^\star(\cdot\mid s))
    \right)
    \le 0.
$$
Writing $\alpha_{t-1}=\lambda+\beta/(t-1)$ and carrying the score from $Z_{\theta_{t-1}}$ to $A_t$ through the regression Bayes target $Z_t^A$ (all quantities evaluated at $s$),
$$
\begin{aligned}
&\left\langle
    \pi^\star-\pi_{\theta_t},A_t
\right\rangle
+
\lambda
\left(
    \varphi(\pi_{\theta_t})-\varphi(\pi^\star)
\right) \\
&=
\underbrace{
\left\langle
    \pi^\star-\pi_{\theta_t},Z_{\theta_{t-1}}
\right\rangle
+
\alpha_{t-1}
\left(
    \varphi(\pi_{\theta_t})-\varphi(\pi^\star)
\right)
}_{\le 0}
-
\frac{\beta}{t-1}
\left(
    \varphi(\pi_{\theta_t})-\varphi(\pi^\star)
\right) \\
&\qquad
+
\left\langle
    \pi^\star-\pi_{\theta_t},A_t-Z_t^A
\right\rangle
+
\left\langle
    \pi^\star-\pi_{\theta_t},Z_t^A-Z_{\theta_{t-1}}
\right\rangle.
\end{aligned}
$$
We bound the three remaining terms. Since $\varphi(p)=\sum_a p(a)\ln p(a)\in[-\ln|\mathcal A|,0]$,
$$
    -\frac{\beta}{t-1}
    \left(
        \varphi(\pi_{\theta_t})-\varphi(\pi^\star)
    \right)
    \le
    \frac{\beta}{t-1}\ln|\mathcal A|.
$$
By H\"older's inequality with $\|\pi^\star-\pi_{\theta_t}\|_1\le2$,
$$
    \left\langle
        \pi^\star-\pi_{\theta_t},A_t-Z_t^A
    \right\rangle
    \le
    2\left\|A_t(s,\cdot)-Z_t^A(s,\cdot)\right\|_\infty .
$$
By the Cauchy--Schwarz inequality with $\|\pi^\star-\pi_{\theta_t}\|_2\le\|\pi^\star\|_2+\|\pi_{\theta_t}\|_2\le2$ (each distribution has $\ell_2$ norm at most $1$),
$$
    \left\langle
        \pi^\star-\pi_{\theta_t},Z_t^A-Z_{\theta_{t-1}}
    \right\rangle
    \le
    2\left\|Z_{\theta_{t-1}}(s,\cdot)-Z_t^A(s,\cdot)\right\|_2 .
$$
Summing the three bounds proves the claim.
\end{proof}

\begin{proposition}[Direct regularized windowed value transfer]
\label{prop:windowed_learned_readout_value_transfer_clean}
For every comparator policy $\pi^\star$ and every $T\ge2$,
$$
\begin{aligned}
&\frac1{T-1}\sum_{t=2}^T
\left[
\left(J(\pi^\star)-J(\pi_{\theta_t})\right)
+
\frac{\lambda}{1-\gamma}
\mathbb E_{s\sim d^{\pi^\star}}
\left(
    \varphi(\pi_{\theta_t}(\cdot\mid s))
    -
    \varphi(\pi^\star(\cdot\mid s))
\right)
\right] \\
&\quad\le
\frac1{1-\gamma}
\Bigg[
\frac{\beta\ln|\mathcal A|}{T-1}\sum_{t=2}^T\frac1{t-1}
+
\frac2{T-1}\sum_{t=2}^T
\left\|A_t-Z_t^A\right\|_\infty \\
&\qquad+
\frac{2}{T-1}\sum_{t=2}^T
\sqrt{C_t^{\mathrm{cov}}(\pi^\star)}
\left\|
    Z_{\theta_{t-1}}-Z_t^A
\right\|_{2,\mu_t}
\Bigg].
\end{aligned}
$$
Consequently, if $C_t^{\mathrm{cov}}(\pi^\star)\le C_{\mathrm{cov}}(\pi^\star)$ for all $t=2,\ldots,T$, the last sum is bounded by $2\sqrt{C_{\mathrm{cov}}(\pi^\star)}\cdot\frac1{T-1}\sum_{t=2}^T\|Z_{\theta_{t-1}}-Z_t^A\|_{2,\mu_t}$.
\end{proposition}

\begin{proof}
By the performance difference lemma,
$$
    J(\pi^\star)-J(\pi_{\theta_t})
    =
    \frac1{1-\gamma}
    \mathbb E_{s\sim d^{\pi^\star}}
    \left[
        \left\langle
            \pi^\star(\cdot\mid s)-\pi_{\theta_t}(\cdot\mid s),
            A_t(s,\cdot)
        \right\rangle
    \right],
$$
where $A_t=A^{\pi_{\theta_t}}$. Apply Lemma~\ref{lem:direct_regularized_learned_readout_clean} pointwise in $s$ and take $\mathbb E_{s\sim d^{\pi^\star}}$:
$$
\begin{aligned}
&J(\pi^\star)-J(\pi_{\theta_t})
+
\frac{\lambda}{1-\gamma}
\mathbb E_{s\sim d^{\pi^\star}}
\left(
    \varphi(\pi_{\theta_t})-\varphi(\pi^\star)
\right) \\
&\quad\le
\frac1{1-\gamma}
\left[
\frac{\beta}{t-1}\ln|\mathcal A|
+
2\left\|A_t-Z_t^A\right\|_\infty
+
2\,\mathbb E_{s\sim d^{\pi^\star}}
\left\|
    Z_{\theta_{t-1}}(s,\cdot)-Z_t^A(s,\cdot)
\right\|_2
\right].
\end{aligned}
$$
For the last term, by Jensen's inequality and the change of measure bound (as in the coverage discussion above),
$$
    \mathbb E_{s\sim d^{\pi^\star}}
    \left\|
        Z_{\theta_{t-1}}(s,\cdot)-Z_t^A(s,\cdot)
    \right\|_2
    \le
    \left(
    \sum_{s,a}d^{\pi^\star}(s)
    \left(
        Z_{\theta_{t-1}}(s,a)-Z_t^A(s,a)
    \right)^2
    \right)^{1/2}
    \le
    \sqrt{C_t^{\mathrm{cov}}(\pi^\star)}
    \left\|
        Z_{\theta_{t-1}}-Z_t^A
    \right\|_{2,\mu_t}.
$$
Averaging over $t=2,\ldots,T$ gives the display. The final assertion uses $C_t^{\mathrm{cov}}(\pi^\star)\le C_{\mathrm{cov}}(\pi^\star)$.
\end{proof}

\FloatBarrier
\subsection{Layer 3: replay regression error decomposition}
\label{app:theory_setup}

Layer 2 reduces the learned policy value bound to the score error
$$
    \left\|Z_{\theta_{t-1}}-Z_t^A\right\|_{2,\mu_t},
$$
the residual of the learned score against the regression Bayes target $Z_t^A$.
The norm is taken under the window replay marginal $\mu_t$ defined above.

We recall the population replay target with true advantage labels,
$$
    Z_t^A(s,a)
    =
    \frac{
        \sum_{k\in\mathcal W_t}\mu_k(s,a)A_k(s,a)
    }{
        \sum_{k\in\mathcal W_t}\mu_k(s,a)
    }
$$
on the support of $\mu_t$. Its critic label analogue is
$$
    \widehat Z_t(s,a)
    =
    \frac{
        \sum_{k\in\mathcal W_t}\mu_k(s,a)\widehat A_k(s,a)
    }{
        \sum_{k\in\mathcal W_t}\mu_k(s,a)
    }.
$$
Let
$$
    \widetilde Z_t
$$
be the finite buffer empirical target represented by the actual replay sample. In the infinite buffer limit,
$$
    \widetilde Z_t=\widehat Z_t.
$$
The score error chain is
$$
Z_{\theta_{t-1}}
\quad\longrightarrow\quad
\widetilde Z_t
\quad\longrightarrow\quad
\widehat Z_t
\quad\longrightarrow\quad
Z_t^A.
$$

\begin{assumption}[Primitive windowed replay error bounds]
\label{ass:replay_error_decomp}
For every $T\ge2$, the following bounds hold.
\begin{enumerate}
    \item Fitting error:
    $$
    \frac1{T-1}\sum_{t=2}^T
    \left\|
        Z_{\theta_{t-1}}-
        \widetilde Z_t
    \right\|_{2,\mu_t}
    \le
    \bar\epsilon_{\mathrm{fit}}.
    $$

    That is, the fitting error is not growing on average.

    \item Finite replay concentration:
    $$
    \left\|
        \widetilde Z_t
        -
        \widehat Z_t
    \right\|_{2,\mu_t}
    \le
    C_{\mathrm{var}}t^{-\zeta/2}
    \qquad\text{for every }t\ge2.
    $$

    This rate is the standard $N^{-1/2}$ concentration for the empirical mean of $N_t\asymp t^\zeta$ bounded labels around its conditional expectation, with $C_{\mathrm{var}}$ depending on the label range $G_A$, the regression function class capacity, and a logarithmic factor from a union bound over $t$.

    \item Critic label bias:
    $$
    \frac1{T-1}\sum_{t=2}^T
    \left\|
        \widehat Z_t
        -
        Z_t^A
    \right\|_{2,\mu_t}
    \le
    \epsilon_0.
    $$

    That is, the critic bias is not growing.

\end{enumerate}
\end{assumption}

\begin{lemma}[Averaged windowed score error]
\label{lem:averaged_windowed_score_error_clean}
Under Assumption~\ref{ass:replay_error_decomp}, for every $T\ge2$,
$$
\begin{aligned}
&\frac1{T-1}\sum_{t=2}^T
\left\|
Z_{\theta_{t-1}}-Z_t^A
\right\|_{2,\mu_t} \\
&\qquad\le
\bar\epsilon_{\mathrm{fit}}
+
\epsilon_0
+
\frac{C_{\mathrm{var}}}{1-\zeta/2}
\frac{T^{1-\zeta/2}}{T-1}.
\end{aligned}
$$
\end{lemma}

\begin{proof}
For each $t$, apply the triangle inequality along the chain of replay targets ending at $Z_t^A$:
$$
\begin{aligned}
&\left\|
Z_{\theta_{t-1}}-Z_t^A
\right\|_{2,\mu_t} \\
&\quad\le
\left\|
Z_{\theta_{t-1}}-
\widetilde Z_t
\right\|_{2,\mu_t}
+
\left\|
\widetilde Z_t
-
\widehat Z_t
\right\|_{2,\mu_t}
+
\left\|
\widehat Z_t
-
Z_t^A
\right\|_{2,\mu_t}.
\end{aligned}
$$
Average over $t=2,\ldots,T$. The fitting and critic terms are controlled by the averaged bounds in Assumption~\ref{ass:replay_error_decomp}. For finite replay concentration,
$$
\begin{aligned}
\frac1{T-1}\sum_{t=2}^T C_{\mathrm{var}}t^{-\zeta/2}
&\le
\frac{C_{\mathrm{var}}}{T-1}
\sum_{t=2}^T t^{-\zeta/2} \\
&\le
\frac{C_{\mathrm{var}}}{1-\zeta/2}
\frac{T^{1-\zeta/2}}{T-1}.
\end{aligned}
$$
Combining these estimates proves the claim.
\end{proof}

\FloatBarrier
\subsection{Final windowed value bound}

This subsection contains the proofs of Theorem~\ref{thm:fitted_replay_value_bound} and Corollaries~\ref{cor:unregularized_rate} and~\ref{cor:annealed_rate} from the main text.

Define the learned policy regularized gap
$$
    \mathcal G_T^{\mathrm{reg}}(\pi^\star)
    :=
    \frac1{T-1}\sum_{t=2}^T
    \left[
    \left(J(\pi^\star)-J(\pi_{\theta_t})\right)
    +
    \frac{\lambda}{1-\gamma}
    \mathbb E_{s\sim d^{\pi^\star}}
    \left(
        \varphi(\pi_{\theta_t}(\cdot\mid s))
        -
        \varphi(\pi^\star(\cdot\mid s))
    \right)
    \right].
$$

\begin{proof}[Proof of Theorem~\ref{thm:fitted_replay_value_bound}]
Start from Proposition~\ref{prop:windowed_learned_readout_value_transfer_clean}, whose left hand side is $\mathcal G_T^{\mathrm{reg}}(\pi^\star)$. Bound the local drift term $\|A_t-Z_t^A\|_\infty$ using Lemma~\ref{lem:recent_window_tracks_current_advantage_clean} and the score error term using Lemma~\ref{lem:averaged_windowed_score_error_clean}. Under the reward bound $|r(s,a)|\le R_{\max}$,
$$
    L_A=\frac{(1+\gamma)R_{\max}C_\pi}{(1-\gamma)^2},
    \qquad
    C_{\mathrm{loc}}
    =
    \frac{(1+\gamma)R_{\max}C_\pi}{(1-\gamma)^2(1-2^{\zeta-1})}.
$$
The stabilizer contribution satisfies
$$
    \frac{\beta\ln|\mathcal A|}{T-1}
    \sum_{t=2}^T\frac1{t-1}
    \le
    \frac{\beta\ln|\mathcal A|(1+\ln T)}{T-1}.
$$
Substituting these estimates gives the bound in Theorem~\ref{thm:fitted_replay_value_bound}. Equivalently, after expanding $C_{\mathrm{loc}}$,
$$
\begin{aligned}
\mathcal G_T^{\mathrm{reg}}(\pi^\star)
&\le
\frac1{1-\gamma}
\Bigg[
\frac{\beta\ln|\mathcal A|(1+\ln T)}{T-1}
+
\frac{2(1+\gamma)R_{\max}C_\pi}{\zeta(1-\gamma)^2(1-2^{\zeta-1})}
\frac{T^\zeta}{T-1} \\
&\qquad+
2\sqrt{C_{\mathrm{cov}}(\pi^\star)}
\left(
\bar\epsilon_{\mathrm{fit}}
+
\epsilon_0
+
\frac{C_{\mathrm{var}}}{1-\zeta/2}
\frac{T^{1-\zeta/2}}{T-1}
\right)
\Bigg].
\end{aligned}
$$
For the balanced choice $\zeta=2/3$,
$$
    \frac{2C_{\mathrm{loc}}}{\zeta}
    =
    \frac{3(1+\gamma)R_{\max}C_\pi}
    {(1-\gamma)^2(1-2^{-1/3})},
    \qquad
    \frac{C_{\mathrm{var}}}{1-\zeta/2}
    =
    \frac32 C_{\mathrm{var}},
$$
so
$$
\begin{aligned}
\mathcal G_T^{\mathrm{reg}}(\pi^\star)
&\le
\frac1{1-\gamma}
\Bigg[
\frac{\beta\ln|\mathcal A|(1+\ln T)}{T-1}
+
\frac{3(1+\gamma)R_{\max}C_\pi}{(1-\gamma)^2(1-2^{-1/3})}
\frac{T^{2/3}}{T-1} \\
&\qquad+
2\sqrt{C_{\mathrm{cov}}(\pi^\star)}
\left(
\bar\epsilon_{\mathrm{fit}}
+
\epsilon_0
+
\frac32 C_{\mathrm{var}}
\frac{T^{2/3}}{T-1}
\right)
\Bigg].
\end{aligned}
$$
Thus, for fixed $\lambda>0$, bounded coverage, and vanishing fitting and critic bias floors, $\mathcal G_T^{\mathrm{reg}}(\pi^\star)=O(T^{-1/3})$.
\end{proof}

\begin{remark}[Reading the bound]

The regularized gap is the natural goal for an entropy-regularized method. It is the quantity the RDA iterates optimize, the entropy floor and the stabilizer control it directly, and it converts to the ordinary value gap at the cost of one explicit entropy bias term (Corollary~\ref{cor:unregularized_rate}).

The bound has three groups of terms: the dual averaging stabilizer (optimization) term, the local replay window drift, and the replay regression error, namely supervised fit $\bar\epsilon_{\mathrm{fit}}$, critic label bias $\epsilon_0$, and finite buffer variance $C_{\mathrm{var}}$, transferred to the comparator through coverage. The proof is given in Appendix~\ref{app:proofs}.

The replay terms expose the window length tradeoff. With $W_t\asymp t^\zeta$, finite buffer variance scales as $T^{-\zeta/2}$ and replay drift as $T^{\zeta-1}$. Balancing the two gives $\zeta=2/3$ and a $T^{-1/3}$ rate. Smaller $\zeta$ uses fresher but noisier data, while larger $\zeta$ averages more data but increases stale replay drift.
\end{remark}

\begin{proof}[Proof of Corollary~\ref{cor:unregularized_rate}]
Starting from Theorem~\ref{thm:fitted_replay_value_bound}, add
$$
    \frac{\lambda}{1-\gamma}
    \mathbb E_{s\sim d^{\pi^\star}}
    \left[
        \varphi(\pi^\star(\cdot\mid s))
        -
        \varphi(\pi_{\theta_t}(\cdot\mid s))
    \right]
$$
to the regularized gap. Since $\varphi(p)\in[-\ln|\mathcal A|,0]$,
$$
    \varphi(\pi^\star(\cdot\mid s))
    -
    \varphi(\pi_{\theta_t}(\cdot\mid s))
    \le
    \ln|\mathcal A|.
$$
This gives the exact bound stated in Corollary~\ref{cor:unregularized_rate}. For $\zeta=2/3$ and fixed $\lambda>0$, collecting the nonvanishing fitting and critic terms separately from the stabilizer, drift, and variance terms gives
$$
    \frac1{T-1}\sum_{t=2}^T
    \left(J(\pi^\star)-J(\pi_{\theta_t})\right)
    \le
    \frac{\lambda\ln|\mathcal A|}{1-\gamma}
    +
    \frac{2\sqrt{C_{\mathrm{cov}}(\pi^\star)}}{1-\gamma}
    (\bar\epsilon_{\mathrm{fit}}+\epsilon_0)
    +
    O\!\left(T^{-1/3}\right).
$$
The $O(T^{-1/3})$ coefficient is free of any explicit $1/\lambda$ factor. The primitive constants $C_\pi$, $C_{\mathrm{var}}$, $C_{\mathrm{cov}}(\pi^\star)$, and constants hidden in $\bar\epsilon_{\mathrm{fit}}$ and $\epsilon_0$ may still depend on $\lambda$ through the learned policy path, replay coverage, and regression difficulty.
\end{proof}

Annealed-$\lambda$ remark. Because $\lambda$ now appears \emph{only} in the entropy bias term (the score error coefficient is $\lambda$-free), annealing $\lambda$ no longer inflates the replay terms. This is the improvement over the old softmax stability route, where the $G_A/\lambda$ factor forced $\lambda_T\asymp T^{-1/6}$ and an unregularized $T^{-1/6}$ rate.

\begin{corollary}[Annealed unregularized rate]
\label{cor:annealed_rate}
In the setting of Corollary~\ref{cor:unregularized_rate}, suppose the fitting and critic floors vanish, $\bar\epsilon_{\mathrm{fit}}=\epsilon_0=0$, and $C_\pi,C_{\mathrm{var}},C_{\mathrm{cov}}(\pi^\star)$ remain uniformly bounded as $\lambda\to0$. Then $\zeta=2/3$ and $\lambda_T\asymp T^{-1/3}$ give $\frac1{T-1}\sum_{t=2}^T(J(\pi^\star)-J(\pi_{\theta_t}))=O(T^{-1/3})$.

\end{corollary}

\begin{proof}[Proof of Corollary~\ref{cor:annealed_rate}]
Under the hypotheses of Corollary~\ref{cor:annealed_rate}, set $\bar\epsilon_{\mathrm{fit}}=\epsilon_0=0$. In Corollary~\ref{cor:unregularized_rate}, the only term with explicit dependence on $\lambda$ is then the entropy bias
$$
    \frac{\lambda\ln|\mathcal A|}{1-\gamma}.
$$
The remaining terms are $O(T^{-1/3})$ and contain no explicit $1/\lambda$ factor. Under the stated uniform boundedness assumptions on $C_\pi$, $C_{\mathrm{var}}$, $C_{\mathrm{cov}}(\pi^\star)$, and the constants in the replay error bounds, their coefficients remain bounded along the schedule $\lambda=\lambda_T\to0$. Taking $\lambda_T\asymp T^{-1/3}$ makes the entropy bias $O(T^{-1/3})$, which proves
$$
    \frac1{T-1}\sum_{t=2}^T
    \left(J(\pi^\star)-J(\pi_{\theta_t})\right)
    =
    O\!\left(T^{-1/3}\right).
$$
\end{proof}

Because the score error coefficient $2\sqrt{C_{\mathrm{cov}}(\pi^\star)}$ is free of any explicit $1/\lambda$ factor, annealing the entropy weight leaves the replay terms at $T^{-1/3}$, so the optimization and replay contributions match. The condition $\lambda>0$ controls the optimization term. The unregularized ablation $\lambda=0$ (Appendix~\ref{app:additional_ablations}) is an empirical stress test. The hypothesis of Corollary~\ref{cor:annealed_rate} that $C_\pi$, $C_{\mathrm{var}}$, and $C_{\mathrm{cov}}(\pi^\star)$ stay bounded as $\lambda\to0$ is a genuine restriction, since a vanishing entropy floor sharpens the readout and can inflate both the path constant and the coverage coefficient. Corollary~\ref{cor:unregularized_rate} at fixed $\lambda$ is therefore the primary statement, and the annealed rate is conditional on that hypothesis.

\FloatBarrier
\subsection{Density weighted surrogate and the action sampling mismatch}
\label{app:density_weighted_surrogate}

This subsection connects the direct proof to the main text product space RDA surrogate. The main proof bridges through the replay regression Bayes target $Z_t^A$, which is exactly the object fit by squared loss. The main text product space surrogate instead induces a state density weighted window score $\bar Z_t^W$.

Define the average state occupancy in the current window by
$$
    \bar d_t^W(s)
    =
    \frac1{W_t}
    \sum_{k\in\mathcal W_t}d^{\pi_{\theta_k}}(s).
$$
The density weighted windowed score is
$$
    \bar Z_t^W(s,a)
    =
    \begin{cases}
    \displaystyle
    \frac{
        \sum_{k\in\mathcal W_t}d^{\pi_{\theta_k}}(s)A_k(s,a)
    }{
        \sum_{k\in\mathcal W_t}d^{\pi_{\theta_k}}(s)
    },
    &\displaystyle \sum_{k\in\mathcal W_t}d^{\pi_{\theta_k}}(s)>0,\\[2.25ex]
    \displaystyle
    \frac1{W_t}
    \sum_{k\in\mathcal W_t}A_k(s,a),
    &\displaystyle \sum_{k\in\mathcal W_t}d^{\pi_{\theta_k}}(s)=0.
    \end{cases}
$$
The second line is only a tie breaker on states with zero window state mass. It does not change the density weighted surrogate, but it makes $\bar Z_t^W$ pointwise defined everywhere.

The exact windowed product space surrogate corresponding to the main text is
$$
    \arg\max_{\pi\in\Pi}
    \left\{
        \frac1{W_t}
        \sum_{k\in\mathcal W_t}
        \langle A_k,\pi\rangle_{d^{\pi_{\theta_k}}}
        -
        \alpha_{t-1}\varphi_{\bar d_t^W}(\pi)
    \right\},
$$
where $\varphi(p)=\sum_a p(a)\log p(a)=-H(p)$. Expanding the state expectation gives
$$
\begin{aligned}
&\frac1{W_t}
\sum_{k\in\mathcal W_t}
\langle A_k,\pi\rangle_{d^{\pi_{\theta_k}}}
-
\alpha_{t-1}\varphi_{\bar d_t^W}(\pi)
\\
&=
\frac1{W_t}
\sum_{k\in\mathcal W_t}
\sum_s
 d^{\pi_{\theta_k}}(s)
\left\langle
    \pi(\cdot\mid s),A_k(s,\cdot)
\right\rangle
-
\alpha_{t-1}
\sum_s
\bar d_t^W(s)
\varphi(\pi(\cdot\mid s))
\\
&=
\sum_s
\bar d_t^W(s)
\left[
    \left\langle
        \pi(\cdot\mid s),
        \bar Z_t^W(s,\cdot)
    \right\rangle
    -
    \alpha_{t-1}
    \varphi(\pi(\cdot\mid s))
\right]
\\
&=
\left\langle
    \bar Z_t^W,\pi
\right\rangle_{\bar d_t^W}
-
\alpha_{t-1}
\varphi_{\bar d_t^W}(\pi),
\end{aligned}
$$
where the equality uses the density weighted branch on states with $\bar d_t^W(s)>0$, and states with $\bar d_t^W(s)=0$ carry zero mass in the display. Thus the state density factor multiplies both the linear advantage term and the entropy term, and the statewise optimizer on the support of $\bar d_t^W$ is $\operatorname{softmax}(\bar Z_t^W(s,\cdot)/\alpha_{t-1})$.

The gap between the replay Bayes target $Z_t^A$ and the density weighted surrogate $\bar Z_t^W$ is the action sampling mismatch. The next corollary bounds it and shows that a value guarantee phrased against the density weighted surrogate holds at the same rate.

\begin{corollary}[Action sampling mismatch and the density weighted surrogate]
\label{cor:density_weighted_surrogate}
Under the assumptions of Theorem~\ref{thm:fitted_replay_value_bound}, for every $t\ge2$,
$$
    \left\|Z_t^A-\bar Z_t^W\right\|_{2,\mu_t}
    \le
    2C_{\mathrm{loc}}t^{\zeta-1}.
$$
Consequently, replacing the replay Bayes target $Z_t^A$ by the density weighted surrogate $\bar Z_t^W$ in the score residual adds at most $\frac{2\sqrt{C_{\mathrm{cov}}(\pi^\star)}}{1-\gamma}\cdot\frac{2C_{\mathrm{loc}}}{\zeta}\frac{T^\zeta}{T-1}=O(T^{\zeta-1})$ to the bound of Theorem~\ref{thm:fitted_replay_value_bound}, the same order as the local drift term, and does not change the $T^{-1/3}$ rate.
\end{corollary}

\begin{proof}
On the support of $\mu_t$, both $Z_t^A(s,a)$ and $\bar Z_t^W(s,a)$ are convex combinations of $\{A_k(s,a):k\in\mathcal W_t\}$, so
$$
\begin{aligned}
\left\|Z_t^A-\bar Z_t^W\right\|_{2,\mu_t}
&\le
\sup_{s,a}\left|Z_t^A(s,a)-\bar Z_t^W(s,a)\right| \\
&\le
2\max_{k\in\mathcal W_t}\left\|A_k-A_t\right\|_\infty
\le
2C_{\mathrm{loc}}t^{\zeta-1},
\end{aligned}
$$
where the last inequality is the convex combination stability bound from the proof of Lemma~\ref{lem:recent_window_tracks_current_advantage_clean}. For the second claim, the triangle inequality gives
$$
    \left\|Z_{\theta_{t-1}}-\bar Z_t^W\right\|_{2,\mu_t}
    \le
    \left\|Z_{\theta_{t-1}}-Z_t^A\right\|_{2,\mu_t}
    +
    2C_{\mathrm{loc}}t^{\zeta-1},
$$
and averaging the extra term over $t=2,\dots,T$ contributes at most $\frac{2C_{\mathrm{loc}}}{\zeta}\frac{T^\zeta}{T-1}$ to the averaged score residual, which enters the value bound through the prefactor $\frac{2\sqrt{C_{\mathrm{cov}}(\pi^\star)}}{1-\gamma}$ of Theorem~\ref{thm:fitted_replay_value_bound}. The addition is of the same order as the local drift term already present in Theorem~\ref{thm:fitted_replay_value_bound}.
\end{proof}

The main statements are kept against $Z_t^A$ because this removes the action sampling mismatch from the critical path while retaining the transparent connection to the main text surrogate written here.

\FloatBarrier
\subsection{Scope of the guarantee}

The theorem provides the design rationale for the estimator and window schedule and analyzes the core algorithmic procedure implemented in code. Its objects are the stored labels, the finite window, and the fitted score. Two assumptions hold as stated in the implementation, with bounded rewards on every benchmark and an exact readout within the chosen score class, using a softmax on Atari and a diagonal Gaussian induced by the quadratic score in MuJoCo. The temperature schedule is the analyzed $\alpha_t=\lambda+\beta/t$. The analyzed window grows as $t^{2/3}$ to balance replay concentration against drift as $t\to\infty$. At a fixed budget the same tradeoff selects a fixed window, and the window ablation (Appendix~\ref{app:ablation_window}) confirms that the optimum is interior.

The remaining conditions concern the learned path, and each has a mechanism behind it. Path stability first. In full-history RDA the dual score is a running average of $k$ rounds of feedback, so consecutive scores differ by $O(1/k)$ and the readout inherits that rate. This is where Assumption~\ref{ass:learned_policy_path_stability_clean} originates. A sliding window gives up that mechanism in exchange for faster adaptation, and the implementation supplies stability by other means: the regression is warm-started from the previous round, each round replaces a small fraction of the window (one of $24$ rollouts in the GAE MuJoCo setting), the clipped dual regression bounds the movement of every stored prediction by $\epsilon_{\mathrm{clip}}$ per update, and the temperature decays. We do not verify the $C_\pi/k$ rate, and the reported runs did not need a step-size schedule, but a schedule gives a direct handle on it. With gradient clipping, the parameter movement per round is bounded by the step size times a constant, so a $1/k$ step size enforces $1/k$ policy movement whenever the readout is Lipschitz in the parameters uniformly over states. Coverage next. The coverage coefficient is finite exactly when replay puts positive mass on every action at every state the comparator visits. The entropy floor supplies the action half: with $\lambda>0$ every readout policy has full support on a finite action set, so $\mu_t(s,a)>0$ at every visited state. What remains is state coverage, that replay visits the states the comparator visits, the same exploration requirement that every comparator-based policy improvement guarantee carries. Fitting and concentration last. The fitting error enters as a per-round average of the error of the score the optimizer returns, which the theorem keeps as a floor. The finite replay concentration rate $C_{\mathrm{var}}t^{-\zeta/2}$ is the variance scaling of an average over $W_t\asymp t^{\zeta}$ samples.

Only two implementation heuristics lie outside the analyzed update: the clipped dual regression of Appendix~\ref{subsec:appendix_clipped_dual_regression}, and the per-batch normalization and clipping of advantage labels. Both act in the direction of the assumptions. The clip bounds per-update movement, which is what path stability asks for, and the normalization is an affine rescaling of the labels, which changes the effective temperature scale and nothing else in the readout. The twin-$Q$ runs autotune $\alpha_t$ as SAC does, which sits outside the analyzed schedule, and MuJoCo actions are continuous while the proof assumes finite actions. The Gaussian readout is the continuous instance of the same mirror map, and the experiments test whether the finite-action guarantee transfers.

\newpage
\section{Related Work}
\label{app:additional_related}
\label{app:related-work}

This appendix reviews the work closest to RDA2C. The organizing distinction is between (i) last iterate policy space mirror updates, (ii) cumulative FTRL/RDA updates, and (iii) practical regression or replay mechanisms used to instantiate these ideas with function approximation. RDA2C sits at the intersection of the last two: it uses a dual averaging actor objective, but estimates the cumulative dual object through replayed advantage labels in a deep actor--critic loop.

\FloatBarrier
\subsection{Policy mirror descent, natural gradients, and trust region methods}
\label{app:rw:pmd-npg}

A long line of work interprets modern policy optimization as approximate mirror descent in policy space, with KL divergence as the Bregman divergence induced by negative entropy. Natural policy gradient introduced the use of the Fisher geometry for policy search \citep{kakade2001natural}, and TRPO later operationalized a trust region version of policy improvement with a KL constraint \citep{schulman2015trust}. PPO replaces the constrained update by a clipped surrogate that approximately limits deviation from the behavior policy \citep{schulman2017proximal}, while MPO uses an EM-like KL regularized policy improvement step \citep{rw_abdolmaleki2018mpo}. Maximum entropy methods such as SAC optimize entropy regularized objectives and can be interpreted within the broader regularized MDP and KL regularized policy optimization family \citep{rw_haarnoja2018sac,rw_geist2019regularizedmdp}.

Theoretical work has made this mirror descent view increasingly explicit. Convex analytic treatments of entropy regularized MDPs relate exact trust region and regularized policy improvement updates to mirror descent and duality \citep{rw_neu2017unified,rw_geist2019regularizedmdp}. Adaptive trust region and policy mirror descent analyses give global convergence and sampling complexity guarantees for exact or near exact policy space updates \citep{rw_shani2020adaptive,rw_lan2023pmd}. Functional mirror ascent and surrogate optimization analyses further clarify how policy improvement objectives can be derived independently of a particular parameterization \citep{rw_vaswani2022surrogates}. Empirically motivated mirror descent actor--critic methods, including mirror descent policy optimization (MDPO), exploratory conservative policy optimization (ECPO), and mirror descent actor--critic (MDAC), instantiate related KL/entropy regularized updates in deep RL systems \citep{rw_tomar2022mdpo,rw_mei2019ecpo,rw_iwaki2025mdac}.

Q--NPG is relevant to RDA2C because it connects value function regression, compatible function approximation, and natural gradient policy updates. In the policy gradient theory of \citet{rw_agarwal2019pgtheory}, NPG and Q--NPG updates are written through regression problems whose solutions define a natural policy gradient step. This is close in spirit to our use of advantage labels, but the update remains a last iterate PMD/NPG step: the fitted object determines a local update from the current policy. RDA2C instead fits a replay estimated state weighted average advantage score and reads the policy from that cumulative object through the entropy mirror map.

\FloatBarrier
\subsection{FTRL, regularized dual averaging, and policy dual averaging}
\label{app:rw:ftrl-rda}
\label{sec:related-work}

Regularized dual averaging originates in online convex optimization and stochastic optimization, where the algorithm maintains an average of past subgradients in the dual space and maps back to the primal variable through a regularized mirror map. Foundational treatments include Nesterov's primal-dual subgradient schemes \citep{nesterov2009primal} and Xiao's regularized dual-averaging method \citep{xiao2010dual}. RDA is closely related to follow-the-regularized-leader: \citet{rw_mcmahan2011ftrl} establishes equivalence results connecting mirror descent and FTRL updates under suitable choices of regularization and linearization.

This equivalence is important for positioning RDA2C. We do not claim that FTRL/RDA generically improves the worst case online learning rate over mirror descent. Our use is representational: FTRL/RDA exposes the cumulative feedback as a single object. PMD/OMD writes the next update as a local step around the last policy. RDA/FTRL writes the next update as a regularized optimization over the cumulative linearized feedback. In reinforcement learning, that cumulative feedback is naturally represented by accumulated advantage or $Q$ labels. This is what makes replay useful as an actor representation: the buffer approximates the object the RDA update needs.

The closest precedent for the update itself is POLITEX \citep{rw_abbasi2019politex}, which sets the next policy to the softmax of the cumulative sum of past action-value estimates and proves regret bounds through the reduction to expert prediction. RDA2C keeps that cumulative-softmax form but replaces the explicit sum of stored critics by one regression on replayed advantage labels, adds the persistent entropy floor and the decaying stabilizer of RDA, and analyzes a finite replay window in place of the full history. In the same direction, \citet{rw_vieillard2020leverage} show that KL-regularized policy iteration implicitly averages past $Q$ estimates, and Deep CPI \citep{rw_vieillard2020deepcpi} makes the related conservative mixture update practical with function approximation. Both hold the cumulative object recursively, in the current parameters. RDA2C holds it in the label buffer, which is what makes the window an explicit design variable.

Within RL, a close theoretical precursor is policy dual averaging. \citet{rw_ju2022generalspaces} study policy optimization over general state and action spaces and introduce PDA methods designed to work with function approximation more directly than exact policy mirror descent. That line studies cleaner regimes with exact or bounded error evaluation, and the original continuous action update is defined by a statewise optimization subproblem. The actor accelerated PDA follow up \citep{rw_ju2026aapda} addresses this computational issue by learning an actor that approximates the PDA target, while DAPO studies dual-side approximation error through the Bregman divergence induced by the mirror potential \citep{rw_xiong2024dapo}. RDA2C differs by eliminating the separate projection tracking actor: the learned dual score model $Z_\theta$ is the policy's natural coordinate representation, and the current policy is obtained directly through the entropy readout $\nabla\varphi^*(Z_\theta/\alpha_t)$.

Two further contrasts follow. PDA solves the readout statewise and numerically, while RDA2C replaces that solve with one regression and an exact mirror map. DAPO measures the error of a given dual approximation, while RDA2C specifies how to construct the dual object from replayed labels. The analyses differ the same way: prior guarantees assume exact or bounded-error policy evaluation, while our finite-time analysis charges the update for critic bias, finite-buffer variance, and replay-induced distribution shift.

\FloatBarrier
\subsection{Replay, off policy actor--critic learning, and stale data}
\label{app:rw:replay-stale-data}

Experience replay is a standard mechanism for improving sample efficiency in deep RL. In off-policy actor--critic methods, replay usually provides data for Bellman regression and actor-gradient estimates, as in SAC \citep{rw_haarnoja2018sac}. In on-policy or near on-policy methods such as PPO, data are typically reused only for a small number of epochs, with clipping or KL control limiting the effect of policy drift \citep{schulman2017proximal}. In both cases, stale data are primarily handled as an off-policy correction, stabilization, or variance-reduction issue.

RDA2C changes the role of replay on the actor side. Because the ideal RDA update depends on a historical average of policy space feedback, past advantage labeled samples are not merely stale approximations to the newest gradient. They are samples from the empirical dual objective that defines the next actor. This distinction also explains why buffer size appears explicitly in our theory. A short window reduces distribution shift but raises estimator variance, while a long window averages more labels but can bias the dual estimate toward older policies. The main theorem formalizes this replay-specific decomposition through critic bias, finite buffer variance, and replay induced distribution shift.

\FloatBarrier
\subsection{Reduction and regression views of policy improvement}
\label{app:rw:reductions-regression}

RDA2C is also related to reductions that cast sequential decision making as iterative supervised learning on aggregated data. DAgger aggregates expert-labeled states to reduce imitation learning to no-regret online learning \citep{ross2011reduction}. NRPI extends this perspective to reinforcement learning by aggregating cost-to-go examples under an exploration distribution and reducing approximate policy iteration to no regret learning \citep{rw_ross2014nrpi}. These methods are close in spirit to RDA2C because the dataset is not an incidental implementation detail: it is the object through which the next policy improvement problem is specified.

Recent regression based policy optimization methods make a related move in modern RL and RLHF settings. REBEL reduces reinforcement learning to squared loss regression on relative rewards, REFUEL extends a regression and aggregation viewpoint to multi-turn RLHF under self-generated data, and A$^\star$-PO estimates an optimal-advantage-like object before applying simple regression policy updates \citep{rw_gao2024rebel,rw_gao2024refuel,rw_brantley2025astarpo}. RDA2C follows the same broad reduction to regression philosophy, but its regression target is derived from regularized dual averaging: the model $Z_\theta$ approximates a windowed average of policy space improvement feedback, and the actor update is the exact entropy mirror map readout from that learned dual score.

Advantage-weighted methods also turn policy improvement into supervised learning. AWR \citep{rw_peng2019awr}, AWAC \citep{rw_nair2020awac}, and CRR \citep{rw_wang2020crr} fit the policy by weighted maximum likelihood on replayed actions, with weights that increase with the estimated advantage. That regression is in the primal: the target is a reweighted behavior distribution, and the policy is the fitted likelihood model. RDA2C regresses in the dual: the target is the advantage label itself, the fitted object is a score, and the policy is its mirror-map readout. The two differ in what the buffer represents. For advantage weighting, replay supplies actions to imitate. For RDA2C, replay supplies the cumulative feedback that defines the next policy.

Empirical studies of example forgetting in neural networks provide an additional motivation for keeping policy improvement history in an explicit buffer \citep{rw_toneva2019forgetting}. Recursive updates keep that history only in the actor weights. This motivation is secondary. The primary distinction is algorithmic: RDA2C's replay buffer estimates the dual object required by the RDA update, and recursive actor updates keep past feedback only in the current parameter vector.

\end{document}